\documentclass{article}

\usepackage[preprint]{neurips_2026}

\usepackage[utf8]{inputenc} 
\usepackage[T1]{fontenc}    
\usepackage{hyperref}       
\usepackage{url}            
\usepackage{booktabs}       
\usepackage{amsfonts}       
\usepackage{amssymb}
\usepackage{nicefrac}       
\usepackage{microtype}      
\usepackage{xcolor}         

\usepackage{amsmath, amsthm, enumitem, bbm, subcaption, tikz, pgf, thmtools, thm-restate, placeins, mathtools, algorithm, algpseudocode, titletoc}

\usetikzlibrary{decorations.pathreplacing}
\usetikzlibrary{calc}

\theoremstyle{plain}
\newtheorem{theorem}{Theorem}[section]
\newtheorem{proposition}[theorem]{Proposition}
\newtheorem{lemma}[theorem]{Lemma}
\newtheorem{corollary}[theorem]{Corollary}
\theoremstyle{definition}
\newtheorem{example}[theorem]{Example}
\newtheorem{definition}[theorem]{Definition}

\newcommand{\layernumber}{\ell}
\newcommand{\numhiddenlayers}{L}
\newcommand{\maxcumulantorder}{K}

\newcommand{\tens}[2]{\mathcal{T}^{#1}(\mathbb{R}^{#2})}
\newcommand{\syms}[2]{\mathcal{S}^{#1}(\mathbb{R}^{#2})}
\newcommand{\tran}{^\top}
\newcommand{\pow}{^{\mathsf P}}

\DeclareMathOperator{\tr}{tr}
\DeclareMathOperator{\type}{type}
\newcommand{\Diag}{\mathfrak{D}}
\newcommand{\Perm}{\operatorname{P}}
\DeclareMathOperator{\diag}{\tilde\Diag^{-1}_{(2)}}

\newcommand{\cDia}[1]{\operatorname{cDia}_{\left[\leq {#1}\right]}}
\renewcommand{\Vec}{\operatorname{Vec}}
\newcommand{\cVec}[1]{\operatorname{cVec}_{\left[\leq {#1}\right]}}
\DeclareMathOperator{\Pair}{Pair}

\newcommand{\E}{\mathbb{E}}
\DeclareMathOperator{\Var}{Var}
\DeclareMathOperator{\He}{He}

\newcommand{\R}{\mathbb{R}}
\newcommand{\Z}{\mathbb{Z}}
\DeclareMathOperator{\ReLU}{ReLU}
\newcommand{\nth}{^\text{th}}

\newcommand{\bfA}{\mathbf{A}}
\newcommand{\bfB}{\mathbf{B}}
\newcommand{\bfC}{\mathbf{C}}
\newcommand{\bfD}{\mathbf{D}}

\newcommand{\bfM}{\mathbf{M}}

\newcommand{\bfQ}{\mathbf{Q}}

\newcommand{\bfS}{\mathbf{S}}
\newcommand{\bfT}{\mathbf{T}}

\newcommand{\bfW}{\mathbf{W}}

\newcommand{\bff}{\mathbf{f}}

\makeatletter
\newcommand{\ifpreprint}[2]{\if@preprint#1\else#2\fi}
\newcommand{\redactifanon}[1]{\if@anonymous[redacted]\else#1\fi}
\newcommand{\cutforsubmission}[1]{\if@anonymous\else{\iffalse\color{blue!60!black}\fi#1}\fi}
\makeatother

\title{Estimating the expected output of wide random MLPs more efficiently than sampling}

\author{
  Wilson Wu\thanks{Joint first authors.} \\
  \texttt{wilson@alignment.org} \\
  \And
  Victor Lecomte\footnotemark[1] \\
  \texttt{victor@alignment.org} \\
  \And
  Michael Winer\footnotemark[1] \\
  \texttt{mike@alignment.org} \\
  \And
  George Robinson\footnotemark[1] \\
  \texttt{george@alignment.org} \\
  \And
  Jacob Hilton\thanks{Corresponding author.}\hspace{0.33em}\thanks{Joint senior authors.} \\
  \texttt{jacob@alignment.org} \\
  \And
  Paul Christiano\footnotemark[3] \\
  \texttt{paul@alignment.org} \\
  \AND
  \normalfont{Alignment Research Center}
}

\begin{document}

\maketitle

\begin{abstract}
By far the most common way to estimate an expected loss in machine learning is to draw samples, compute the loss on each one, and take the empirical average. However, sampling is not necessarily optimal. Given an MLP at initialization, we show how to estimate its expected output over Gaussian inputs without running samples through the network at all. Instead, we produce approximate representations of the distributions of activations at each layer, leveraging tools such as cumulants and Hermite expansions. We show both theoretically and empirically that for sufficiently wide networks, our estimator achieves a target mean squared error using substantially fewer FLOPs than Monte Carlo sampling. We find moreover that our methods perform particularly well at estimating the probabilities of rare events, and additionally demonstrate how they can be used for model training. Together, these findings suggest a path to producing models with a greatly reduced probability of catastrophic tail risks.
\end{abstract}

\section{Introduction}\label{introductionsection}

Stochastic gradient descent is based on a simple idea: to estimate the expected loss of a neural network (and hence its gradient), we draw some samples, compute the loss on each one, and average the results. This sample-based approach is so natural that it is rarely questioned, yet it may be far from optimal. By analogy, consider a one-dimensional definite integral: we could approximate it numerically using Monte Carlo sampling over the domain of integration, but when a closed form for the integral exists, evaluating it directly is far more efficient. Similarly, when the input distribution and loss function of a neural network can be specified formally (such as with board games, formal mathematics, or simulated physical environments), we might hope to use analytic methods to estimate the expected loss much more efficiently than with sampling.

This hope may seem far-fetched for neural networks, since their inputs are high-dimensional and their behavior is highly irregular. Yet we show that, at least in some well-behaved cases, it can in fact be realized. Even though computing the expected loss exactly is typically intractable, there can be tractable analytic approximations that significantly outperform Monte Carlo sampling in terms of mean squared error for a given computational budget. In our definite integral analogy, these might correspond to a rapidly convergent series expansion for the integral.

The particular case we study is that of wide random multilayer perceptrons (MLPs): fully-connected feedforward neural networks with random weights, in the large-width limit. Importantly, our algorithms take as input the weights of a \textit{particular} network and estimate the expected output over random inputs, i.e., a \textit{quenched} average, in the language of statistical physics. This is not to be confused with the \textit{annealed} average over the choice of weights, as is commonplace elsewhere in deep learning theory \citep*{pdlt}. \textit{After} comparing our estimates with ground truth, we average over the choice of weights, which are drawn independently from $\mathcal N\left(0,\frac cn\right)$ for some constant $c$, where $n$ is the network width.

Our main result is the design of sample-free estimation procedures that outperform Monte Carlo sampling at this problem, both theoretically and empirically.

Theoretically, we prove that in the average case, for networks with a fixed, constant depth, our best-performing algorithms achieve a mean squared error of $O(\varepsilon^2)$ in time $O(\frac n{\varepsilon^2})$.\footnote{Our theoretical results assume polynomial activation functions and conditions on the growth rate of $\varepsilon\!\left(n\right)$.} By comparison, Monte Carlo sampling achieves a mean squared error of $\Theta(\varepsilon^2)$ in time $\Theta(\frac{n^2}{\varepsilon^2})$,\footnote{We assume the use of naive rather than Strassen-type matrix multiplication in all algorithms.} a factor of $n$ slower. This means that for any given depth, our algorithms are guaranteed to outperform Monte Carlo sampling at sufficiently large width, although the dependence of our algorithms on depth is worse.

\begin{figure}[t]
  \centering
  \makebox[0pt][c]{\scalebox{0.8}{\input{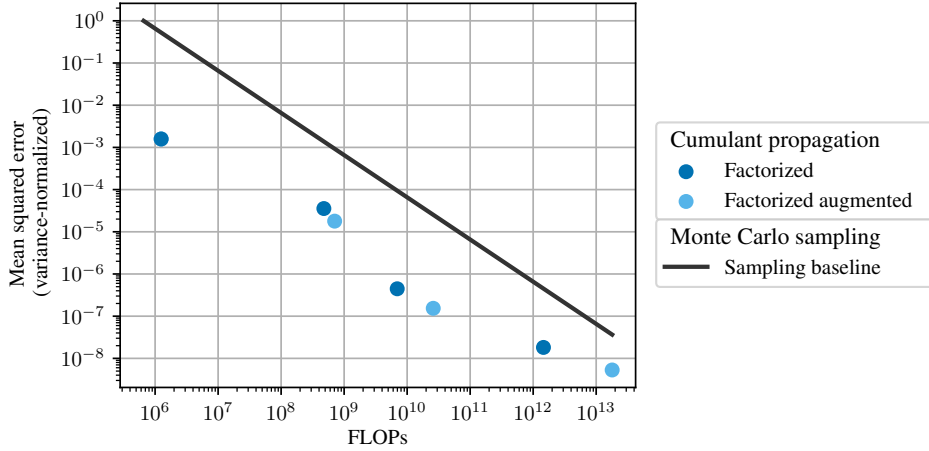}}}
  \caption{Performance of our best cumulant propagation algorithms at estimating the mean of random ReLU MLPs with 4 hidden layers and width 256. Mean over 5 random seeds shown. (Error bars hidden because they would be too small to be readable, given the logarithmic scale.)}
  \label{msevsflopsmainfigure}
\end{figure}

Empirically, we validate the performance of our algorithms at realistic widths. Figure \ref{msevsflopsmainfigure} shows the mean squared error of our best-performing algorithms versus the number of floating-point operations (FLOPs) they perform, for random ReLU networks with 4 hidden layers and width 256. We are able to outperform Monte Carlo sampling at computational budgets spanning 7 orders of magnitude, in some cases by a factor of over 100.

Furthermore, since our algorithms produce differentiable estimates, they can also be used to train neural networks in certain simple settings. As a proof of concept, we apply our algorithms to distill a random teacher into a student of a different width, a process we refer to as \textit{mechanistic distillation}.

Despite some similarities in performance, our algorithms are very unlike Monte Carlo sampling. Instead of drawing samples, we employ \textit{cumulant propagation} \citep[Appendix D]{poi}. In cumulant propagation, we form an approximate representation of each activation distribution using its low-order cumulants, which can be thought of as a series expansion about a Gaussian distribution. We start with the input distribution and propagate this forward through the network layer by layer until we reach the output. The smaller the error tolerance $\varepsilon$ as a function of the network width $n$, the higher the order of cumulants we track: roughly speaking, if $\varepsilon\!\left(n\right)^2\sim\frac 1{n^\maxcumulantorder}$, then we track cumulants of order at most $\maxcumulantorder$.

Our motivation for producing analytic loss estimates goes beyond their overall computational efficiency, and is closely related to the problem of \textit{low probability estimation} \citep{lpe}. When we apply our algorithms to estimating the tails of the output distribution, we find that they outperform Monte Carlo sampling to an increasing degree as the true probability decreases. This suggests that using analytic loss estimates to train networks could greatly reduce the probability of catastrophic tail risks.

\cutforsubmission{

The remainder of the paper is structured as follows.
\begin{itemize}
\item
In Section \ref{problemstatementsection}, we provide a formal definition of our MLP setup, a precise problem statement, and a description of the Monte Carlo sampling baseline.
\item
In Section \ref{preliminariessection}, we define the main mathematical constructions used by our algorithms.
\item
In Section \ref{methodssection}, we outline our sample-free estimation procedures, including \textit{augmented} and \textit{factorized} versions. Further explanation and pseudocode for special cases can be found in Appendix \ref{meancovpropappendix}.
\item
In Section \ref{theoreticalresultssection}, we state our main theoretical results about our algorithms. 
\item
In Section \ref{empiricalresultssection}, we describe our experimental setup and our main empirical results. Further empirical results can be found in Appendices \ref{empiricalappendix}, \ref{depthappendix}, \ref{ablationappendix}, \ref{otheractivationsappendix}, \ref{lpebaselineappendix}, \ref{mechdistappendix} and \ref{walltimeappendix}.
\item
In Section \ref{discussionsection}, we discuss the implications of our results for improving training efficiency for losses dominated by rare events, as well as open problems and conjectures.
\end{itemize}

We also provide a technical supplement, Appendix \ref{supplementappendix}, containing precise descriptions of all of our algorithms and full proofs of all of our theoretical results.

}

\section{Problem statement}\label{problemstatementsection}

\subsection{MLP definition}

We consider MLPs with equal input dimension, output dimension, and hidden dimensions, and without a final activation function. We do not incorporate bias terms, but we show how to adapt our algorithms to work for these in Appendix \ref{otheractivationsappendix}. A formal definition is as follows.

\begin{restatable}[]{definition}{mlpdef}\label{mlpdefinition}
The \textit{multilayer perceptron} (MLP) with $\numhiddenlayers\geq 0$ hidden layers, width $n\geq 1$, activation functions $\phi_1,\dots,\phi_\numhiddenlayers:\mathbb R\to\mathbb R$ and weights $\boldsymbol\theta=\left(\mathbf W^{\left(1\right)},\dots,\mathbf W^{\left(\numhiddenlayers+1\right)}\right)\in\mathbb R^{\left(\numhiddenlayers+1\right)\times n \times n}$ is the function $M_{\boldsymbol\theta}:\mathbb R^n\to\mathbb R^n$ defined by
\[M_{\boldsymbol\theta}\!\left(\mathbf x\right)=\mathbf W^{\left(\numhiddenlayers+1\right)}\phi_\numhiddenlayers\!\left(\mathbf W^{\left(\numhiddenlayers\right)}\dots\phi_2\!\left(\mathbf W^{\left(2\right)}\phi_1\!\left(\mathbf W^{\left(1\right)}\mathbf x\right)\right)\dots\right),\]
where the activation functions are applied coordinatewise.
\end{restatable}

In our theoretical results, we consider polynomial activation functions, but this is not a practical limitation. In our experiments, our default choice of activation function is the rectified linear unit, $\operatorname{ReLU}\left(z\right):=\max\left(z,0\right)$, and we consider other activation functions in Appendix \ref{otheractivationsappendix}.

\subsection{Central estimation problem}\label{estimationproblemsubsection}

The central problem we consider is as follows. Given weights $\boldsymbol\theta\in\mathbb R^{\left(\numhiddenlayers+1\right)\times n \times n}$ drawn independently from $\mathcal N\left(0,\frac 2n\right)$ (following \citet{heinit}), estimate
\[\mathbb E_{X\sim\mathcal N\left(0,\mathbf I_n\right)}\left[M_{\boldsymbol\theta}\left(X\right)\right],\]
where $M_{\boldsymbol\theta}$ is the MLP with weights $\boldsymbol\theta$. More precisely, an estimation procedure should take as input $\boldsymbol\theta$ and $\varepsilon>0$, and produce an estimate with mean squared error $O(\varepsilon^2)$ over the choice of $\boldsymbol\theta$ for each output neuron.

\cutforsubmission{Each entry of the above expectation can be thought of as the expected loss of the network under a simple linear loss function of the network's output. Our methods can easily be adapted to other closed-form loss functions that only depend on the network's output, by treating the loss function as an additional activation function. For more complex formally-defined loss functions that also depend on the network's input, it is often still possible to adapt our methods, although more work may be required. We demonstrate this in the case of a simple distillation loss in Appendix \ref{mechdistappendix}.}

\subsection{Monte Carlo sampling baseline}

The baseline estimation procedure we consider is Monte Carlo sampling. This procedure samples $X^{\left(1\right)},\dots,X^{\left(N\right)}\overset{\text{i.i.d.}}\sim\mathcal N\left(0,\mathbf I_n\right)$ for $N=\left\lceil\frac 1{\varepsilon^2}\right\rceil$ and computes the estimate
\[\frac 1N\sum_{r=1}^NM_{\boldsymbol\theta}\!\left(X^{\left(r\right)}\right).\]
This achieves the required mean squared error of $\Theta(\varepsilon^2)$, since a single forward pass has variance $\Theta(1)$ under our choice of initialization, and runs in time $\Theta(\frac{n^2}{\varepsilon^2})$ (recall that we treat the number of hidden layers $\numhiddenlayers$ as a constant in our asymptotics).

\section{Mathematical preliminaries}\label{preliminariessection}

\subsection{Cumulants}\label{cumulantssubsection}

Our sample-free estimation procedures work by estimating the \textit{cumulants} of the network's activations. The $k\nth$-order cumulants of a vector-valued random variable $X\in\mathbb R^n$, denoted $\kappa\left[X_{i_1},\dots,X_{i_k}\right]$ for $i_1,\dots,i_k=1,\dots,n$, are defined recursively in terms of the moments of $X$ by the partition formula
\[\mathbb E\!\left[X_{i_1}\cdots X_{i_k}\right]=\sum_{\substack{\text{partitions}\\\pi\text{ of }\left\{1,\dots,k\right\}}}\prod_{S\in\pi}\kappa\!\left[X_{i_s}:s\in S\right].\]
For example,
\begin{align*}
\mathbb E\!\left[X_i\right]=&\,\kappa\!\left[X_i\right]\qquad\quad\text{and}\\
\mathbb E\!\left[X_iX_j\right]=&\,\kappa\!\left[X_i\right]\kappa\!\left[X_j\right]+\kappa\!\left[X_i,X_j\right].
\end{align*}
Thus the first-order cumulants form the mean vector of $X$, and the second-order cumulants form the covariance matrix of $X$. The third- and higher-order cumulants of a distribution can be thought of as measuring the deviation of the distribution from Gaussianity, since these cumulants are zero for a Gaussian distribution.

\subsection{Hermite expansions}\label{hermitesubsection}

To handle non-linear activation functions, we take \textit{Hermite expansions}. Given $Y\sim\mathcal N\!\left(\mu,\sigma^2\right)$ and $\phi:\mathbb R\to\mathbb R$ with $\mathbb E\!\left[\phi\left(Y\right)^2\right]<\infty$, the Hermite expansion of $\phi$ with respect to $Y$ is the series expansion
\[\phi\left(z\right)=\sum_{k=0}^\infty\frac 1{k!}\,\widehat\phi^{\left(\mu,\sigma^2\right)}_k\sigma^k\operatorname{He}_k\!\left(\frac{z-\mu}\sigma\right)\]
where $\operatorname{He}_k$ is the $k\nth$ probabilist's \textit{Hermite polynomial}, the monic polynomial of degree $k$ that is orthogonal to all polynomials of lower degree under the standard Gaussian measure. The \textit{Hermite coefficients}\footnote{Other authors use the term \textit{Hermite coefficient} to refer to the same expression divided by $k!$.} in this expansion are given by
\begin{equation}
\label{eq:hermite-def}
   \widehat\phi^{\left(\mu,\sigma^2\right)}_k:=\frac 1{\sigma^k}\mathbb E\!\left[\phi\left(Y\right)\operatorname{He}_k\!\left(\frac{Y-\mu}\sigma\right)\right]. 
\end{equation}
\cutforsubmission{Truncating the expansion gives the best low-degree approximation to $\phi$ in mean squared error under $Y$.}

\subsection{Diagram summation formula for cumulants}\label{diagramssummarysubsection}

After taking the Hermite expansion of a non-linear activation function $\phi:\mathbb R\to\mathbb R$, we will need to compute the cumulants of $\phi\left(Z_1\right),\dots,\phi\left(Z_n\right)$ using the cumulants of $Z_1,\dots,Z_n$ and the Hermite coefficients of $\phi$. To do this, we apply a formula that expresses the desired cumulants as a sum over combinatorial objects we call \textit{diagrams}. We refer to this formula as the \textit{(Hermite-based) diagram summation formula for cumulants}, and we state and explain it in Appendix \ref{diagramsappendix}.

\subsection{Sizes of cumulants with repeated indices}\label{repeatedindicessubsection}

The typical size of a $k\nth$-order cumulant of the network's activations strongly depends on whether a single neuron index appears multiple times. This has an effect on the error of truncating our diagram summation formula for cumulants, which impacts the design of our algorithm.

To illustrate this, we consider the special case of our diagram summation formula in which third- and higher-order cumulants are ignored. This is the covariance propagation formula of \citet[Theorem 1]{covprop}, which we derive as Corollary \ref{covpropcorollary}. This states that when $Z_1,\dots,Z_n$ are multivariate Gaussian, then
\[\operatorname{Cov}\left[\phi\left(Z_{i}\right),\phi\left(Z_{j}\right)\right]=\sum_{k=1}^\infty\frac 1{k!}\;\widehat\phi^{\left(\mathbb E\left[Z_{i}\right],\operatorname{Var}\left[Z_{i}\right]\right)}_k\,\widehat\phi^{\left(\mathbb E\left[Z_{j}\right],\operatorname{Var}\left[Z_{j}\right]\right)}_k\operatorname{Cov}\left[Z_{i},Z_{j}\right]^k.\]
Typically, $|\operatorname{Cov}[Z_i,Z_j]|$ is $O(\frac 1{\sqrt n})$ when $i\neq j$, but $O(1)$ when $i=j$. This means that the error of truncating this formula decays rapidly with $n$ when $i\neq j$, but does not decay with $n$ when $i=j$. This motivates us to change our approach from applying our diagram summation formula naively.

\section{Methods}\label{methodssection}

\subsection{Basic cumulant propagation algorithm}\label{kpropsubsection}

With the above preliminaries in place, we can now roughly describe our sample-free estimation procedures. We begin with the basic version of our algorithm, before going on to describe variants.

Our algorithm first picks some maximum order $\maxcumulantorder$ of cumulants to track. In our theoretical results, we choose $\maxcumulantorder$ to achieve $O\!\left(\varepsilon^2\right)$ error, and in our experiments, we treat $\maxcumulantorder$ as a hyperparameter. Given $\maxcumulantorder$, we start with the input distribution $\mathcal N\!\left(0,\mathbf I_n\right)$, which has second-order cumulant matrix $\mathbf I_n$ and all other cumulants equal to $0$. We then proceed through each operation in the network, updating the estimated cumulants at each step.
\begin{itemize}
\item
At each matrix multiplication step, we update the estimated cumulants using the multilinearity property of cumulants. This involves taking an einsum of the $k\nth$-order estimated cumulant tensor with $k$ copies of the weight matrix, for $k=1,\dots,\maxcumulantorder$.
\item
At each activation function step, we update the estimated cumulants using a truncation of the diagram summation formula for cumulants. This formula allows us to compute the cumulants of the output in terms of the cumulants of the input and the Hermite coefficients of the activation function.
\end{itemize}
Once we reach the output of the network, we return the final estimated first-order cumulant vector.

As stated so far, the error of this algorithm is too large, because of the problem with sizes of cumulants with repeated indices described in Section \ref{repeatedindicessubsection}. To remedy this, we make two further adjustments.
\begin{enumerate}
\item
\textbf{Power cumulants.} Instead of taking Hermite expansions of the activation function only, we take Hermite expansions of \textit{powers} of the activation function, and use these to estimate the \textit{power cumulants} of the activation output. The $k\nth$-order power cumulants of $X$ are the cumulants of the form $\kappa\!\left[X_{i_m}^{\alpha_1},\dots,X_{i_m}^{\alpha_m}\right]$ with $\alpha_1,\dots,\alpha_m\geq 1$, $\alpha_1+\dots+\alpha_m=k$ and $i_1,\dots,i_m$ \textit{distinct}, from which we can obtain the $k\nth$-order cumulants of $X$ by applying the partition formula and its inverse version. For example, when $\maxcumulantorder=2$, we compute Gaussian variances exactly, while continuing to use an approximation for off-diagonal Gaussian covariances.
\item
\textbf{Extra full trace for odd $\maxcumulantorder$.} When the maximum cumulant order $\maxcumulantorder$ is odd, we also track a single number corresponding to the \textit{full trace} of the $\left(\maxcumulantorder+1\right)\nth$-order cumulant tensor, $\sum_{i_1,\dots,i_{(\maxcumulantorder+1)/2}=1}^n\kappa\!\left[X_{i_1},X_{i_1},X_{i_2},X_{i_2},\dots,X_{i_{(\maxcumulantorder+1)/2}},X_{i_{(\maxcumulantorder+1)/2}}\right]$. For example, when $\maxcumulantorder=1$, we also track the covariance matrix trace $\sum_{i=1}^n\operatorname{Var}\!\left[X_i\right]$.
\end{enumerate}

To help explain our sample-free estimation procedures more precisely, we provide pseudocode for the cases $\maxcumulantorder=1$ and $\maxcumulantorder=2$ in Appendix \ref{meancovpropappendix}. A precise and fully general formulation is significantly more involved, and is given in Appendix \ref{kpropfull}. Python implementations are given at \redactifanon{\url{https://github.com/alignment-research-center/mlp_kprop}}.

\subsection{Augmented algorithm}

So far, we have roughly described the basic version of our algorithm. However, we have some flexibility to trade off between mean squared error and runtime by varying exactly which cumulants to track. One possibility (when $\maxcumulantorder$ is odd, say) is to track not merely the full trace of the $\left(\maxcumulantorder+1\right)\nth$-order cumulant tensor, but the entire tensor \textit{except for its traceless component} (a traceless tensor is one for which taking the trace over any pair of indices results in the zero tensor). This improves the mean squared error of the algorithm, but increases its runtime.  We refer to this as the \textit{augmented} version of our algorithm, and describe it more precisely in Appendix \ref{sec:augmented-algo}.

The augmented version of our algorithm is not the only possible tradeoff between mean squared error and runtime that can be made. Rather, it serves as an example of how our algorithm can be flexibly adapted based on the computational budget, beyond merely choosing a different value of $\maxcumulantorder$.

\subsection{Factorized algorithms}\label{factorizedsubsection}

Another important change we can make to our algorithm is to modify how the cumulant tensors are represented. For $\maxcumulantorder\geq 3$, the tensor of $\maxcumulantorder\nth$-order cumulants admits a factorized representation of size $n^{\maxcumulantorder-1}$ that can be preserved through each step of the algorithm. By representing this tensor in factorized form, we can improve the runtime of our algorithm by a factor of $n$, up to constant factors,\footnote{The new constant factors scale quadratically rather than linearly width depth, as can be seen for $\maxcumulantorder\leq 4$ in Appendix \ref{flopappendix}.} without changing its final output. We refer to this as the \textit{factorized} version of our algorithm, and describe it more precisely in Appendix \ref{sec:factorized-algo}.

Unlike the augmented version of our algorithm, the factorized version offers a strict performance improvement for sufficiently large $n$, which allows us to substantially outperform Monte Carlo sampling for large $n$. We can also apply both modifications at once to obtain a \textit{factorized augmented} version of our algorithm.

\section{Theoretical results}\label{theoreticalresultssection}




Our main theoretical results are as follows.

\begin{restatable}[]{theorem}{mainthm}\label{matchingsamplingtheorem}
Fix an integer $\numhiddenlayers\geq 0$ and polynomial activation functions $\phi_1,\dots,\phi_\numhiddenlayers$. Given weights $\boldsymbol\theta\in\mathbb R^{\left(\numhiddenlayers+1\right)\times n\times n}$ drawn independently from $\mathcal N\left(0,\frac 2n\right)$ for an integer $n\geq 1$, write $M_{\boldsymbol\theta}:\mathbb R^n\to\mathbb R^n$ for the MLP with $\numhiddenlayers$ hidden layers, width $n$, activation functions $\phi_1,\dots,\phi_\numhiddenlayers$ and weights $\boldsymbol\theta$. Then there is an estimation procedure $\mathbb G\left(\boldsymbol\theta,\varepsilon\right)$ such that for every noticeable\footnote{This means that for some $k\geq 0$, $\varepsilon\!\left(n\right)\geq\frac 1{n^k}$ for all sufficiently large $n$ \citep[Chapter 1.2]{goldreich}.} function $0<\varepsilon\!\left(n\right)\leq 1$ and every output index $i$, as $n\to\infty$,
\[\mathbb E_{\boldsymbol\theta}\!\left[\left(\mathbb G\!\left(\boldsymbol\theta,\varepsilon\!\left(n\right)\right)_i-\mathbb E_{X\sim\mathcal N\left(0,\mathbf I_n\right)}\left[M_{\boldsymbol\theta}\left(X\right)\right]_i\right)^2\right]=O\!\left(\varepsilon\!\left(n\right)^2\right)\]
and $\mathbb G\!\left(\boldsymbol\theta,\varepsilon\right)$ runs in time $O\!\left(\frac{n^2}{\varepsilon^2}\right)$.
\end{restatable}

\begin{restatable}[]{theorem}{beatingthm}\label{beatingsamplingtheorem}
In Theorem \ref{matchingsamplingtheorem}, if we also require that $\varepsilon\!\left(n\right)^2=O\!\left(\frac 1{n^2}\right)$, then $\mathbb G\left(\boldsymbol\theta,\varepsilon\right)$ can be made to run in time $O\!\left(\frac n{\varepsilon^2}\right)$.
\end{restatable}


Theorem \ref{matchingsamplingtheorem} is obtained from the basic version of our sample-free estimation procedure, and Theorem \ref{beatingsamplingtheorem} from the factorized version. Thus our results show that, up to constant factors, the basic version of our algorithm matches the performance of Monte Carlo sampling in terms of mean squared error for a given computational budget, while the factorized version outperforms Monte Carlo sampling by a factor equal to the network width $n$, providing $\varepsilon$ is not too large.\footnote{A mean squared error of $\varepsilon\!\left(n\right)^2=O(\frac 1{n^2})$ is achieved with $\maxcumulantorder\geq 2$, which costs as much as sampling with $\Omega(n)$ samples.}

\cutforsubmission{These results treat the number of hidden layers $\numhiddenlayers$ as a constant, and choose the maximum cumulant order $\maxcumulantorder$ based on the error tolerance $\varepsilon$. Expressed in terms of these parameters, we conjecture that our algorithms have variance-normalized mean squared error at most $c_\maxcumulantorder(\numhiddenlayers/n)^\maxcumulantorder$, although the dependence on $\numhiddenlayers$ is difficult to analyze precisely. The runtime of the basic version of our algorithm is at most $c_\maxcumulantorder^\prime\numhiddenlayers n^{\maxcumulantorder+1}$, and the runtime of the factorized version is at most $c_\maxcumulantorder^{\prime\prime}\numhiddenlayers^2n^\maxcumulantorder$ whenever $\maxcumulantorder\geq 3$. Here, $c_\maxcumulantorder$, $c_\maxcumulantorder^\prime$ and $c_\maxcumulantorder^{\prime\prime}$ are constants that depend only on $\maxcumulantorder$.}

We provide full proofs of both results in the technical supplement, Appendix \ref{supplementappendix}. We believe that with additional effort, both results can be extended from polynomials to general activation functions.\footnote{We make the mild assumptions that the activation functions are measurable and polynomially bounded.}

\section{Empirical results}\label{empiricalresultssection}

\subsection{Experimental setup}

We empirically validate the performance of all 4 versions of our algorithm (basic, augmented, factorized and factorized augmented) on ReLU MLPs with a variety of depths and widths. We vary the width $n$ from 4 to 256 and the number of hidden layers $\numhiddenlayers$ from 2 to 12 (using the same network for each choice of $\numhiddenlayers$, but stopping after different numbers of layers). We vary the maximum cumulant order $\maxcumulantorder$ from 1 to 4. For each configuration, we calculate mean squared error (MSE) over all output neurons and over 5 different random seeds.

To calculate MSE, we estimate ground truth using Monte Carlo sampling with $2^{34}\approx2\times10^{10}$ samples, which greatly exceeds any of our algorithms' computational budgets. We use these same samples to estimate the variance of a single forward pass, and divide MSE by this variance to obtain \textit{variance-normalized MSE}. This choice of units means that Monte Carlo sampling with $N$ samples has a variance-normalized MSE of exactly $\frac 1N$.

\begin{figure}[t]
  \centering
  \makebox[0pt][c]{\hspace*{1.2
  cm}\scalebox{0.7}{\input{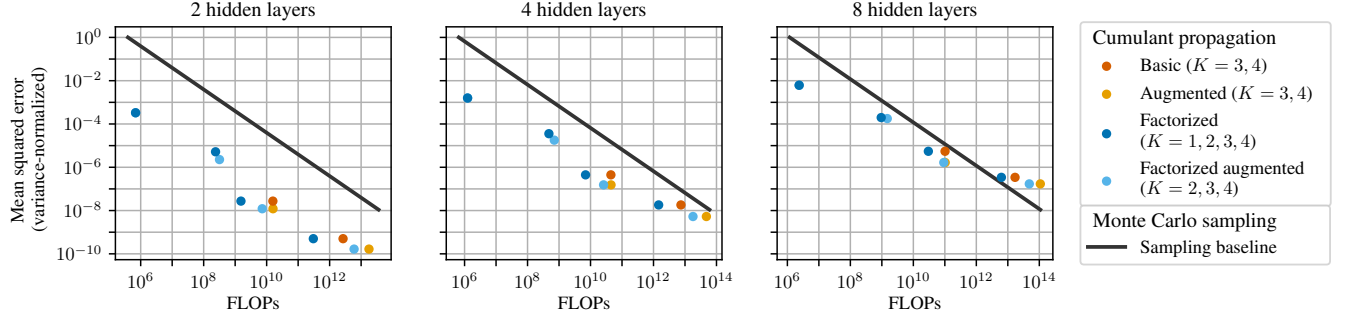}}}
  \caption{Performance of all versions of our algorithm for different networks of width 256 and $\maxcumulantorder$ varying from 1 to 4. All 4 versions are identical for $\maxcumulantorder=1$, and the factorized versions are identical to their unfactorized counterparts for $\maxcumulantorder=2$. (Error bars hidden because they would be too small to be readable.)}
  \label{msevsflopsdepthfigure}
\end{figure}

We measure the computational budget of our algorithms in floating-point operations (FLOPs), which we count by applying an extension of PyTorch's FLOP counter to our Python implementation. However, since cumulant tensors are symmetric, naively taking cumulant tensor einsums results in redundant computation. Rather than implementing kernels for symmetric tensor einsums, we simply adjust our counts to exclude the redundant FLOPs. We describe this adjustment in more detail in Appendix \ref{flopappendix}.

\subsection{Comparison with Monte Carlo sampling}

We compare the mean squared error of our sample-free estimation procedures to that of Monte Carlo sampling. Our results at width 256 with 2, 4 and 8 hidden layers are shown in Figure \ref{msevsflopsdepthfigure}, and our results at additional widths and depths are given in Appendix \ref{empiricalappendix}.

The factorized and factorized augmented versions of our algorithm outperform Monte Carlo sampling by a factor of around 1,000 for 2 hidden layers, and around 10--100 for 4 hidden layers. They start to underperform sampling for 8 hidden layers once the maximum cumulant order $\maxcumulantorder$ reaches 4, although by our theoretical results, they would outperform sampling once again if we were to increase the width further.

We also tested the performance of Quasi-Monte Carlo sampling based on scrambled Sobol sequences, as implemented in SciPy \citep{scipyqmc}. The results are also shown in Figure \ref{msevsflopsdepthfigure}. Quasi-Monte Carlo sampling offers a modest benefit of around a factor of 1--10 over simple Monte Carlo sampling, significantly less than the improvement offered by our algorithms.

\subsection{Scaling with width and depth}

The central finding of our theoretical analysis is that the mean squared error of our algorithm is $O\!\left(\frac 1{n^\maxcumulantorder}\right)$, where $n$ is the width and $\maxcumulantorder$ is the maximum cumulant order. To validate this finding explicitly, we plot the mean squared error of the basic version of algorithm as a function of the width in Figure \ref{msevswidthfigure}(\subref{msevswidthsubfigure}). As expected, the mean squared error closely follows a curve proportional to $\frac 1{n^\maxcumulantorder}$ when $n$ is large. Note, however, that increasing $\maxcumulantorder$ can \textit{worsen} performance at small values of $n$. We discuss scaling with depth in Appendix \ref{depthappendix}.

\subsection{Power cumulant ablation}\label{ablationsubsection}

Our theoretical error analysis was critical to the design of our algorithms. In particular, it led directly to the two adjustments described in Section \ref{kpropsubsection}: the use of power cumulants, and the inclusion of the full trace of the $\left(\maxcumulantorder+1\right)\nth$-order cumulant tensor for odd $\maxcumulantorder$. To validate the necessity of these adjustments, we check the performance of our algorithm without them. We plot the mean squared error of this ablated version of our algorithm as a function of the width $n$ in Figure \ref{msevswidthfigure}(\subref{msevswidthablationsubfigure}). As predicted by our theoretical analysis, the mean squared error is $O(1)$ instead of $O\!\left(\frac 1{n^\maxcumulantorder}\right)$, and so the ablated version of our algorithm no longer matches the performance of sampling in the large-width limit. We provide further details of this ablation, including a more direct comparison to Monte Carlo sampling, in Appendix \ref{ablationappendix}.

\begin{figure}[t]
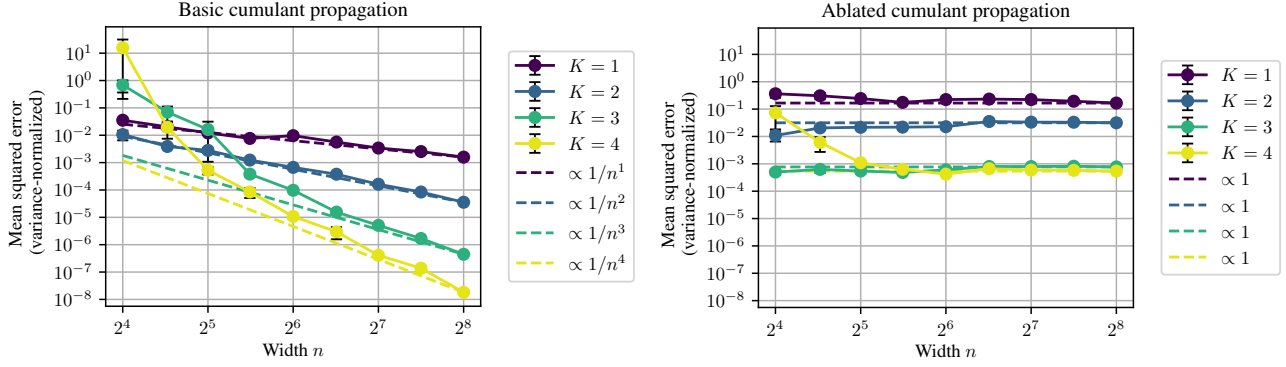

  \centering
  \makebox[\textwidth][c]{%
    \begin{subfigure}{0.5\textwidth}
      \centering
      \hspace*{-1.25cm}\scalebox{0.7}{\input{figures/mse_vs_width.pgf}}
      \caption{The mean squared error of the basic (equivalently, factorized) version of our algorithm is $O\!\left(\frac 1{n^\maxcumulantorder}\right)$.}
      \label{msevswidthsubfigure}
    \end{subfigure}\hspace{8mm}%
    \begin{subfigure}{0.5\textwidth}
      \centering
      \hspace*{-0.4cm}\scalebox{0.7}{\input{figures/mse_vs_width_ablation.pgf}}
      \caption{The mean squared error of the power cumulant ablation of our algorithm is $O(1)$ (horizontal).}
      \label{msevswidthablationsubfigure}
    \end{subfigure}
  }
  \caption{Width scaling of the basic and ablated versions of our algorithm on networks with 4 hidden layers. Dashed lines are those with the predicted large-$n$ slope passing through the rightmost point. Error bars show $\pm 1$ standard error over 5 random seeds.}
  \label{msevswidthfigure}
\end{figure}

\subsection{Other activation functions}

Our algorithm extends easily to activation functions other than $\operatorname{ReLU}$. To demonstrate this, we check the performance of our algorithm on $\operatorname{GELU}$ \citep{gelu} (a popular alternative to $\operatorname{ReLU}$) and $\tanh$ (to test a bounded activation function). Our results are given in Appendix \ref{otheractivationsappendix}. Results for both activation functions are qualitatively similar to our results for $\operatorname{ReLU}$.

\subsection{Low probability estimation}\label{lpeexperimentsubsection}

As we discuss later in Section \ref{lpediscussionsubsection}, we are especially interested in outperforming sampling at estimating low probabilities. To test the performance of our algorithms at this task, we replace the final $\operatorname{ReLU}$ activation in our networks with the threshold function $\mathbbm 1_{z>3}$, so that the final expected value can be interpreted as a low probability. We compare the factorized version of our algorithm with maximum cumulant order $\maxcumulantorder=3$ to Monte Carlo sampling with a comparable number of FLOPs, as well as three other baselines that are explained in Appendix \ref{lpebaselineappendix}.

Our results at width 256 with 4 hidden layers (3 ReLUs and one threshold function) are shown in Figure \ref{msevsprobfigure}. This shows root mean squared error relative to root mean squared probability, with probabilities bucketed into intervals $[10^{k-\frac 12},10^{k+\frac 12})$ for integers $k$. Monte Carlo sampling with $N$ samples is worse than simply guessing zero once the probability drops below $1/N$. On the other hand, using a similar number of FLOPs to Monte Carlo sampling with $N=5,695$ samples, our algorithm achieves a relative error of under 30\% for probabilities 100 times lower than $1/N$, and continues to outperform all baselines until the probabilities are nearly 10,000 lower.

\subsection{Mechanistic distillation}\label{mechdistsubsection}

Since our algorithms produce differentiable estimates, they can be used to train networks by applying gradient descent to the estimate of the expected loss. We demonstrate this approach by using it to distill a random teacher MLP into a student MLP of a different width, a process we refer to as \textit{mechanistic distillation}. We compare this to ordinary distillation using stochastic gradient descent with a comparable number of FLOPs per gradient step.

Our results for a random teacher MLP with width 16 and a student MLP with width 8 are shown in Figure \ref{mechdistfigure}. Further details of our experimental setup are given in Appendix \ref{mechdistappendix}. Mechanistic distillation works well, but is slightly less efficient than distillation using stochastic gradient descent. This may be because our mechanistic estimation algorithms only have good guarantees for random networks, not partially trained ones. We discuss how our algorithms might be adapted to partially-trained networks in Section \ref{lpediscussionsubsection}.

\begin{figure}[t]
  \centering
  \makebox[\textwidth][c]{%
    \begin{minipage}[t]{0.56\textwidth}
      \centering
      \hspace*{-2cm}\scalebox{0.7}{\input{figures/mse_vs_prob.pgf}}
      \captionof{figure}{Performance of the factorized version of our algorithm with $\maxcumulantorder=3$ at low probability estimation. Error bars show $\pm 1$ standard error of the mean squared error over the random seeds used in each bucket. There were 205 seeds used in total across the 7 buckets.}
      \label{msevsprobfigure}
    \end{minipage}\hspace{8mm}%
    \begin{minipage}[t]{0.44\textwidth}
      \centering
      \scalebox{0.7}{\input{figures/mechanistic_distillation.pgf}}
      \captionof{figure}{Learning curves for distillation of a random teacher MLP with width 16 into a student MLP with width 8, both with 4 hidden layers.}
      \label{mechdistfigure}
    \end{minipage}%
  }
\end{figure}

\subsection{Wall-clock time}

By default, we measure computational cost in FLOPs, to minimize sensitivity to implementation details. Typically, almost all of the FLOPs both in our algorithms and in Monte Carlo sampling are part of matrix multiplications (or more general einsums), so FLOP counts should be representative of the performance of optimized implementations. However, in practice, our algorithms also include a number of additional FLOPs that, while being a small fraction of the total, take up the majority of the wall-clock time. Although we expect that this issue could be resolved with additional implementation effort, for the interested reader, we show mean squared error against wall-clock time in Appendix \ref{walltimeappendix}.

\section{Discussion}\label{discussionsection}

\subsection{Outperforming random sampling using mechanistic estimation}

We have shown that for wide random MLPs, cumulant propagation outperforms random sampling for estimating the expected loss. This is a striking result, since sampling is the strong default for such problems. Instead of using sampling, cumulant propagation directly analyzes how the weights transform the input distribution. We refer to such methods as \textit{mechanistic}, since they work by reasoning about the mechanisms of the network, in much the same way as methods from the field of mechanistic interpretability \citep{mechinterpcommunity}.

Following \citet{mspblog}, we conjecture that the performance of \textit{any} estimation procedure can be matched (or exceeded) by a mechanistic one, providing we are given the training trajectory or selection process of any parameters involved. The intuition behind this is that it should never be better to blindly check things than to choose what to check more deliberately. This conjecture is bold: it would imply, for example, that the win rate of one chess-playing network over another can be estimated well without having the networks play chess against each other, by reasoning deductively about the rules of chess and the strengths and weaknesses of different strategies. Nevertheless, the algorithms in this paper give a first glimpse of how this might be possible.

\subsection{Mechanistic training and low probability estimation}\label{lpediscussionsubsection}

The obvious application of mechanistic loss estimation is to train neural networks, by applying gradient descent to the mechanistic loss estimate. We refer to this as \textit{mechanistic training}. We have provided a proof of concept of this by demonstrating mechanistic distillation. However, since our method only has good guarantees for networks at initialization, it would not work well for the entirety of training. To adapt our method to partially-trained networks, we may need to adopt an approach akin to the ``structure versus randomness'' dichotomy \citep{structurevsrandomness}: instead of treating the weights as purely random, we could treat them as having a structured component that has been selected to give the model low loss, together with a random component. We leave further elaboration of this idea to future work.

Mechanistic training does not only offer to match or improve training efficiency: it should also produce models that generalize very differently. For example, consider a loss function with a certain ``dangerous'' event that is very rare but has very high loss. Stochastic gradient descent may fail to sample the event even once during training, resulting in a model that does not account for the dangerous event at all. On the other hand, mechanistic training could employ much better estimates for the probability of the rare event, resulting in a model that avoids it far more reliably, despite having potentially never seen the event happen. This could greatly improve safety, especially in the presence of distribution shifts that make the dangerous event more likely. This application of mechanistic estimation is what motivates our study of low probability estimation, and is discussed further in \citet{lpe}.


\ifpreprint{
\subsection{Open problems}

There is enormous scope for further research on mechanistic estimation. Even for random MLPs, we do not know whether it is possible to match the performance of random sampling at large depth (i.e., asymptotically as the number of hidden layers $\numhiddenlayers$ tends to infinity), or for narrow networks (i.e., asymptotically as the error tolerance $\varepsilon\to 0$ while the width $n$ is held fixed). We could also ask the same questions for MLPs with weights drawn from non-independent or non-Gaussian distributions.

Beyond random MLPs, learned MLPs pose a key challenge for more sophisticated applications of mechanistic training, along with the related case of adversarially-chosen MLPs. Both of these cases seem likely to require access to the training trajectory or adversarial selection process, in order to identify the relevant structure in a ``structure versus randomness'' approach. Another key challenge is to extend beyond closed-form input distributions to empirical input distributions for which we only have black-box access.

Approaching the problem from a more theoretical perspective, the broad conjecture that \textit{any} estimation procedure can be matched (or exceeded) by a mechanistic one could be investigated in a variety of settings, such as matrix permanents, expectations of Boolean circuits, average behavior of cellular automata, or a variety of other computational problems expressible as expectations.

}{
\subsection{Open problems}

For a discussion of directions for further research, see the arXiv version of the paper.
}

\newcommand{\relatedwork}{

\subsection{Covariance propagation for neural networks}

A very similar approach is taken by \citet{kpropmnist}, who explain how cumulant propagation can be applied to MLPs using Hermite expansions (or equivalently Gram--Charlier expansions, in their terminology). Furthermore, they use class-conditional cumulants of the input distribution to perform mechanistic training on small models in settings such as MNIST. However, their implementation and error analysis focuses on covariance propagation, i.e., the special case of cumulant propagation with maximum cumulant order 2, which corresponds to approximating activations as multivariate Gaussian distributions. A similar approach is also taken by \citet{exactcovprop}, who also study covariance propagation, but with exact calculations for off-diagonal covariances as well as variances. Our error analysis for random networks shows that the effect of these exact off-diagonal covariance calculations is lower-order in the width of the network, and would therefore only have a minimal effect on mean squared error in our setting. We build on both of these works by extending covariance propagation to arbitrarily high cumulant order, in the setting of wide random networks: providing an explicit formulation of the algorithm, checking its performance empirically, and conducting the corresponding error analysis.

\subsection{Analytic uncertainty propagation}

There is a body of work in a similar vein that is motivated by the problem of analytically propagating an uncertain input through a network. The most direct precursor to our own work in this area is that of \cite{covprop}, who derive formulas for covariance propagation that can be applied to ReLU networks. Indeed, Theorem 1 of \citet{covprop} is a special case of our diagram summation formula, as explained in Appendix \ref{diagramscovpropsubappendix}. Their work builds in turn upon work of \citet{meanvarprop}, who derive formulas for propagating means and variances only, but not off-diagonal covariances. Thus our work generalizes these methods to arbitrarily high cumulant order, and shows moreover that this achieves arbitrarily high accuracy in the large-width limit. Separately, \citet{covpropzero} derive an alternative covariance propagation formula specifically for ReLU networks with 1 hidden layer. This whole line of work dates at least as far back as \citet{freyhinton}, who already provided analytic formulas for the mean and variance of ReLU and several other activation functions on a Gaussian input distribution.

\subsection{Wide random networks}

There is an extensive literature on the behavior of neural networks with random weights in the limit as the width tends to infinity, as detailed in \citet*{pdlt}. Some of the key ideas in this area include the view of infinitely-wide random networks as Gaussian processes \citep{nngpold,nngp,nngpnew} and the Neural Tangent Kernel associated with the training of infinitely-wide networks \citep{ntk}. Notably, \citet{nngpks} studies the cumulants of outputs of random networks in the infinite-width limit, and \citet{dunbaraaronson} studies how the covariances decay with depth. However, this body of work focuses almost exclusively on the output distribution \textit{for fixed inputs over random weights}, whereas we study the output distribution \textit{for fixed weights over random inputs}. This distinction is critical: instead of studying properties of the network \textit{ensemble}, we study algorithms that take in the weights of a \textit{particular} network.

\subsection{Mechanistic estimation and low probability estimation}

The question of whether there is a general framework for deductive, probabilistic estimation was posed by \citet{poi}, who also introduced cumulant propagation in the context of arithmetic circuits. \citet{lpe} studied both mechanistic and non-mechanistic methods for low probability estimation, and their \textit{activation extrapolation} methods have features in common with our low probability estimation methods. The broad conjecture that the performance of any estimation procedure can be matched or exceeded by a mechanistic one was originally posed by \citet{mspblog}.

}

\section{Related work}

\ifpreprint{\relatedwork}{
For a literature review, see the first appendix.



}

\cutforsubmission{

\section{Conclusion}

We have demonstrated both theoretically and empirically that our sample-free method for estimating the expected loss of wide MLPs at initialization is substantially more efficient than Monte Carlo sampling. With further work, this could lead to greatly improved training efficiency for losses dominated by rare events.

\section{Acknowledgments}

We are grateful to Eric Neyman for work that helped to inspire this project, to Ekene Ezeunala for work on mechanistic estimation of the Hermite coefficients of $\operatorname{GELU}$, and to Harshita Khera for invaluable operational support. We are also grateful to Matthew Ding, Geoffrey Irving, Vinayak Pathak, Jess Riedel, Linus Tang and Gabriel Wu for comments on drafts, as well as to the anonymous reviewers for helpful feedback on the first version of the paper.

The Alignment Research Center received funding for this work from the UK AI Security Institute's Alignment Project. WW was partially supported by the National Science Foundation Graduate Research Fellowship under Grant No.\ DGE-2040434. Any opinions, findings, and conclusions or recommendations expressed in this material are those of the authors and do not necessarily reflect the views of the National Science Foundation. MW acknowledges the Institute for Advanced Study, and DOE grant DE-SC0009988.
}

\bibliographystyle{abbrvnat}
\bibliography{bibliography}

\clearpage

\appendix

\ifpreprint{}{
\section*{Related work}
\begingroup
\renewcommand{\thesection}{R}
\renewcommand{\theHsection}{R}
\relatedwork
\endgroup
\clearpage
}

\section{Diagram summation formula for cumulants}\label{diagramsappendix}

The diagram summation formula for cumulants is our key tool for propagating cumulants (as defined in Section \ref{cumulantssubsection}) through non-linear activation functions. Given a random variable $Z\in\mathbb R^n$ and an activation function $\phi:\mathbb R\to\mathbb R$, this formula expresses the cumulants of $\phi\left(Z_1\right),\dots,\phi\left(Z_n\right)$ in terms of the cumulants of $Z_1,\dots,Z_n$ and the Hermite coefficients of $\phi$ (as defined in Section \ref{hermitesubsection}). In order to state it, we begin by defining diagrams and their attributes.

\begin{definition}
\label{def:diagram}

Let $\underline k:=(k_1,\ldots,k_r)$ be an $r$-tuple of non-negative integers. A \textit{diagram} on $\underline k$ is a partition $\pi$ of the set
\[V:=\bigcup_{v=1}^r\{(v,1),(v,2),\dots,(v,k_v)\}.\]

We say that $\pi$ is connected to mean that for any two elements $(v_{\text{start}},w_{\text{start}}),(v_{\text{end}},w_{\text{end}})\in V$, there is a sequence $(v_1,w_1),\dots,(v_m,w_m)\in V$ with $(v_1,w_1)=(v_{\text{start}},w_{\text{start}})$, $(v_m,w_m)=(v_{\text{end}},w_{\text{end}})$ and for all $i=1,\dots,m-1$, either $v_i=v_{i+1}$ or $(v_i,w_i)$ and $(v_{i+1},w_{i+1})$ lie in the same part of $\pi$. In other words, every two elements of $V$ are linked by a path for which consecutive points lie either in the same one of the $r$ sets used to define $V$ or in the same part of $\pi$. (In the degenerate case in which $k_v=0$ for some $v$, we consider $\pi$ to be connected if and only if $r=1$.) The set of all connected diagrams on $\underline{k}$ is denoted $\mathrm{cDia}(\underline{k})$.




We say that $\pi$ is \textit{$\left[\leq M\right]$-mixed} to mean that every part $S\in\pi$ of size at most $M$ satisfies $\left\{\left(v_1,w_1\right),\left(v_2,w_2\right)\right\}\subseteq S$ for some $v_1\neq v_2$.

The set of all connected $\left[\leq M\right]$-mixed diagrams on $\underline k$ is denoted $\cDia{M}(\underline k)$.
\end{definition}

We visualize diagrams by drawing a vertex for each pair $\left(v,w\right)$, encircling the vertex set $\left\{\left(v,w\right):w\in\left\{1,\dots,k_v\right\}\right\}$ for each $v$, and drawing each part of the partition as a hyperedge. We draw some examples of diagrams and explain whether they are connected or $\left[\leq 2\right]$-mixed in Figure \ref{diagramsfigure}.

\tikzset{
  cumulant/.style={black, thick},
  vertex/.style={circle, fill=black, inner sep=1.0pt}
}

\newcommand{\drawskeleton}{%
  \coordinate (TL) at (0, 2.8);
  \coordinate (TR) at (2.8, 2.8);
  \coordinate (BL) at (0, 0);
  \coordinate (BR) at (2.8, 0);
  \draw[thick] (TL) ellipse[x radius=0.50cm, y radius=0.22cm, rotate=45];
  \draw[thick] (TR) ellipse[x radius=0.42cm, y radius=0.19cm, rotate=-45];
  \draw[thick] (BL) ellipse[x radius=0.32cm, y radius=0.15cm, rotate=-45];
  \draw[thick] (BR) ellipse[x radius=0.50cm, y radius=0.22cm, rotate=45];
  \pgfmathsetmacro{\ux}{0.707}
  \pgfmathsetmacro{\uy}{0.707}
  \node[vertex] (TL-top) at ($(TL)+(\ux*0.28, \uy*0.28)$) {};
  \node[vertex] (TL-mid) at (TL) {};
  \node[vertex] (TL-bot) at ($(TL)+(-\ux*0.28, -\uy*0.28)$) {};
  \node[vertex] (TR-top) at ($(TR)+(-\ux*0.20, \uy*0.20)$) {};
  \node[vertex] (TR-bot) at ($(TR)+(\ux*0.20, -\uy*0.20)$) {};
  \node[vertex] (BL-dot) at (BL) {};
  \node[vertex] (BR-top) at ($(BR)+(\ux*0.28, \uy*0.28)$) {};
  \node[vertex] (BR-mid) at (BR) {};
  \node[vertex] (BR-bot) at ($(BR)+(-\ux*0.28, -\uy*0.28)$) {};
  \node[left=8pt] at ($(TL)+(-\ux*0.35, -\uy*0.02)$) {$\hat{\phi}_3$};
  \node[right=8pt] at ($(TR)+(\ux*0.30, -\uy*0.02)$) {$\hat{\phi}_2$};
  \node[left=8pt] at ($(BL)+(-\ux*0.22, \uy*0.02)$) {$\hat{\phi}_1$};
  \node[right=8pt] at ($(BR)+(\ux*0.35, \uy*0.02)$) {$\hat{\phi}_3$};
}

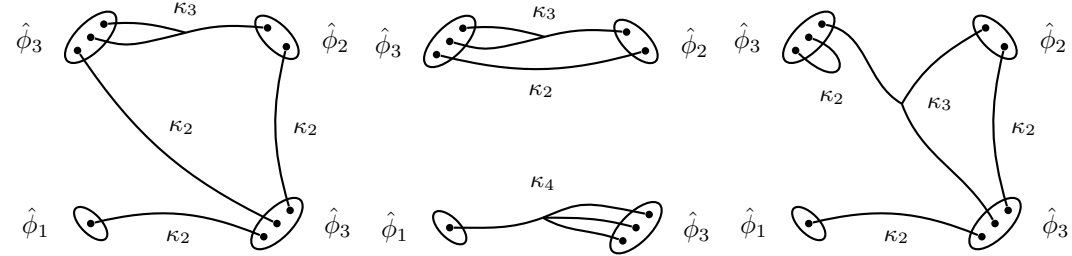
\begin{figure}[h]
\centering
\begin{subfigure}[b]{0.32\textwidth}
\centering
\begin{tikzpicture}[scale=0.88]
  \drawskeleton

  \coordinate (K3junc) at ($(TL-top)!0.5!(TR-top) + (0,-0.10)$);
  \draw[cumulant] (TL-top) to[out=5, in=165] (K3junc);
  \draw[cumulant] (TL-mid) .. controls (0.55, 2.65) .. (K3junc);
  \draw[cumulant] (TR-top) to[out=175, in=15] (K3junc);
  \node[above=3pt] at (K3junc) {\small $\kappa_3$};

  \draw[cumulant] (TL-bot) to[bend right=15]
    node[pos=0.35, right=8pt] {\small $\kappa_2$} (BR-mid);

  \draw[cumulant] (TR-bot) to[bend right=15]
    node[pos=0.5, right=3pt] {\small $\kappa_2$} (BR-top);

  \draw[cumulant] (BL-dot) to[bend left=18]
    node[midway, below=2.5pt] {\small $\kappa_2$} (BR-bot);
\end{tikzpicture}
\caption{The diagram is both connected and $\left[\leq 2\right]$-mixed, and so the term \textbf{is} included.}
\end{subfigure}
\hfill
\begin{subfigure}[b]{0.32\textwidth}
\centering
\begin{tikzpicture}[scale=0.88]
  \drawskeleton

  \coordinate (K3junc) at ($(TL-top)!0.5!(TR-top) + (0,-0.10)$);
  \draw[cumulant] (TL-top) to[out=5, in=165] (K3junc);
  \draw[cumulant] (TL-mid) .. controls (0.55, 2.65) .. (K3junc);
  \draw[cumulant] (TR-top) to[out=175, in=15] (K3junc);
  \node[above=3pt] at (K3junc) {\small $\kappa_3$};

  \draw[cumulant] (TL-bot) to[bend right=15]
    node[midway, below=2pt] {\small $\kappa_2$} (TR-bot);

  \coordinate (K4junc) at (1.4, 0.15);
  \draw[cumulant] (BL-dot) to[out=0, in=200] (K4junc);
  \draw[cumulant] (BR-top) .. controls (2, 0.4) .. (K4junc);
  \draw[cumulant] (BR-mid) to[out=165, in=0] (K4junc);
  \draw[cumulant] (BR-bot) to[out=160, in=-30] (K4junc);
  \node[above=7pt] at (K4junc) {\small $\kappa_4$};
\end{tikzpicture}
\caption{The diagram is $\left[\leq 2\right]$-mixed but not connected, and so the term \textbf{is not} included.}
\end{subfigure}
\hfill
\begin{subfigure}[b]{0.32\textwidth}
\centering
\begin{tikzpicture}[scale=0.88]
  \drawskeleton

  \coordinate (K3junc) at (1.4, 1.8);
  \draw[cumulant] (TL-top) to[out=-10, in=150] (K3junc);
  \draw[cumulant] (TR-top) to[out=-150, in=60] (K3junc);
  \draw[cumulant] (BR-mid) to[out=120, in=-70] (K3junc);
  \node[right=6pt] at (K3junc) {\small $\kappa_3$};

  \draw[cumulant] (TL-mid) .. controls (0.70, 2.55) and (0.55, 1.90) .. (TL-bot);
  \node at (0.35, 1.9) {\small $\kappa_2$};

  \draw[cumulant] (TR-bot) to[bend right=15]
    node[midway, right=3pt] {\small $\kappa_2$} (BR-top);

  \draw[cumulant] (BL-dot) to[bend left=18]
    node[midway, below=2.5pt] {\small $\kappa_2$} (BR-bot);
\end{tikzpicture}
\caption{The diagram is connected but not $\left[\leq 2\right]$-mixed (see top left loop), and so the term \textbf{is not} included.}
\end{subfigure}

\caption{Drawings of diagrams corresponding to included and excluded terms from the Hermite-based diagram summation formula, for $\left(k_1,k_2,k_3,k_4\right)=\left(3,2,3,1\right)$ (clockwise from top left). Hermite coefficients correspond to encircled vertex sets and cumulants correspond to hyperedges.}
\label{diagramsfigure}
\end{figure}

Our diagrams are essentially cluster-expansion diagrams from statistical physics \citep{clusterexpansions}, except that we generalize pairwise edges to hyperedges in order to account for third- and higher-order cumulants.

With these definitions in place, together with the notation for Hermite coefficients introduced in Section \ref{hermitesubsection} (equation ~\eqref{eq:hermite-def}), we may now state the formula. We refer to the formula as \textit{Hermite-based} to distinguish it from an analogous but simpler \textit{polynomial} diagram summation formula, which we introduce in Appendix \ref{sec:polypolyproof}.\footnote{We state the diagram summation formula here for polynomial activation functions, but it also applies to non-polynomials under certain analytic conditions. We discuss this in more detail in Appendix \ref{diagramsproof}.}

\begin{restatable}[Hermite-based diagram summation formula for cumulants]{theorem}{diagramsummationformula}\label{diagramstheorem}
Let $Z\in\mathbb R^n$ be a random variable with finite moments and let $\phi:\mathbb R\to\mathbb R$ be a polynomial. Then
\[
\begin{aligned}
\kappa\left[\phi\left(Z_{j_1}\right),\dots,\phi\left(Z_{j_r}\right)\right]=\sum_{\underline{k}\in \mathbb{Z}_{\ge0}^r}\left(\prod_{v=1}^r\frac 1{k_v!}\,\widehat\phi^{\left(\mathbb E\left[Z_{j_v}\right],\operatorname{Var}\left[Z_{j_v}\right]\right)}_{k_v}\right)\sum_{\pi \in \mathrm{cDia}_{[\leq2]}(\underline{k})}\prod_{S\in\pi}\kappa\left[Z_{j_v}:\left(v,w\right)\in S\right]
\end{aligned}.
\]
\end{restatable}

We prove Theorem \ref{diagramstheorem} in Appendix \ref{diagramsproof}.

Note that in this formula, we do not use a single Hermite expansion of $\phi$, but rather take separate Hermite expansions of $\phi$ with respect to Gaussian distributions with means and variances matching each of $Z_1,\dots,Z_n$.

The special case of this formula in which $\phi$ is the step function $\mathbbm 1_{z\geq 0}$ and $Z$ is multivariate Gaussian appears in \citet{tetrachoric} (see \citet[equation (3.2)]{genzbretz}).

\subsection{Covariance propagation formula}\label{diagramscovpropsubappendix}

To help understand the Hermite-based diagram summation formula for cumulants, it is instructive to consider the special case applicable to covariance propagation, i.e., cumulant propagation with maximum cumulant order $\maxcumulantorder=2$.

\begin{corollary}[\citet{covprop}]\label{covpropcorollary}
Let $Y\in\mathbb R^n$ be a multivariate Gaussian random variable and let $\phi:\mathbb R\to\mathbb R$ be a polynomial. Then
\[\operatorname{Cov}\left[\phi\left(Y_{j_1}\right),\phi\left(Y_{j_2}\right)\right]=\sum_{k=1}^\infty\frac 1{k!}\;\widehat\phi^{\left(\mathbb E\left[Y_{j_1}\right],\operatorname{Var}\left[Y_{j_1}\right]\right)}_k\,\widehat\phi^{\left(\mathbb E\left[Y_{j_2}\right],\operatorname{Var}\left[Y_{j_2}\right]\right)}_k\operatorname{Cov}\left[Y_{j_1},Y_{j_2}\right]^k.\]
\end{corollary}

This formula exactly matches the covariance propagation formula of \citet[Theorem 1]{covprop}
using the integration-by-parts identity
\[\widehat\phi^{\left(\mu,\sigma^2\right)}_k:=\frac 1{\sigma^k}\mathbb E\left[\phi\left(Y\right)\operatorname{He}_k\left(\frac{Y-\mu}\sigma\right)\right]=\mathbb E\left[\phi^{\left(k\right)}\left(Y\right)\right]=\frac{\partial^k}{\partial\mu^k}\mathbb E\left[\phi\left(Y\right)\right]\]
for $Y\sim\mathcal N\left(\mu,\sigma^2\right)$.

\begin{proof}
Apply Theorem \ref{diagramstheorem} in the case $r=2$ to $Y$, giving an expression for $\operatorname{Cov}\left[\phi\left(Y_{j_1}\right),\phi\left(Y_{j_2}\right)\right]$. Since $Y$ is multivariate Gaussian, its third- and higher-order cumulants vanish, and so the only diagrams on $\left(k_1,k_2\right)$ that survive have the form $\left\{\left\{\left(1,\rho\left(w\right)\right),\left(2,w\right)\right\}:w\in\left\{1,\dots,k_1\right\}\right\}$ for some permutation $\rho$ of $\left\{1,\dots,k_1\right\}$, with $k_1=k_2$. The number of such permutations is $k_1!$, which cancels out one of the factors of $\frac 1{k_1!}$. Finally, the connectedness condition eliminates the diagram with $k_1=k_2=0$, leaving a sum starting at $k_1=k_2=1$.
\end{proof}

\clearpage

\section{Mean and covariance propagation}\label{meancovpropappendix}

In this Appendix we provide pseudocode for the basic version of our cumulant propagation algorithm in the cases $\maxcumulantorder=1$ (\textit{mean propagation}) and $\maxcumulantorder=2$ (\textit{covariance propagation}), where $\maxcumulantorder$ is the maximum cumulant order. We also indicate roughly how these special cases derive from our general algorithm, and how the ablated versions differ. A precise formulation of our general algorithm is significantly more involved, and is given in Appendix \ref{kpropfull}.

\subsection{Pseudocode}

Mean propagation tracks both the mean and the trace of the covariance matrix of activations.\footnote{Our usage of the term \textit{mean propagation} differs from that of \citet{poi}, who only track the mean, like in our ablated version of mean propagation.} Tracking the entire diagonal of the covariance matrix would have lower mean squared error, but we provide the version that is consistent with our more general formulation. Pseudocode is as follows.

\begin{algorithm}[H]
\caption{Mean propagation ($\maxcumulantorder=1$)}
\label{meanpropalgorithm}
\begin{algorithmic}[1]
\Require weights $\mathbf W^{\left(1\right)},\dots,\mathbf W^{\left(\numhiddenlayers+1\right)}\in\mathbb R^{n\times n}$, activation functions $\phi_1,\dots,\phi_\numhiddenlayers:\mathbb R\to\mathbb R$
\State $\boldsymbol\mu\gets\mathbf 0\in\mathbb R^n$
\State $\sigma^2_{\text{mean}}\gets 1\in\mathbb R$
\For{$\layernumber=1,\dots,\numhiddenlayers+1$}
    \State $\boldsymbol\mu\gets\mathbf W^{\left(\layernumber\right)}\boldsymbol\mu$
    \State $\sigma^2_i\gets\sigma^2_{\text{mean}}\sum_{j=1}^n\left(W^{\left(\layernumber\right)}_{ij}\right)^2$  \textbf{for} $i=1,\dots,n$
    \If{$\layernumber\leq\numhiddenlayers$}
        \State $a_i\gets\mathbb E_{Y\sim\mathcal N\left(\mu_i,\sigma_i^2\right)}\left[\phi_\layernumber\left(Y\right)\right]$ \textbf{for} $i=1,\dots,n$
        \State $b_i\gets\mathbb E_{Y\sim\mathcal N\left(\mu_i,\sigma_i^2\right)}\left[\phi_\layernumber\left(Y\right)^2\right]$ \textbf{for} $i=1,\dots,n$
        \State $\mu_i\gets a_i$ \textbf{for} $i=1,\dots,n$
        \State $\sigma^2_{\text{mean}}\gets\frac 1n\sum_{i=1}^n\left(b_i-a_i^2\right)$
    \EndIf
\EndFor
\State \Return $\boldsymbol\mu$
\end{algorithmic}
\end{algorithm}

Covariance propagation tracks both the mean and the full covariance matrix of activations. Pseudocode is as follows.

\begin{algorithm}[H]
\caption{Covariance propagation ($\maxcumulantorder=2$)}
\label{covpropalgorithm}
\begin{algorithmic}[1]
\Require weights $\mathbf W^{\left(1\right)},\dots,\mathbf W^{\left(\numhiddenlayers+1\right)}\in\mathbb R^{n\times n}$, activation functions $\phi_1,\dots,\phi_\numhiddenlayers:\mathbb R\to\mathbb R$
\State $\boldsymbol\mu\gets\mathbf 0\in\mathbb R^n$
\State $\boldsymbol\Sigma\gets\mathbf I_n\in\mathbb R^{n\times n}$
\For{$\layernumber=1,\dots,\numhiddenlayers+1$}
    \State $\boldsymbol\mu\gets\mathbf W^{\left(\layernumber\right)}\boldsymbol\mu$ \label{linearlinestart}
    \State $\boldsymbol\Sigma\gets\mathbf W^{\left(\layernumber\right)}\boldsymbol\Sigma\left(\mathbf W^{\left(\layernumber\right)}\right)^\top$ \label{linearlineend}
    \If{$\layernumber\leq\numhiddenlayers$}
        \State $a_i\gets\mathbb E_{Y\sim\mathcal N\left(\mu_i,\Sigma_{ii}\right)}\left[\phi_\layernumber\left(Y\right)\right]$ \textbf{for} $i=1,\dots,n$
        \State $b_i\gets\mathbb E_{Y\sim\mathcal N\left(\mu_i,\Sigma_{ii}\right)}\left[\phi_\layernumber\left(Y\right)^2\right]$ \textbf{for} $i=1,\dots,n$
        \State $c_i\gets\frac 1{\sqrt{\Sigma_{ii}}}\mathbb E_{Y\sim\mathcal N\left(\mu_i,\Sigma_{ii}\right)}\left[\phi_\layernumber\left(Y\right)\left(\frac{Y-\mu_i}{\sqrt{\Sigma_{ii}}}\right)\right]$ \textbf{for} $i=1,\dots,n$ \label{covariancefirstline}
        \State $\mu_i\gets a_i$ \textbf{for} $i=1,\dots,n$
        \State $\Sigma_{ii}\gets b_i-a_i^2$ \textbf{for} $i=1,\dots,n$
        \State $\Sigma_{ij}\gets\Sigma_{ij}c_ic_j$ \textbf{for} $i,j=1,\dots,n$ \textbf{if} $i\neq j$ \label{covariancesecondline}
    \EndIf
\EndFor
\State \Return $\boldsymbol\mu$
\end{algorithmic}
\end{algorithm}

When the activation functions $\phi_1,\dots,\phi_\numhiddenlayers$ are piecewise polynomials, such as when they are all $\operatorname{ReLU}$, the Gaussian expectations in these algorithms can be computed straightforwardly.
Each expectation can be decomposed into a finite sum of integrals of polynomials against the standard Gaussian density over finite intervals. By repeated integration by parts, these integrals can be expressed in terms of the standard Gaussian density function and cumulative distribution function evaluated at the interval endpoints. For explicit formulas in the case of $\operatorname{ReLU}$, see See Proposition \ref{prop:hermite-relu}, and for an alternative approach that we use for $\tanh$ and $\operatorname{GELU}$, see Appendix \ref{quadraturesubsection}.

\subsection{Derivation from cumulant propagation}\label{meancovpropderivationsubappendix}

It is instructive to consider how mean propagation and covariance propagation derive from the more general formulation of our basic cumulant propagation algorithm.

The matrix multiplication step (lines \ref{linearlinestart}--\ref{linearlineend} in both algorithms) is straightforward. If $X\in\mathbb R^n$ is a random variable with mean $\boldsymbol\mu$ and covariance matrix $\boldsymbol\Sigma$, then $\mathbf WX$ has mean $\mathbf W\boldsymbol\mu$ and covariance matrix $\mathbf W\boldsymbol\Sigma\mathbf W^\top$. The higher cumulants of $X$ transform similarly. In mean propagation, we avoid computing the full covariance matrix explicitly, and instead approximate $\boldsymbol\Sigma$ by $\sigma^2_{\text{mean}}\mathbf I_n$.

For the activation function step, we use a Gaussian approximation. Given a random variable $Z\in\mathbb R^n$ and an activation function $\phi:\mathbb R\to\mathbb R$ (to be applied coordinatewise), we approximate the mean and covariance matrix of $\phi\left(Z\right)$ by the mean and covariance matrix of $\phi\left(Y\right)$, where $Y\in\mathbb R^n$ is a multivariate Gaussian with the same mean and covariance matrix as $Z$.

To compute the means and variances of $\phi\left(Y\right)$, we only need to evaluate one-dimensional integrals, and these can be computed exactly when $\phi$ is a piecewise polynomial. Using power cumulants allows us to take advantage of this by handling these variances separately from the off-diagonal covariances. For the off-diagonal covariances of $\phi\left(Y\right)$ (lines \ref{covariancefirstline} and \ref{covariancesecondline} of covariance propagation), we use the leading-order term of the formula given in Corollary \ref{covpropcorollary}.

\subsection{Ablated mean and covariance propagation}\label{meancovpropablationsubappendix}

To help explain the ablated version of cumulant propagation described in Appendix \ref{ablationappendix}, we describe the changes that would need to be made to mean and covariance propagation to obtain their ablated versions.

For mean propagation, the ablated version no longer tracks the trace of the covariance matrix. This amounts to omitting lines 2, 5, 8 and 10, and replacing $\sigma_i^2$ by the default value $1$ in line 7. Since only the mean is tracked, the use of power cumulants makes no difference.

For covariance propagation, the ablated version no longer uses the power cumulants of the outputs of activation functions. This amounts to omitting lines 8 and 11, and removing the condition $i\neq j$ from line 12.

\clearpage

\section{Additional widths and depths}\label{empiricalappendix}

Figure \ref{msevsflopsfullfigure} compares the mean squared error of our sample-free estimation procedures to that of Monte Carlo sampling at different widths from 16 to 256 and different numbers of hidden layers from 2 to 12. Our procedures are exact when there is only 1 hidden layer.

Our algorithms perform poorly at low width, especially when the number of hidden layers $\numhiddenlayers$ is large, and can even get \textit{worse} with the maximum cumulant order $\maxcumulantorder$ in this setting. However, for any fixed $\numhiddenlayers$ and $\maxcumulantorder$, our algorithms match or exceed the performance of sampling as the width grows.

Note also that at low depth, factorized algorithms use fewer FLOPs than their unfactorized counterparts, but at high enough depth (for fixed width), they use more FLOPs.

\begin{figure}[h]
  \centering
  \makebox[0pt][c]{\hspace*{2cm}\scalebox{0.7}{\input{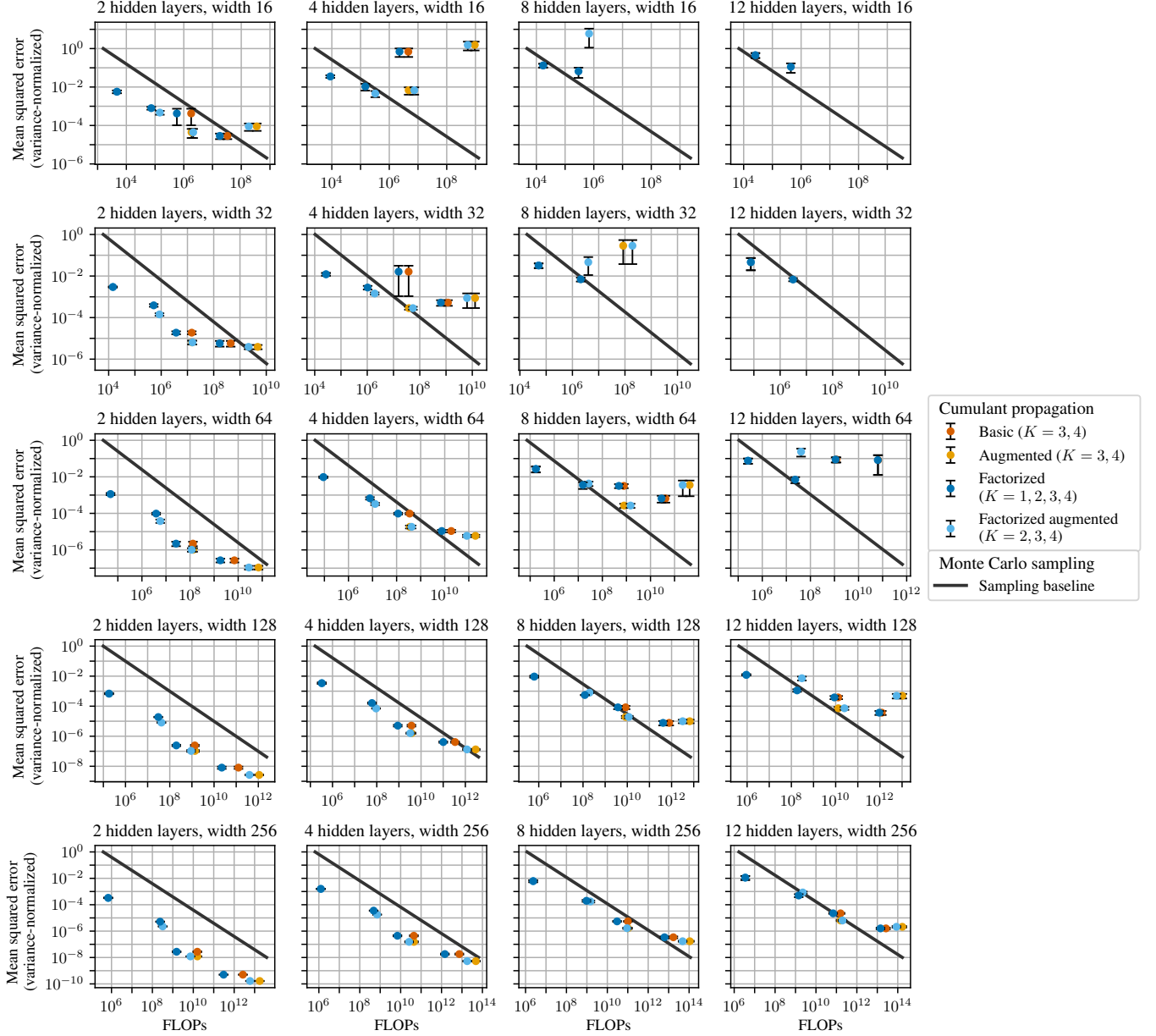}}}
  \caption{Performance of all versions of our algorithm for networks of different widths and depths and $\maxcumulantorder$ varying from 1 to 4. Error bars show $\pm 1$ standard error over 5 random seeds. Missing points have variance-normalized mean squared error greater than 10.}
  \label{msevsflopsfullfigure}
\end{figure}

\clearpage

\section{Scaling with depth}\label{depthappendix}

As discussed in Section \ref{theoreticalresultssection}, we conjecture that our algorithms have mean squared error at most $c_\maxcumulantorder(\numhiddenlayers/n)^\maxcumulantorder$, where $n$ is the width, $\numhiddenlayers$ is the number of hidden layers, $\maxcumulantorder$ is the maximum cumulant order and $c_\maxcumulantorder$ is a constant that depends only on $\maxcumulantorder$. We conjecture this depth scaling based on a heuristic analysis of the error, but we can also test it empirically. To do this, we plot the mean squared error of the basic version of algorithm as a function of the number of hidden layers in Figure \ref{msevsdepthfigure}. As conjectured, the mean squared error roughly follows a curve proportional to $\numhiddenlayers^\maxcumulantorder$ when $\numhiddenlayers$ is large.

\begin{figure}[h]
  \centering
  \makebox[0pt][c]{\scalebox{0.7}{\input{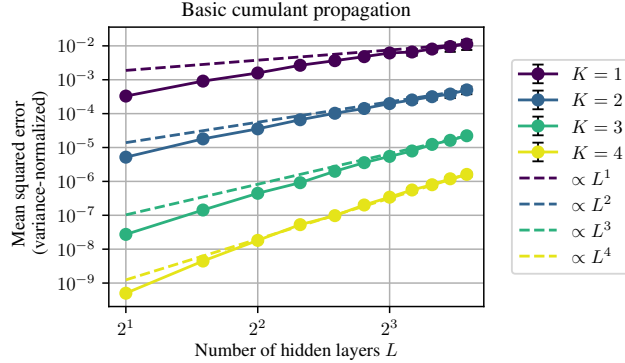}}}
  \caption{Depth scaling of the basic (equivalently, factorized) version of our algorithm on networks with width 256. Dashed lines are those with the conjectured large-$n$ slope passing through the rightmost point. Error bars show $\pm 1$ standard error over 5 random seeds.}
  \label{msevsdepthfigure}
\end{figure}

This depth scaling is worse than Monte Carlo sampling, whose error does not increase with depth. We believe that a sample-free algorithm can be developed whose error also does not increase with depth either, but we leave this problem to future work.

\clearpage

\section{Power cumulant ablation}\label{ablationappendix}

As discussed in Section \ref{ablationsubsection}, we checked the performance of an ablated version of our algorithm without the two adjustments described in Section \ref{kpropsubsection}. Specifically, the differences between the basic and ablated versions of our algorithm are as follows:
\begin{enumerate}
\item
As discussed in Section \ref{diagramssummarysubsection}, we use the Hermite-based diagram summation formula for cumulants, Theorem \ref{diagramstheorem}, to compute the cumulants of $\phi\left(Z_1\right),\dots,\phi\left(Z_n\right)$ from the cumulants of $Z_1,\dots,Z_n$ for a non-linear activation function $\phi:\mathbb R\to\mathbb R$. The basic version of our algorithm uses this formula to compute the power cumulants of $\phi\left(Z_1\right),\dots,\phi\left(Z_n\right)$, as defined in Section \ref{kpropsubsection}, by taking Hermite expansions of \textit{powers} of $\phi$. It then uses these power cumulants to obtain the cumulants of $\phi\left(Z_1\right),\dots,\phi\left(Z_n\right)$. The ablated version of our algorithm instead uses the Hermite-based diagram summation formula to compute the cumulants of $\phi\left(Z_1\right),\dots,\phi\left(Z_n\right)$ directly, only taking Hermite expansions of $\phi$ itself.
\item
When the maximum cumulant order $\maxcumulantorder$ is odd, the basic version of our algorithm tracks the full trace of the $\left(\maxcumulantorder+1\right)\nth$-order cumulant tensor, as defined in Section \ref{kpropsubsection}. The ablated version of our algorithm instead approximates this as zero. When $\maxcumulantorder$ is even, there is no difference between the two algorithms in this respect.
\end{enumerate}

To help explain precisely what we mean by this, we explain how to ablate the pseudocode for $\maxcumulantorder=1$ (mean propagation) and $\maxcumulantorder=2$ (covariance propagation) in Appendix \ref{meancovpropablationsubappendix}.

We compared the performance of an \textit{ablated factorized} version of our algorithm to the unablated factorized versions and to Monte Carlo sampling. Our results for networks with width 256 and different numbers of hidden layers from 2 to 12 are shown in Figure \ref{msevsflopsablationfigure}. As expected, the ablated version is significantly worse than the unablated versions, and is no longer competitive with Monte Carlo sampling. The scaling of the ablated version of our algorithm with width (which is unaffected by factorization, since this only affects the runtime) is also discussed in Section \ref{ablationsubsection}.

\begin{figure}[h]
  \centering
  \makebox[0pt][c]{\hspace*{1cm}\scalebox{0.7}{\input{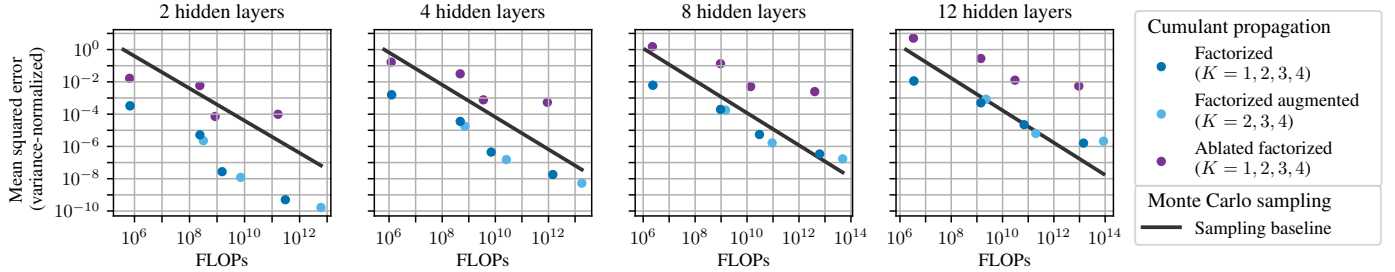}}}
  \caption{Performance of the ablated factorized version of our algorithm, as well as the unablated factorized versions, for different networks of width 256. Mean over 5 random seeds shown. (Error bars hidden because they would be too small to be readable.)}
  \label{msevsflopsablationfigure}
\end{figure}

\clearpage

\section{Other activation functions}\label{otheractivationsappendix}

In this Appendix, we validate the performance of our algorithms for activation functions other than $\operatorname{ReLU}$. We consider two additional activation functions:

\begin{itemize}
\item
$\operatorname{GELU}\left(z\right):=z\Phi\left(z\right)$, where $\Phi$ is the standard Gaussian cumulative distribution function \citep{gelu}
\item
$\tanh\left(z\right):=\frac {1-e^{-2z}}{1+e^{-2z}}$
\end{itemize}

For the remainder of this Appendix, fix the number of hidden layers $\numhiddenlayers\geq 0$.

\subsection{Bias terms}

For $\operatorname{GELU}$, we consider multilayer perceptrons (MLPs) without bias terms, as defined in Definition \ref{mlpdefinition}.

For $\tanh$, if we were to consider MLPs without bias terms, then the expectation of every output neuron would be zero by symmetry, making the estimation problem trivial. We therefore incorporate bias terms, meaning that the weights of the MLP are $\boldsymbol\theta=\left(\mathbf W^{\left(1\right)},\dots,\mathbf W^{\left(\numhiddenlayers+1\right)},\mathbf b^{\left(1\right)},\dots,\mathbf b^{\left(\numhiddenlayers+1\right)}\right)$ with $\mathbf W^{\left(\layernumber\right)}\in\mathbb R^{n\times n}$ and $\mathbf b^{\left(\layernumber\right)}\in\mathbb R^n$ for $\layernumber=1,\dots,\numhiddenlayers+1$, and $\mathbf b^{\left(\layernumber\right)}$ is added to the $\layernumber\nth$ hidden layer's activation vector immediately after multiplying it by $\mathbf W^{\left(\layernumber\right)}$.

It is easy to adjust our sample-free estimation procedures to incorporate these bias terms. Our cumulant propagation algorithms simply add the bias vector to the estimate of the mean at the corresponding layer. Equivalently, this can be thought of as using a different activation function for each neuron that has been shifted by the corresponding bias component.

\subsection{Weight initialization}

For $\operatorname{GELU}$, we use He initialization \citep{heinit}, which ensures that when $n$ is large and the inputs have mean 0 and variance 1, all pre-activations and outputs also have mean 0 and variance $O(1)$ (treating the depth as a constant). This is defined for MLPs without bias terms as follows.

\begin{definition}
Consider an MLP with activation functions $\phi_1,\dots,\phi_\numhiddenlayers:\mathbb R\to\mathbb R$ and weights $\boldsymbol\theta=\left(\mathbf W^{\left(1\right)},\dots,\mathbf W^{\left(\numhiddenlayers+1\right)}\right)\in\mathbb R^{\left(\numhiddenlayers+1\right)\times n \times n}$. We say that $\boldsymbol\theta$ is \textit{He-initialized} to mean that $\mathbf W^{\left(1\right)},\dots,\mathbf W^{\left(\numhiddenlayers+1\right)}$ are independent random variables with
\begin{align*}
W^{\left(1\right)}_{ij}\overset{\text{i.i.d.}}\sim&\,\mathcal N\left(0,\frac 1n\right)\\
W^{\left(\layernumber+1\right)}_{ij}\overset{\text{i.i.d.}}\sim&\,\mathcal N\left(0,\frac 1{n\,\mathbb E_{Z\sim\mathcal N\left(0,1\right)}[\phi_\layernumber\left(Z\right)^2]}\right),\quad\layernumber=1,\dots,\numhiddenlayers
\end{align*}
for $i,j=1,\dots,n$.
\end{definition}

For $\tanh$, we use a \textit{critical initialization} \citep{Poole2016Exponential,Schoenholz2017Deep,Pennington17Resurrecting}, defined for MLPs with bias terms as follows.

\begin{definition}
Consider an MLP and with activation functions $\phi_1,\dots,\phi_\numhiddenlayers:\mathbb R\to\mathbb R$ and weights $\boldsymbol\theta=\left(\mathbf W^{\left(1\right)},\dots,\mathbf W^{\left(\numhiddenlayers+1\right)},\mathbf b^{\left(1\right)},\dots,\mathbf b^{\left(\numhiddenlayers+1\right)}\right)$. We say that $\boldsymbol\theta$ is \textit{critical at variance $q^*$} to mean that $\mathbf W^{\left(1\right)},\dots,\mathbf W^{\left(\numhiddenlayers+1\right)},\mathbf b^{(1)},\dots,\mathbf b^{(\numhiddenlayers+1)}$ are independent random variables with
\begin{align*}
W^{\left(1\right)}_{ij}&\overset{\text{i.i.d.}}\sim\,\mathcal N\left(0,\frac {q^*}n\right)\\
b_i^{(1)}&=0\\
W^{\left(\layernumber+1\right)}_{ij}&\overset{\text{i.i.d.}}\sim\,\mathcal N\left(0,\frac {\sigma_{\layernumber+1,\mathrm w}^2}{n}\right),&\layernumber=1,\dots,\numhiddenlayers\\
b^{\left(\layernumber+1\right)}_{i}&\overset{\text{i.i.d.}}\sim\,\mathcal N\left(0,\sigma_{\layernumber+1,\mathrm b}^2\right),&\layernumber=1,\dots,\numhiddenlayers
\end{align*}
for $i,j=1,\dots,n$,
where the initialization scales 
$\sigma_{\layernumber,\mathrm w}$ and $\sigma_{\layernumber,\mathrm b}$
are chosen to satisfy
\begin{align*}
    q^*&=\sigma_{\layernumber+1,\mathrm w}^2\mathbb{E}[\phi_\layernumber(\sqrt{q^*}Z)]+\sigma_{\layernumber+1,\mathrm b}^2,&\layernumber=1,\dots,\numhiddenlayers\\
    1&=\sigma^2_{\layernumber+1,\mathrm w} \mathbb{E}[\phi_\layernumber'(\sqrt{q^*} Z)^2],&\layernumber=1,\dots,\numhiddenlayers
\end{align*}
for $Z\sim\mathcal{N}(0,1)$.
\end{definition}

For $\tanh$, we use critical initialization at $q^*=0.85$, corresponding to $\sigma_{\mathrm w}^2=2.025$ and $\sigma_{\mathrm b}^2=0.111$. 
This is the ``large'' $\tanh$ setting used in \citet[\S 3]{Pennington17Resurrecting}.

As an aside, note that He initialization and critical initialization at $q^*=1$ coincide for $\operatorname{ReLU}$, 
both giving weight variance $\frac 2n$ and zero bias.
Thus our main $\operatorname{ReLU}$ experiments can be thought of as being performed in either initialization.

\subsection{Hermite coefficient estimation}\label{quadraturesubsection}

In order to estimate Hermite coefficients, as defined in Section \ref{hermitesubsection}, our sample-free estimation procedures must estimate Gaussian expectations of a power of the activation function multiplied by a polynomial. For both $\operatorname{GELU}$ and $\tanh$, this amounts to estimating $\mathbb E_{Z\sim\mathcal N\left(0,1\right)}\left[\psi\left(Z\right)\right]$ for some $\psi:\mathbb R\to\mathbb R$ that is analytic and of polynomial growth on an infinite strip around the real axis. A standard approach in such cases is Gauss--Hermite quadrature \citep{Wang2025Convergence}, which produces the estimate
\begin{align*}
    \mathbb E_{Z\sim\mathcal N\left(0,1\right)} [\psi\left(Z\right)]&\approx \sum_{i=1}^m \frac{(m-1)!}m\frac 1{\operatorname{He}_{m-1}(x_i)^2}\psi(x_i),
\end{align*}
where $\operatorname{He}_{m-1}$ is the $(m-1)\nth$ probabilist's Hermite polynomial and $x_1,\ldots, x_m$ are its distinct real roots.\footnote{This method is arguably not ``sample-free'' and hence not a ``mechanistic'' way to estimate this one-dimensional integral, but we defer a resolution to this issue to future work.}
For our experiments, we take the number of roots to be $m=100$.

We also use the same technique to estimate the Gaussian expectations used in the definition of He initialization and critical initialization.

\subsection{Results}

\begin{figure}[t]
  \centering
  \makebox[0pt][c]{\hspace*{1cm}\scalebox{0.7}{\input{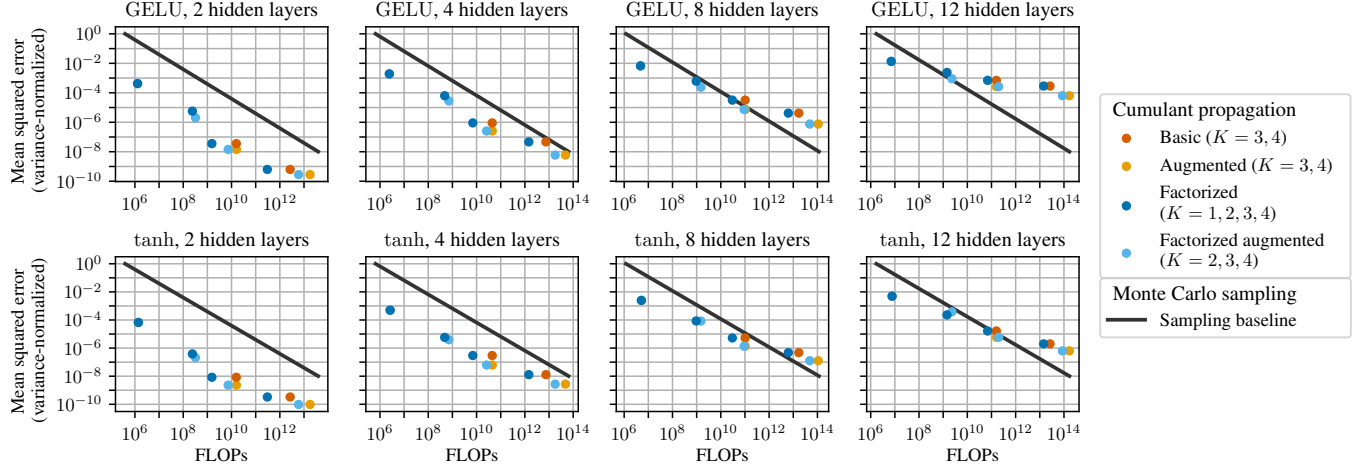}}}
  \caption{Performance of all versions of our algorithm for different GELU and tanh networks of width 256. Mean over 5 random seeds shown. (Error bars hidden because they would be too small to be readable.)}
  \label{msevsflopsotheractsfigure}
\end{figure}

Figure \ref{msevsflopsotheractsfigure} compares the mean squared error of our sample-free estimation procedures to that of Monte Carlo sampling for GELU and tanh networks of width 256 with different numbers of hidden layers from 2 to 12. Our results are qualitatively similar to our results for ReLU.

\clearpage

\section{Low probability estimation baselines}\label{lpebaselineappendix}

Our results on low probability estimation are given in Section \ref{lpeexperimentsubsection}. Here we explain our four baseline methods to which we compare cumulant propagation.

\begin{enumerate}
\item
\textbf{Monte Carlo sampling.} This is our usual baseline algorithm. However, it performs very poorly once the probability drops below $1/(\text{number of samples})$, since the squared error is very large in the unlikely event that there is a positive sample.
\item
\textbf{Monte Carlo sampling with Gaussian extrapolation.} We use Monte Carlo sampling to estimate the mean and variance of each input to the final threshold function $\mathbbm 1_{z>3}$, and output the probability that a Gaussian with that mean and variance would exceed the threshold of $3$. This method is similar to the Gaussian Logit Difference method of \citet{lpe}. Empirically, this method performs similarly to using cumulant propagation to estimate the mean and variance of each of these neurons, i.e., most of the error comes from failing to account for the third- and higher-order cumulants of the inputs to the final threshold function.
\item
\textbf{Monte Carlo sampling with cumulant extrapolation.} We use Monte Carlo sampling to estimate the mean, variance and third cumulant of each input to the final threshold function $\mathbbm 1_{z>3}$. We then use our usual cumulant propagation algorithm, involving Hermite expansions, to propagate these through the final threshold function. The purpose of this baseline is to check whether it is better to use sampling or cumulant propagation to compute these cumulants.
\item
\textbf{Zero.} This estimate simply outputs zero, which always has a relative error of 100\%. We include this because it outperforms naive Monte Carlo sampling at low probabilities.
\end{enumerate}

Our empirical results show that once the probability drops below $1/(\text{number of samples})$, cumulant propagation performs best, followed by Monte Carlo sampling with cumulant extrapolation, followed by Monte Carlo sampling with Gaussian extrapolation, followed by zero, followed by naive Monte Carlo sampling. This shows that computing the third-order cumulants of the inputs to the threshold function is helpful for low probability estimation, and moreover that cumulant propagation computes these more accurately than Monte Carlo sampling. At sufficiently small probabilities, however, none of the methods significantly outperform zero.

\clearpage

\section{Mechanistic distillation setup}\label{mechdistappendix}

Our results on mechanistic distillation are given in Section \ref{mechdistsubsection}. Here we provide further details of our experimental setup.

We distill a random teacher MLP $M_{\theta_T}$ with hidden dimension $n_T$ and $\numhiddenlayers$ hidden layers into a student MLP $M_{\theta_S}$ with hidden dimension $n_S$ and the same number of hidden layers. Both networks have input and output dimension $n$, use $\operatorname{ReLU}$ activation functions and are initialized using He initialization \citep{heinit}. We optimize the student parameter $\theta_S$ to minimize the loss
\[\frac 1n\sum_{i=1}^n\mathbb E_{X\sim\mathcal N\left(0,\mathbf I_n\right)}\left[\left(M_{\theta_T}\left(X\right)_i-M_{\theta_S}\left(X\right)_i\right)^2\right].\]
In our experiments, we take $n=n_T=16$, $n_S=8$ and $\numhiddenlayers=4$, since the learning problem can be challenging for much larger $n_T$.

We train the student network in two different ways. The first way uses ordinary stochastic gradient descent, which takes steps of gradient descent on estimates of the loss produced using Monte Carlo sampling. The second way uses mechanistic training, which takes steps of gradient descent on estimates of the loss produced by cumulant propagation. In our experiments, we use the basic version of our cumulant propagation algorithm with maximum cumulant order $\maxcumulantorder=2$, which has good mean squared error even for widths as low as 8. For stochastic gradient descent, we use a batch size of 36, which uses a comparable number of FLOPs per gradient step. In both cases, we train for 256 gradient steps and tune the learning rate to within a factor of 2 to optimize the final loss, which amounts to a learning rate of 0.25 for stochastic gradient descent and 0.5 for mechanistic training.

To estimate the above loss with cumulant propagation, we stack the hidden layer activations of the teacher and student networks into vectors of length $n_T+n_S$. This allows the output of both networks to be computed as the output of a single network whose $\layernumber\nth$ weight matrix is the block diagonal matrix
\[\mathbf W^{\left(\layernumber\right)}=\begin{pmatrix}\mathbf W^{\left(\layernumber\right)}_T&\mathbf 0\\\mathbf 0&\mathbf W^{\left(\layernumber\right)}_S\end{pmatrix},\]
where $\mathbf W^{\left(\layernumber\right)}_T$ and $\mathbf W^{\left(\layernumber\right)}_S$ are the $\layernumber\nth$ weight matrices of the teacher and student networks respectively. We run cumulant propagation to estimate the mean and covariance matrix of this combined network, starting from a Gaussian input with covariance matrix
\[\begin{pmatrix}\mathbf I_n&\mathbf I_n\\\mathbf I_n&\mathbf I_n\end{pmatrix},\]
corresponding to a duplicated input. We then estimate the loss in terms of the mean and covariance matrix of the output using the identity
\[\mathbb E[\left(A-B\right)^2]=\left(\mathbb E\left[A\right]-\mathbb E\left[B\right]\right)^2+\operatorname{Var}\left(A\right)+\operatorname{Var}\left(B\right)-2\operatorname{Cov}\left(A,B\right).\]

Implemented naively, this computation involves redundant computation because of the zeros in the weight matrices, but in order to determine the appropriate batch size for stochastic gradient descent, we adjust the FLOP count to exclude this redundant computation, in a similar manner to Appendix \ref{flopappendix}.

Our results are shown in Figure \ref{mechdistfigure} and are discussed in Section \ref{mechdistsubsection}. The average variance of an output neuron is 0.23, so both forms of training are producing some non-trivial feature matching. It is also interesting to compare the mechanistic loss estimate used for mechanistic training to the stochastic loss estimate for the same training run produced using Monte Carlo sampling: the mechanistic loss estimate has bias but no variance, whereas the stochastic loss estimate has variance but no bias. This may favor stochastic gradient descent, which can accumulate uncorrelated error across updates. We leave a remedy to this problem for mechanistic training to future work.

\clearpage

\section{Wall-clock time}\label{walltimeappendix}

We compare the mean squared error of all 4 versions of our algorithm as a function of wall-clock time to that of Monte Carlo sampling. Our implementations were written in PyTorch and run on B200 GPUs. Our results for networks with width 256 and different numbers of hidden layers from 2 to 12 are shown in Figure \ref{msevstimefigure}.

\begin{figure}[h]
  \centering
  \makebox[0pt][c]{\hspace*{1cm}\scalebox{0.7}{\input{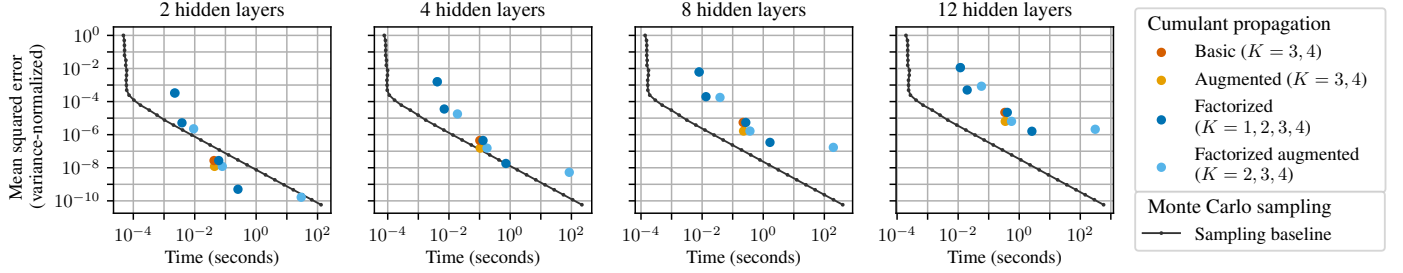}}}
  \caption{Wall-clock performance of all versions of our algorithm for different networks of width 256. Mean over 5 random seeds shown. (Error bars hidden because they would be too small to be readable.)}
  \label{msevstimefigure}
\end{figure}

Our algorithms still outperform sampling for networks with 2 hidden layers, but their relative performance is generally worse than when runtime is measured in FLOPs. This appears to be largely because our algorithms spend the majority of the wall-clock time on a small fraction of the FLOPs that are implemented inefficiently in pure Python. In addition, they perform a significant amount of redundant computation, as described in Appendix \ref{flopappendix}. We expect that with more implementation effort, wall-clock time could be made to track FLOP count much more closely, especially at large width, but we leave this to future work.

\clearpage

\section{FLOP counting}\label{flopappendix}

To measure the FLOPs used by both our cumulant propagation algorithm and sampling baselines, we use an extension of PyTorch's \texttt{FlopCounterMode} with the modifications listed below.

\paragraph{Entrywise operation tracking:} Unlike sampling baselines, our cumulant propagation algorithm expends a large portion of its FLOPs in entrywise operations during the nonlinear step.\footnote{This is a practical weakness of using the Hermite expansion for the nonlinearity which we believe can be corrected in future work.}
Thus, for an accurate comparison of FLOPs, we add tracking for entrywise arithmetic operations (\texttt{+-*/}), which are not counted in the native \texttt{FlopCounterMode}.

\paragraph{Mock symmetric tensor kernels:} Cumulant tensors are symmetric, but our proof-of-concept implementation of cumulant propagation uses full \texttt{torch.Tensor} objects. 
To simulate the number of FLOPs that would be consumed by an implementation with efficient symmetric tensor CUDA kernels, we multiply FLOP counts for elementwise operations on symmetric tensors by the adjustment factor
\begin{equation*}
    \alpha_{n,d}:=\dfrac{\dim\syms{d}{n}}{\dim \tens{d}{n}}=\binom{n+d-1}{d}n^{-d}.
\end{equation*}
(See Appendix~\ref{diagonalsprereq} for notation.)

More generally, given an integer partition $\lambda=(1^{c_1}2^{c_2}\dots d^{c_d})\vdash d$ with $c=\sum_i c_i$ blocks,
the adjustment factor for $\lambda$-symmetric tensors is
\begin{equation*}
    \alpha_{\pi,n}:=\dfrac{\dim\syms{\lambda}{n}}{\dim \tens{d}{n}}=\prod_{i=1}^d \binom{n+c_i-1}{c_i}n^{-c_i}.
\end{equation*}
Now consider the symmetric contraction operation $-\odot \mathbf W\colon \syms{d}{n}\to\syms{d}{n}$.
The optimal path is to contract the weight matrices $\mathbf W$'s into the $d$-tensor one at a time.
Over full tensors, $\mathbf W$-contraction takes $2n^{d+1}$ FLOPs: there are $n^d$ output entries, and each takes $2n$ FLOPs to compute (counting add and multiply separately).
Thus $-\odot \mathbf W$ takes $2d n^{d+1}$ FLOPs total.
Over symmetric tensors, observe that the intermediate after contracting in the $i\nth$ $W$ is $(1^i 2^{d-i})$-symmetric, and thus only needs $\dim \Sigma^{(1^i 2^{d-i})}(\mathbb{R}^n)=\alpha_{(1^i 2^{d-i})}n^d$ entries to represent.
The total number of FLOPs over symmetric tensors can thus be multiplied by the adjustment factor
\begin{equation*}
    \beta_{d,n} := \dfrac{1}{d}\sum_{i=1}^d \alpha_{(1^i 2^{d-i}),n}=\dfrac{1}{d n^d}\sum_{i=1}^d \binom{n+i-1}{i}\binom{n+d-i-1}{d-i}.
\end{equation*}
As $n\to\infty$, the adjustment factors approach
\begin{align*}
    \alpha_{\pi,n} &\to \alpha_{\pi,\infty}:=\prod_{i=1}^d \dfrac{1}{c_i!}\\
    \beta_{d,n}&\to\beta_{d,\infty}:=\dfrac{1}{d}\sum_{i=1}^d \dfrac{1}{i!(d-i)!}.
\end{align*}
For reference,
\begin{equation*}
    \beta_{1,\infty}=1,\quad \beta_{2,\infty}=\dfrac{3}{4}, \quad \beta_{3,\infty}=\dfrac{7}{18},\quad \beta_{4,\infty}=\dfrac{5}{32}.
\end{equation*}

Note that there may be algorithms for performing symmetric tensor contractions even more efficiently than this \citep{symcontract}, but we disregard these for the sake of fairness and simplicity.

\paragraph{Polynomial fit:} Because the FLOP count for cumulant propagation through a depth $L$ MLP is a polynomial in $n$ and $L$ with leading term $n^{\maxcumulantorder+1} L$ for the unfactored version (and $n^{\maxcumulantorder} L^2$ for factored), we can simply fit a polynomial of the specified degree to the FLOP counts for a few small values of $n,L$ to find the exact counts for any values.
We check that these fit polynomials are exact on held out values of $n,L$.
See Table~\ref{flop-table} for the leading order terms of the FLOP polynomials.

\bgroup
\begin{table}[t]
\centering
\makebox[\textwidth][c]{%
\def\arraystretch{1.3}
\begin{tabular}{lllll}
\multicolumn{1}{l}{$\maxcumulantorder$}  &\multicolumn{1}{l}{Basic} &\multicolumn{1}{l}{Augmented} &\multicolumn{1}{l}{Factorized} &\multicolumn{1}{l}{Factorized augmented} \\ \hline
1 & $4n^2\numhiddenlayers + 66n\numhiddenlayers + \ldots$ & \\
2 & $7n^3\numhiddenlayers + 26n^2\numhiddenlayers +\ldots$ & $11n^3\numhiddenlayers + 150n^2\numhiddenlayers +\ldots$ \\
3 & $\frac{7}{3}n^4\numhiddenlayers+252n^3\numhiddenlayers+\ldots$ & $\frac{7}{3}n^4\numhiddenlayers+260n^3\numhiddenlayers+\ldots$ & $30n^3\numhiddenlayers^2+39n^3\numhiddenlayers+\ldots$ & $72n^3\numhiddenlayers^2+264n^3\numhiddenlayers+\ldots$ \\
4 & $\frac{5}{4}n^5\numhiddenlayers+224n^4\numhiddenlayers+\ldots$ & $\frac{5}{4}n^5\numhiddenlayers+3083n^4\numhiddenlayers+\ldots$ & $24n^4\numhiddenlayers^2 + \frac{79}{3}n^4\numhiddenlayers +\ldots$ & $60n^4\numhiddenlayers^2 + 1069n^4\numhiddenlayers +\ldots$ \\
\end{tabular}%
}
\vspace{1em}
\caption{Cumulant propagation FLOP counts expressed as polynomials for a $\operatorname{ReLU}$ MLP with width $n$ and $\numhiddenlayers$ hidden layers. All 4 versions are identical for $\maxcumulantorder=1$, and the factorized versions are identical to their unfactorized counterparts for $\maxcumulantorder=2$. Fractional coefficients arise essentially because the number of distinct entries of a symmetric tensor is a polynomial with fractional coefficients (e.g. $\frac 12n^2+\frac 12n$ for a symmetric matrix).}
\label{flop-table}
\end{table}
\egroup

\clearpage

\FloatBarrier

\renewcommand{\thesection}{S}
\renewcommand{\theHsection}{S}
\section{Technical supplement}\label{supplementappendix}

\edef\supplementsectionoffset{\number\value{section}}
\renewcommand{\thesection}{S.\number\numexpr\value{section}-\supplementsectionoffset\relax}
\renewcommand{\theHsection}{S.\number\numexpr\value{section}-\supplementsectionoffset\relax}

{\LARGE\bfseries Cumulant propagation algorithms and analysis}

\startcontents[supp]
\section*{\contentsname}
\begingroup
\makeatletter
\def\numberline#1{\hb@xt@2em{#1\hfil}}%
\makeatother
\printcontents[supp]{}{1}[1]{}
\endgroup

\section{Introduction}

In this technical supplement, we provide precise descriptions of all the algorithms described in Section \ref{methodssection}, and prove all of our theoretical results: the results of Section \ref{theoreticalresultssection} on the performance of our algorithms, and the Hermite-based diagram summation formula for cumulants given in Appendix \ref{diagramsappendix}.

To begin with, recall the main definition and theorems of the paper:

\mlpdef*

\mainthm*

\beatingthm*

More specifically, we prove \ref{matchingsamplingtheorem} by taking the estimation procedure to be the basic cumulant propagation algorithm described in Section \ref{methodssection}, and Theorem \ref{beatingsamplingtheorem} by taking it to be the factorized version described in Section \ref{factorizedsubsection}.

In order to better convey the core ideas behind these proofs, we introduce an additional, simplified algorithm that avoids the use of Hermite expansions, which we refer to as the \textit{polynomial algorithm}. This algorithm is sufficient to prove Theorem \ref{matchingsamplingtheorem} (and its factorized version for Theorem \ref{beatingsamplingtheorem}), but it is of limited practical interest, because it can only be applied to polynomial activation functions. By contrast, the \textit{Hermite-based} algorithms given in Section \ref{methodssection} can be applied to any measurable, polynomially bounded activation functions, although analytic difficulties prevent us from straightforwardly extending our proofs to this wider class of functions. The supplement begins by analyzing the polynomial algorithm, and then describes the changes necessary to analyze the Hermite-based algorithms.


We will now give a motivating overview of the algorithms and proofs, which should serve as a guide to the rest of the supplement.

\subsection{Basic structure of cumulant propagation}
\label{sec:basicstructure}

The first intuition is simple: to estimate the mean output of an MLP, we will track estimates for the distribution (or statistics of the distribution) of (pre-)activations at each layer, starting at the input distribution (which we assume is Gaussian), and alternately passing through linear layers and non-linear activation layers, to produce an estimate for the distribution of the final layer. Define the following notation, which will serve to describe the MLP layer-by-layer:

\begin{itemize}
    \item For each $\ell \in [L+1] $, let $\boldsymbol\theta ^{(\ell)} = (\mathbf W^{(1)},...,\mathbf W^{(\ell)})$, so $\boldsymbol\theta = \boldsymbol\theta ^{(L+1)}$.
    \item Let $X = X^{(0)}\sim\mathcal{N}(0,\mathbf I_n)$ denote the input to the MLP. Then recursively define
    \begin{align*}
        \forall 1 \leq \ell \leq L+1, \qquad &Z^{(\ell)} := \mathbf W^{(\ell)}X^{(\ell-1)} \\
        \forall 1 \leq \ell \leq L, \qquad & X^{(\ell)} = \phi_{\ell}(Z^{(\ell)}).
    \end{align*}
    Then $Z^{(\layernumber)}$ and $X^{(\layernumber)}$ are the preactivation and activation at layer $\layernumber$, respectively.
\end{itemize}

While our estimation procedure applies to any specific set of weights, Theorem \ref{matchingsamplingtheorem} only requires the estimate to be good \textit{in expectation over random weights and in the limit} $n\to\infty$. Therefore, we can imagine $\mathbf W^{(\ell)}$ to be a large, unstructured matrix. By the central limit theorem, the relation $Z^{(\ell)} = \mathbf W^{(\ell)}X^{(\ell-1)}$ then suggests that regardless of the distribution of $X^{(\ell-1)}$, we expect the distribution of $Z^{(\ell)}$ to be close to Gaussian. Since distributions that are close to Gaussian are well-specified by their low-degree cumulants, we are led to the idea of tracking the cumulants through the layers.

The intuition that higher cumulants decay as $n\to\infty$ is proven in Corollary \ref{cor:elementgrowth}:
\begin{equation}
    \label{eq:growth}
    \mathbb{E} \left[ \left| \kappa_r[Z^{(\ell)}] _{i_1,...,i_r} \right|^2 \right] = 
    \begin{cases}
        O( n^{2-r} ) & \parbox[t]{6cm}{\centering if $r$ is even, and each value in $i_1,\ldots,i_r$\\
appears with even multiplicity, } \\
        O( n^{1-r} ) & \parbox[t]{6cm}{\centering otherwise.}
    \end{cases}
\end{equation}

This heuristically suggests that for integers $K\ge 1$, to get mean squared error (MSE) $O(n^{-K})$ we can do the following:
\begin{itemize}
    \item If $K$ is even, estimate all cumulants of order at most $K$;
    \item If $K$ is odd, estimate all cumulants of order at most $K$ \textit{and} the elements $\kappa_{K+1}[Z^{(\ell)}]_{i_1,i_1,...,i_{\frac{K+1}{2}},i_{\frac{K+1}{2}}}$. (As a more efficient bookkeeping device, we track these terms via the scalar harmonic component of $\kappa_{K+1}$ -- see Section \ref{harmonicprereq}).
\end{itemize}

This is indeed what we do in a class of algorithms we call \textit{cumulant propagation} (so $\maxcumulantorder$ is not literally the maximum cumulant degree propagated, but it serves the same moral purpose). These algorithms estimate the cumulants listed above at each preactivation and activation layer. 

In our main results, our algorithm is given only the tolerance parameter $\varepsilon$. In this case, for each $n$, we take $\maxcumulantorder$ to be the smallest integer such that $\frac 1{n^\maxcumulantorder}\leq\varepsilon^2$ (note that we could have chosen any multiple of $\frac 1{n^\maxcumulantorder}$ arbitrarily). This choice of $\maxcumulantorder$ ensures that the mean squared error is $O(\varepsilon^2)$. To be fast enough, we require the time taken to be $O(n^{K+1})$.\footnote{One might have thought that $O(n^{K+2})$ would be sufficient to achieve a time of $O(\frac{n^2}{\varepsilon^2})$, but this is not the case because $K$ is constrained to be an integer, while $\varepsilon^2$ could be $\Theta(n^{-k})$ for some non-integer $k$.} Therefore, to prove the validity of cumulant propagation, we wish to prove that for each integer $K \ge 1$,
\begin{enumerate}
    \item Tracking the cumulants listed above does indeed produce MSE $O(n^{-K})$. This is our main result, Theorem \ref{thm:main}.
    \item Tracking these cumulants takes time $O(n^{K+1})$. This is easily verified during the description of the algorithm in Sections \ref{sec:polyalgdesc} and \ref{kpropfull}.
\end{enumerate}

Note that for this analysis to hold, we require that $\maxcumulantorder$ is bounded above by a constant. This is exactly where the requirement that $\varepsilon\!\left(n\right)$ is noticeable comes from.

Let us now look at each stage of the algorithms in more detail.

\subsection{The linear step}

Tracking cumulants through the linear step is simple but expensive: cumulants are multi-linear, so the cumulants of degree $r$ transform as
\[
\kappa_r[Z^{(\ell)}]_{i_1,\ldots,i_r} = \sum_{j_1,...,j_r=1}^n \kappa_r[X^{(\ell-1)}]_{j_1,...,j_r}\prod_{s=1}^r \mathbf W_{i_s,j_s} ^{(\ell)}
\]
which takes $O(n^{r+1})$-time, and so can be done for cumulants up to degree $K$. We call this \textit{symmetric contraction} (see Definition \ref{def:symmcontr}). For harmonic components of cumulants of degree $>K$, we must figure out how to propagate the harmonic component itself rather than the whole tensor, but this is relatively routine.

\subsection{The non-linear step for the polynomial algorithm}

The more challenging part of the problem is how to propagate the cumulants through the (possibly non-linear) activation function, i.e. $X^{(\ell)} = \phi_{\ell}(Z^{(\ell)})$. The polynomial algorithm deals with this non-linearity in the following way: since $\phi_{\ell}$ is a polynomial, the cumulant $\kappa_r[X^{(\ell)}]_{i_1,...,i_r}$ can be expanded out by multi-linearity to a linear combination of cumulants of the form
\[
\kappa_r[(Z^{(\ell)}_{i_1})^{k_1},...,(Z^{(\ell)}_{i_r})^{k_r}]
\]
A simple combinatorial argument shows that a cumulant of powers can be written as a polynomial in the cumulants of the underlying random variables (see Lemma \ref{lem:simplepowers}), in particular
\[
    \kappa_r[(Z^{(\ell)}_{i_1})^{k_1},...,(Z^{(\ell)}_{i_r})^{k_r}]
 = \sum_{ \substack{ \pi \vdash [k_1] \sqcup ... \sqcup [k_r] \\ \text{connected}} }^{  } \prod_{ S \in \pi  }^{  } \kappa \left[ Z^{(\ell)} _{ i_v } : (v,l) \in S \right] 
\]
where $\pi$ is a partition (see Section \ref{sec:polypolyproof} for the precise definition of the notation). Combining this with the multi-linearity proves that, if $\phi(t) = \sum_{k=0}^{D} \frac{1}{k!}a_k t^k $ is a polynomial, 
\begin{equation}
    \label{eq:taylordiagrams}
    \kappa_r[\phi(Z)]_{i_1,...,i_r} = \sum_{\underline{k} \in \mathbb \{0,...,D\}^r} \frac{1}{\underline{k}!}\left(\prod_{v=1}^r a_{k_v} \right) \sum_{\substack{\pi\vdash [k_1]\sqcup ...\sqcup [k_r] \\ \text{connected}}} \prod_{S\in\pi} \kappa[Z_{i_v}: (v,l)\in S].
\end{equation}

This is the simplest form of what we call a \textit{diagram summation formula}, in particular this is the polynomial diagram summation formula (see Theorem \ref{polydiagsthm}). This diagram summation formula allows us to propagate our estimates for (the harmonic components of) the cumulants through the non-linearity.

\subsection{The non-linear step for Hermite-based algorithms}
\label{sec:polyadj}

It is clear that the algorithm just discussed has a number of flaws:
\begin{itemize}
    \item It can only be run for polynomial $\phi_{\ell}$; and,
    \item Even for polynomial $\phi_{\ell}$, the constants in the runtime depend on the degree of the polynomial. A more favorable algorithm might prioritize the larger terms in (\ref{eq:taylordiagrams}) so that the number of terms computed depends on $K$ but not $D$.
\end{itemize}
Let us now consider how to alter the algorithm to fix these two issues. A first attempt could go as follows: expand the activation function as a Taylor series around $x=0$, apply multilinearity and Lemma \ref{lem:simplepowers} to get the same diagram summation formula above (equation (\ref{eq:taylordiagrams})) as a formal expansion with infinite summation, and truncate this expression after some finite number of terms. 

Now recall equation (\ref{eq:growth}). This tells us that the means and variances of $Z$ are of constant size as $n\to\infty$ (technically we only give an upper bound, but it is, in fact, tight). Therefore, for any diagram (i.e. value of $\underline{k}$ and $\pi$) in (\ref{eq:taylordiagrams}), we can produce an infinite collection of diagrams with the same growth in $n$ by increasing $\underline{k}$ by $\underline{d}$ and partitioning $\underline{d}$ with means and variances. So, any finite truncation will surely fail. 

A partial solution to this is to arrange the sum on a ``base diagram" of size $\underline{k}$ which contains no part of size 1 or 2 in a single $[k_v]$ (we call such a diagram $[\leq2]$-mixed), and a diagram of size $\underline{d}$ which contains the remaining parts (and does not affect the connectedness of the complete diagram). We need an extra binomial factor to account for the choice of which shape-$\underline{k}$ subset of $\underline{k}+\underline{d}$ is the base diagram.
\begin{align}
    \kappa_r[\phi(Z)]_{i_1,...,i_r} = & \sum_{\underline{k}\in\mathbb{Z}^r_{\ge0}} \left(\sum_{\substack{\pi\vdash [k_1]\sqcup ...\sqcup [k_r] \\ \text{connected} \\ [\leq2]\text{-mixed}}} \prod_{S\in\pi} \kappa[Z_{i_v}: (v,l)\in S] \right) \nonumber \\
    &\hspace{2cm}\prod_{v=1}^r\left( \sum_{d=0}^{\infty} \binom{k_v+d}{k_v} \frac{a_{k_v+d}}{(k_v+d)!} \sum_{\substack{\tau \vdash[d] \\ \text{all parts size }\leq2}} \prod_{S\in\tau}\kappa_{|S|}[Z_{i_v}] \right)
\end{align}

We can now simplify the second term: by the defining property of cumulants, the sum over $\tau$ is equal to $\mathbb E[(Y_{i_v})^d]$, where $Y_{i_v}$ is the 1d Gaussian with the same mean and variance as $Z_{i_v}$. This means that the sum over $d$ evaluates to $\frac{1}{k_v!}\mathbb E \left[ \phi^{(k_v)}(Y_{i_v}) \right]$. This expectation is exactly the $k_v$-th Hermite coefficient of $\phi$ with respect to $Y_{i_v}$, which we notate $\widehat{\phi}_{k_v}^{Y_{i_v}}$ (for the definition, see Section \ref{hermitesubsection})! Therefore, we have shown
\begin{equation}
    \label{eq:basichermite}
    \kappa_r[\phi(Z)]_{i_1,...,i_r} = \sum_{\underline{k}\in\mathbb{Z}^r_{\ge0}} \frac{1}{\underline{k}!} \left( \prod_{v=1}^r \widehat{\phi}_{k_v}^{Y_{i_v}} \right)\left(\sum_{\substack{\pi\vdash [k_1]\sqcup ...\sqcup [k_r] \\ \text{connected} \\ [\leq2]\text{-mixed}}} \prod_{S\in\pi} \kappa[Z_{i_v}: (v,l)\in S] \right)
\end{equation}

This is the Hermite-based diagram summation formula, Theorem \ref{diagramstheorem}. This is significant progress, since the Hermite coefficients may often be computed exactly (or at least approximated efficiently). In fact, if $i_1,...,i_r$ are distinct, we have collected together all of the $\Theta(1)$-growth terms, and this is the expansion that we truncate in our Hermite-based algorithms. 

However, there is still a further adjustment to make: notice that when $i_1=i_2$ in Equation (\ref{eq:basichermite}), the ``base diagram'' part may still contain $\mathrm{Var}[Z_{i_1}]$. For example, the equation simplifies in the case $r=2, i_1=i_2=i$, and $Z=\mathcal{N}(0,I_n)$ to
\begin{align*}
    \kappa_2[\phi(Z)]_{i,i}=\sum_{k\geq 1} \dfrac{1}{k!} \left(\widehat{\phi}_k^Z\right)^2 \Var[Z_i]^k.
\end{align*}
So once again, when some indices are equal, any truncation of (\ref{eq:basichermite}) will fail. The fix here can be thought of as one of the following two equivalent observations: 
\begin{enumerate}
    \item We can explicitly sum this infinite collection of $\Theta(1)$ terms. By the orthogonality property for Hermite polynomials, 
    \[
    \widehat{\phi^2}_0^Z=\mathbb{E}[\phi^2] = \mathbb{E} \left[ \sum_{k_1,k_2=0}^\infty \frac{\sigma^{k_1+k_2}}{k_1!k_2!}\widehat{\phi}_{k_1}^Z\widehat{\phi}_{k_2}^Z \mathrm{He}_{k_1}\left(\sigma^{-1}Z\right)\mathrm{He}_{k_2}\left(\sigma^{-1}Z\right)\right] = \sum_{k=0}^{\infty} \dfrac{1}{k!} \left(\widehat{\phi}_k^Z\right)^2 \sigma^{2k}
    \]
    and so 
    \[
    \kappa_2[\phi(Z)]_{i,i} = \widehat{\phi^2}_0^Z-(\widehat{\phi}_0^Z)^2.
    \]
    We can always do this summation to remove the $\Theta(1)$ terms in (\ref{eq:basichermite}) when indices are repeated.\footnote{For a more involved example, consider $\kappa_3[\phi(Z)]_{i,i,j}$ with $i\neq j$ when $Z$ is a mean zero multivariate Gaussian. The Hermite-based diagram summation formula can be factorized as
    \[
    \kappa_3[\phi(Z)]_{i,i,j} = \sum_{k \ge 1} \frac{1}{k!} (\widehat{\phi}_k^Z)_j \mathrm{Cov}[Z]^k_{i,j} \left( \sum_{a\ge 0} \frac{1}{a!} \mathrm{Var}[Z]_i^a \sum_{ \substack{0\leq b\leq k\\ a+b(k-b) \neq 0}} \binom{k}{b} (\widehat{\phi}_{a+b}^Z)_i(\widehat{\phi}_{k+a-b}^Z)_i \right)
    \]
    Another application of the orthogonality of Hermite coefficients, and a formula for the Hermite expansion of a product of Hermite polynomials shows that the inner parentheses evaluate to $\left( \widehat{\phi^2}_k^Z\right)_i-2\left(\widehat{\phi}_0^Z\right)_i\left(\widehat{\phi}_k^Z\right)_i$
    and so
    \begin{align*}
    \kappa_3[\phi(Z)]_{i,i,j} &= \sum_{k\ge 1}\frac{1}{k!}  \left( \widehat{\phi^2}_k^Z\right)_i(\widehat{\phi}_k^Z)_j \mathrm{Cov}[Z]_{i,j}^k -2\left(\widehat{\phi}_0^Z\right)_i \sum_{k\ge1}\left(\widehat{\phi}_k^Z\right)_i (\widehat{\phi}_k^Z)_j\mathrm{Cov}[Z]_{i,j}^k \\
    &= \kappa_2[\phi(Z)^2,\phi(Z)]_{i,j} - 2 \kappa_1[\phi(Z)] \kappa_2[\phi(Z)]_{i,j}
    \end{align*}
    matching our second interpretation of power cumulants.
    }
    \item Alternatively, by writing the cumulant in terms of expectations and collecting equal indices, we can write any cumulant as a polynomial in cumulants of powers with pairwise distinct indices. For example,
    \begin{align*}
        \kappa_2[\phi(Z)]_{i,i} &= \kappa_1 [\phi^2(Z)]_i - \kappa_1[\phi(Z)]_i^2 \\
        \kappa_3[\phi(Z)]_{i,i,j} &= \kappa_2 [\phi^2(Z),\phi(Z)]_{i,j} - 2 \kappa_2[\phi(Z)]_{i,j}\kappa_1[\phi(Z)]_i.
    \end{align*}
    We call this construction \textit{power cumulants} (see Section \ref{powercumulantsprereq}). 

    These power cumulants can now be expanded by the same expansion as (\ref{eq:basichermite}) with the Hermite coefficients of powers of $\phi$. All $\Theta(1)$-size cumulants in this expression have now been absorbed into the Hermite coefficients. 
\end{enumerate}
These two bookkeeping devices (Hermite coefficients and power cumulants) turn out to be sufficient to deal with both flaws of the polynomial algorithm mentioned at the start of this section. So, to summarize, the outline of the non-linear step for the Hermite-based algorithms is as follows:
\begin{itemize}
    \item First use the combinatorics of power cumulants to rewrite $\kappa_r[\phi_{\ell}(Z^{(\ell)})]_{i_1,...,i_r}$ using cumulants of powers of $\phi_{\ell}(Z^{(\ell)})$ where all indices are distinct.
    \item For a power cumulant with distinct indices, use the corresponding expansion similar to equation (\ref{eq:basichermite}).
    \item Depending on the specific algorithm, truncate this expansion for power cumulants after some finite number of terms so that the tail is sufficiently small in the limit as $n\to\infty$.
\end{itemize}

Since we can compute Hermite coefficients of powers of $\phi$ in many cases beyond polynomial $\phi$, this algorithm can be applied quite broadly. For example, see Section \ref{sec:hermite-relu} for the Hermite coefficients of powers of $\mathrm{ReLU}$. 

\subsection{Proof outline}

We now give an overview of the proof itself. 

As is clear in Section \ref{sec:basicstructure} (see equation (\ref{eq:growth})), it is crucial to understand the asymptotic growth of the true cumulants as $n\to\infty$. While Corollary \ref{cor:elementgrowth} gives the moral result, the real work is done by Proposition \ref{prop:cumulantgrowth}, which states that any connected scalar contraction of cumulants grows linearly in $n$ as $n\to\infty$. This proceeds via layer-by-layer induction, and is independent of the algorithm, simply depending on the choice of activation function. See Section \ref{sec:nonpoly-activation} for a discussion in the non-polynomial setting.

Next, we compute bounds for the expected sizes of the error cumulants (the difference between our estimated cumulants and the true cumulants). For these bounds, we simply bound the square value of any entry of an error cumulant by $n^{-K}$. This proof of this result is essentially a simple induction:
\begin{itemize}
\item The error tensors of $Z^{(\ell+1)}$ are easily computed in terms of contractions of the error tensors of $X^{(\ell)}$.
\item To compute the sizes of these contractions, we use the relevant diagram summation formula (Theorem \ref{polydiagsthm} or \ref{diagramstheorem}) and prove that the terms we drop in the truncation are all sufficiently small (we only drop terms when using the Hermite-based diagram summation formula). These terms include both the error cumulants and true cumulants of $Z^{(\ell)}$, and we use Cauchy-Schwartz to separate the two, bounding the two resulting diagrams individually.
\end{itemize}

Applying the error bound to the final first-order cumulant gives the required result. See Theorem \ref{thm:main} for the proof for the polynomial algorithm, Section \ref{smainproof} for the adjustments required for the Hermite-based algorithms, and Section \ref{sec:nonpoly-activation} for a discussion of the non-polynomial setting. 

\subsection{Supplement structure}

The supplement is structured as follows.

\begin{itemize}
    \item In Section \ref{sec:polypolyproof}, we prove Theorem \ref{matchingsamplingtheorem} using the polynomial algorithm. We also cover the prerequisites for the proof, including the polynomial diagram summation formula (Theorem \ref{polydiagsthm}), harmonic decompositions of symmetric tensors (Theorem \ref{thm:harmonicdecomposition}), and asymptotic growth of tensors and tensor diagrams (Sections \ref{tensorgrowth} and \ref{sec:tensordiags}).
    \item In Section \ref{smathprelims}, we state some preliminary results required for the proofs of the Hermite-based algorithms. These cover some combinatorial results required for working with cumulants and diagrams, as well as explicit computation of Hermite coefficients.
    \item In Section \ref{kpropfull}, we give a full mathematical description of our Hermite-based algorithms.
    \item In Section \ref{smainproof}, we give a detailed sketch of the alterations to the proof in Section \ref{sec:polypolyproof} in order to prove Theorem \ref{matchingsamplingtheorem} using the Hermite-based algorithms, still in the case of polynomial activation functions.
    \item In Section \ref{sec:nonpoly-activation}, we discuss the remaining analytic results that we believe would be sufficient to adapt the proof in Section \ref{smainproof} to apply to non-polynomial activation functions.
    \item In Section \ref{diagramsproof}, we prove both the polynomial and Hermite-based versions of the diagram summation formula for cumulants.
    \item In Section \ref{sec:prelimproofs}, we prove all of remaining the preliminary results of Sections \ref{sec:polypolyproof} and \ref{smathprelims}.
\end{itemize} 

\section{Polynomial algorithm}\label{sec:polypolyproof}
In this Section we will prove Theorem \ref{matchingsamplingtheorem}, the main theoretical result of this paper, in the case that the algorithm used is the polynomial algorithm (note that as per Section \ref{kpropfull}, this algorithm does not agree with the basic algorithm, even in the case that the activation functions are polynomials).

Write $\tens{r}{n}$ for the space of all \textit{$r$-tensors} on $n$ real variables (which is an $\mathbb{R}$-inner product space with the Frobenius inner product). Write $\syms{r}{n}$ for the space of all \textit{symmetric} $r$-tensors on $n$ real variables.

A \textit{partition} $\pi$ of $[r]$, denoted $\pi\vdash[r]$, is a set of disjoint subsets (\textit{blocks}) of $[r]$ with union equal to $[r]$. The set of partitions $\{ \pi : \pi \vdash [r] \} $ forms a lattice with the following notation:
\begin{itemize}
    \item Write $\pi \leq \tau$ if $\forall S\in \pi , \exists T\in \tau$ such that $S\subseteq T$. 
    \item For $\pi,\tau\vdash [r]$, write $\pi \wedge \tau$ for the greatest lower bound (called the \textit{meet}) of $\pi$ and $\tau$, i.e.
    \[
    \pi\wedge\tau = \{ S\cap T: S\in \pi, T\in \tau \} 
    \]
    \item For $\pi,\tau \vdash [r]$, write $\pi\vee \tau$ for the least upper bound (called the \textit{join}) of $\pi$ and $\tau$.
\end{itemize}

\subsection{Polynomial diagram summation formula}

The polynomial algorithm uses a significantly simpler version of Theorem \ref{diagramstheorem} that avoids the need to Hermite coefficients, at the expense of us not being able to truncate the formula (therefore introducing a dependence of the algorithm on the degree of the polynomial). For polynomials, this is not an issue for the theoretical result since both diagram summation formulas have only a finite number of terms, however when we discuss non-polynomial activations we will see that truncating the Hermite-based diagram summation formula is a crucial step.

\begin{restatable}[Polynomial diagram summation formula for cumulants]{theorem}{polydiagramsummationformula}\label{polydiagsthm}
    Let $Z\in \mathbb{R}^n$ be a random variable with finite moments and let $\phi: \mathbb{R}\to\mathbb{R}$ be a polynomial. Then
    \[
\begin{aligned}
\kappa\left[\phi\left(Z_{j_1}\right),\dots,\phi\left(Z_{j_r}\right)\right]=\sum_{\underline{k}\in \mathbb{Z}_{\ge0}^r}\left(\prod_{v=1}^r\frac 1{k_v!} \phi^{(k_v)}(0)\right)\sum_{\pi \in \mathrm{cDia}(\underline{k})}\prod_{S\in\pi}\kappa\left[Z_{j_v}:\left(v,w\right)\in S\right]
\end{aligned}.
\]
\end{restatable}

This is proven in Section \ref{diagramsproof}. See Definition \ref{def:diagram} for the definition of $\mathrm{cDia}(\underline{k})$.
\subsection{Harmonic decomposition of symmetric tensors}\label{harmonicprereq}
\begin{definition}
    The \textit{trace} $\tr\colon \syms{r}{n}\to \syms{r-2}{n}$ is defined by
    \begin{equation*}
    (\tr \bfT )_{i_1,...,i_{r-2}} = \sum_{i=1}^n \bfT _{i,i,i_1,...,i_{r-2}}.
    \end{equation*}
    for $\bfT \in\syms{r}{n}$. Also, for $r\in\{0,1\}$, define $\tr : \syms{r}{n} \to \{0\}$ to be the zero map.
    The \textit{cup} $g\colon\syms{r}{n}\to\syms{r+2}{n}$ is defined by
    \begin{equation*}
    g(\bfT )_{i_1,...,i_{r+2}} = \sum _{ \substack{u,v\in[r+2] \\ u< v}} \delta_{i_u,i_v} \bfT _{ \{i_s:s\not\in \{u,v\}\}},
    \end{equation*}
    where $\delta$ is the Kronecker delta.
    More generally, for $\mathbf M\in\syms{2}{n}$, the \textit{$\mathbf{M}$-cup} $g_{\mathbf M}\colon\syms{r}{n}\to\syms{r+2}{n}$ is defined by
    \begin{equation*}
        g_{\mathbf M}(\bfT )_{i_1,...,i_{r+2}} = \sum _{ \substack{u,v\in[r+2] \\ u< v}} \mathbf M_{i_u,i_v} \bfT _{ \{i_s:s\not\in \{u,v\}\}}.
    \end{equation*}
\end{definition}
\begin{definition}
    \label{def:symmcontr}
    For a rank-$r$ tensor $\bfT$ and a matrix $\bfW$ define the \textit{symmetric contraction of $\bfT$ with $\bfW$} to be the rank-$r$ tensor
    \begin{equation*}
    (\bfT\odot \bfW)_{i_1,...,i_r} := \sum_{j_1,...,j_r=1}^n \bfT_{j_1,...,j_r}\prod_{s=1}^r \bfW_{i_s,j_s}
    \end{equation*}
    This operation can be computed in $O(n^{r+1})$ operations by multiplying in the matrix one index at a time.
\end{definition}

Given a symmetric rank-$r$ tensor $\bfT $, the harmonic decomposition expresses it as a sum of a ``generic'' part and a lower rank object (which then recursively breaks down into still smaller objects). 

\begin{restatable}[Harmonic decomposition of symmetric tensors]{theorem}{harmonicdecompositionformula} \label{thm:harmonicdecomposition}
    Let $\bfT \in\syms{r}{n}$.
    Then there exist unique $h_{r-2s}(\bfT )\in\syms{r-2s}{n}$ for $0\leq s \leq \left\lfloor \frac{r}{2} \right\rfloor $ satisfying
    \begin{align*}
        {\tr}\left( h_{r-2s}(\bfT ) \right) &=0,\quad\quad \forall 0 \leq s \leq \left\lfloor\frac{r}{2}\right\rfloor -1\\
        \bfT  &=  \sum _{0\leq s \leq \left\lfloor \frac{r}{2} \right\rfloor } g^{s}(h_{r-2s}(\bfT ))
    \end{align*}
    In fact,
    \begin{equation}
        \label{harmonicparteq}
    h_{r-2s}(\bfT ) = \sum_{0\leq t \leq \lfloor \frac{r-2s}{2} \rfloor} c^{\mathrm h}_{s,t} g^{t}( {\tr}^{s+t}(\bfT ))
    \end{equation}
    where
    \begin{equation*}
    c^{\mathrm h}_{s,t} = (-1)^t \frac{r+\frac{n}{2}-2s-1}{2^{s+t}s!t!(r+\frac{n}{2}-2s-t-1)_{s+t+1}}
    \end{equation*}
    with $(a)_t := a(a+1)...(a+t-1)$ (the Pochhammer symbol).
\end{restatable}
This is proven in Section \ref{harmonicproof}. For convenience, we notate
\begin{align*}
    h_{\leq r-2s}(\bfT ) &:= \sum_{s'=s}^{\lfloor r/2\rfloor}g^{s'-s}(h_{r-2s'}(\bfT )) \in \mathcal{S}^{r-2s}(\mathbb{R}^n) \\
    h_{\geq r-2s} (\bfT ) &:= \sum_{s'=0}^{s} g^{s'}(h_{r-2s'}(\bfT)) \in \mathcal{S} ^{r}(\mathbb{R}^n)
\end{align*}

To see why this is referred to as the ``harmonic'' decomposition, treat $\bfT\in\syms{r}{n}$
as the tensor of coefficients of a degree-$r$ homogeneous polynomial on $\R^n$.
(I.e., identify $\syms{r}{n}$ with $\mathcal{S}^r(\R^{n*})$ via the standard inner product on $\R^n$.)
Then the Laplacian operator is $r(r-1)\tr$, multiplication by squared radius is $\dfrac{2}{r(r-1)}g\colon \syms{r-2}{n}\to\syms{r}{n}$,
and the harmonic decomposition of symmetric tensors coincides with the usual decomposition of polynomials into spherical harmonics
(cf.~\citet[Theorem 1.1.3, Lemma 1.2.1]{DaiXu2013}).

Another characterization of the harmonic decomposition is as the decomposition of $\syms{r}{n}$
into irreducible representations of the action of $\operatorname{O}(n)$ via the symmetric contraction $\odot$~\citep[Theorem 1.7.2]{DaiXu2013}.
Because the distribution of Gaussian-initialized $\bfW$ is $\operatorname{O}(n)$-symmetric, 
this is intuitively the right level of granularity to track cumulant tensors of activations,
which directly precede the application of $-\odot\bfW$.

\subsection{Tensor growth} 
\label{tensorgrowth}

Although our algorithm applies to a particular instance of an MLP, Theorem \ref{matchingsamplingtheorem} refers to the error of these algorithms in expectation over the random weights and in the limit of large width. To prove these results rigorously, we must first define an appropriate notion of the growth of tensors in the limit of width $n \rightarrow \infty$. 

Let $\Theta^{(\ell)}_n$ be the random variable valued in $\mathbb{R}^{\ell \times n \times n}$ given by the values of $\boldsymbol\theta ^{(\ell)} = (\mathbf{W}^{(1)},...,\mathbf{W}^{(\ell)})$ when these weights are drawn independently from $\mathcal{N}(0,\frac{2}{n})$.

\begin{definition}
    Fix $r \ge 0$. We define a rank-$r$ \textit{tensor sequence} to be a sequence of random variables $T_n $ taking values in $ \tens{r}{n}$. A tensor sequence is said to be \textit{symmetric} if, for every $n \ge 1$, the support of $T_n$ is contained in $\syms{r}{n}$.

    We say that a tensor sequence $(T_n)_n$ is \textit{at layer $\ell$} for some $\ell \in [L+1]$ if the random variables $T_n$ are all measurable functions of $\Theta^{(\ell)}_n$.
\end{definition}

For example, for each $r \ge1$ and $\ell \in [L]$, the (pre-)activation cumulants ($\kappa_r[Z^{(\ell)}]$ and $\kappa_r[X^{(\ell)}]$ respectively) form symmetric tensor sequences at layer $\ell$.

\begin{definition}
    We say the tensor sequence $T=(T_n)_{n=1}^{\infty}$ has $n^{\rho}$\textit{-growth} (written suggestively as $T=O(n^{\rho})$) if, for each $1\leq p < \infty$, 
    $$
        \limsup_{n\geq 1}\max_{i_1,\ldots, i_r\in[n]}  \|(T_n)_{i_1,...,i_r}\|_p n^{-\rho}  < \infty
    $$
    where $\| \cdot\|_p$ denotes the $L_p$-norm. Since the $L_p$-norms increase as $p$ increases, it is sufficient to check this for even integers $p \in 2\mathbb{Z}_{\ge1}$.
\end{definition}

If $r=0$, we simply have a sequence of 1d random variables $X_n$, and this has $n^{\rho}$-growth if for every $1\leq p <\infty $, $(L_p(T_n))_{n\ge1} = O(n^{\rho})$ in the traditional sense of big-$O$ notation.

\subsection{Tensor diagrams}
\label{sec:tensordiags}
When computing the cumulants of $\phi(Z)$ in terms of the cumulants of $Z$, we will see tensor networks of cumulants. We formalize these using tensor sequences in the following way:

\begin{definition}
    We define a \textit{tensor diagram} as a tuple $D=(V_1,V_2,E,I, \{T_v\}_{v\in V_1})$ where
    \begin{enumerate}
        \item $B_D:=(V_1,V_2,E)$ is a finite bipartite (multi-)graph (i.e. we allow multiple edges),
        \item $I\subset V_2$
        \item For each $v\in V_1$, $T_v$ is a symmetric tensor sequence with $\deg(v) = \mathrm{rank}(T_v)$.
    \end{enumerate}
    Let $\mathrm{rank}(D):=|V_2 \setminus I|$. Call $D$ \textit{connected} if $B_D$ is connected; call it \textit{fully-contracted} if $\mathrm{rank}(D)=0$. We say the diagram $D$ is \textit{at layer} $\ell$ if every tensor sequence is at layer $\ell$.
\end{definition}

This should be interpreted as an einsum in the following way: the vertices $v\in V_2$ correspond to indices, those in $I$ are contracted indices, those not in $I$ are ``hanging" indices, and the tensor diagram corresponds to a tensor sequence where the contracted indices have been summed over, giving
\begin{equation}
\label{tensordiagramdef}
T_{D,n} = \left( \sum_{\forall v\in I, i_v=1}^n \prod_{u\in V_1} ((T_u)_{n})_{i_v : (u,v) \in E})\right) _{i_v:v\not\in I}
\end{equation}
Here we gloss over two details: first, note that in order to produce a truly well-defined tensor sequence, we need to pick a bijection between $V_2\setminus I$ and $[\deg(D)]$; and second, for this to be a well-defined tensor sequence, we must correctly reconcile the underlying sources of randomness for the different tensor sequences. In all cases we consider, the random variables will have shared randomness. 



Sometimes, we will refer to a tensor being written as a sum of diagrams, and we should emphasize here that this means an \textit{element-wise sum} of tensors with equal rank, \textit{not a disjoint sum of diagrams}. In particular, a disjoint sum of fully-contracted diagrams results in multiplication of the corresponding scalars, whereas a sum of fully-contracted diagrams results in a sum of the corresponding scalars.

We now look at the effect of moving from a tensor diagram at layer $\ell$ to one at layer $\ell+1$ by contracting each tensor sequence by the transition matrix $\mathbf W^{(\ell+1)}$. For this purpose, note that we can extend the notation $-\odot \mathbf{W}^{(\ell+1)} $ of Definition \ref{def:symmcontr} from tensors to tensor sequences at layer $\ell$ in the obvious way, producing a tensor sequence at layer $\ell+1$. For a tensor diagram $D$ at layer $\ell$, let $D \odot \mathbf W^{(\ell+1)}$ be the tensor diagram at layer $\ell+1$ resulting from replacing each $T_v$ with $T_n \odot \mathbf W^{(\ell+1)}$ for all $v\in V_1$ (note that this differs from computing $D$ via an einsum as a rank-$|V_2\setminus I|$ tensor sequence and applying $-\odot \mathbf{W}^{(\ell+1)}$).

\begin{example}
    The post-activation cumulants at layer $\ell$ and the pre-activation cumulants at layer $\ell+1$ are related simply via
    $$
        \kappa_d[Z^{(\ell+1)}] = \kappa_d [ X^{(\ell)} ] \odot \mathbf W^{(\ell+1)}.
    $$
\end{example}

\begin{restatable}[Diagram expectations]{proposition}{wdiags}\label{wdiagsprop}
    Let $D=(V_1,V_2,E,I, \{T_v \} _{v\in V_1})$ be a rank $r$ tensor diagram at layer $\ell$. If $|E|$ is odd, then the rank $r$ tensor sequence $\mathbb{E}_{\mathbf W^{(\ell+1)}}[D \odot \mathbf W^{(\ell+1)}]$ is identically zero. If $|E|$ is even, 
    $$
        \mathbb{E}_{\mathbf W^{(\ell+1)}} [D \odot \mathbf W^{(\ell+1)}] = \left( \frac{2}{n} \right)^{|E|/2} \sum _{\pi \in \mathcal{P}_2 (E)} D_{\pi,1}D_{\pi,2},
    $$
    where $D_{\pi,1}$ is a rank 0 (scalar) tensor diagram at layer $\ell$ constructed as $(V_1,\pi,\pi, \{(v,S):v\in e\in S \in \pi\},\{T_v\}_{v\in V_1})$, and $D_{\pi,2}$ is a constant rank $r$ tensor diagram (constant meaning that the random variables are all constant) constructed as $(\pi,V_2,I,\{(S,w):w\in e\in S\in \pi \} , (I_2)_{S\in \pi})$. Here, $\mathcal{P}_2(E)$ denotes the set of perfect pairings of $E$. Furthermore, the number of connected components of $D_{\pi,1}D_{\pi,2}$ is at most $|E|/2$ more than the number of connected components of $D$.
\end{restatable}

\begin{example}
    Consider a symmetric $n\times n$-matrix $M$. Then $M \odot \mathbf{W} = \mathbf{W}M\mathbf{W}^{\top}$. Therefore $\|\mathbf{W}M\mathbf{W}^{\top}\|_F^2$ is a diagram of the form $D\odot \mathbf{W}$ for some fully contracted diagram $D$. There are $4$ edges in $D$, so the right hand side of Proposition \ref{wdiagsprop} has $3$ terms. These are most easily understood in diagrammatic form, as in Figure \ref{fig:wdiags}. In fact, we see that
    \[
        \mathbb{E}_{\mathbf{W}}\left[ \| \mathbf{W}M\mathbf{W}^{\top}\|_F^2 \right] = \left(\frac{2}{n}\right)^2\left(n\mathrm{Tr}(M)^2+n\|M\|_F^2 + n^2 \|M\|_F^2\right) 
    \]
\end{example}

\begin{figure}[ht]
\centering
\begin{tikzpicture}

\node[draw] (Ma1) at (-0.12, 0.5) {$M$};
\node[draw] (Ma2) at (-0.12,-0.5) {$M$};

\draw (0.18, 0.7) -- (1, 0.7) coordinate (ta1);
\draw (0.18, 0.3) -- (1, 0.3) coordinate (ta2);
\node[fill=white, inner sep=1.2pt] at (0.7, 0.7) {$\mathbf{W}$};
\node[fill=white, inner sep=1.2pt] at (0.7, 0.3) {$\mathbf{W}$};

\draw (0.18,-0.3) -- (1,-0.3) coordinate (ba2);
\draw (0.18,-0.7) -- (1,-0.7) coordinate (ba1);
\node[fill=white, inner sep=1.2pt] at (0.7,-0.3) {$\mathbf{W}$};
\node[fill=white, inner sep=1.2pt] at (0.7,-0.7) {$\mathbf{W}$};

\draw (ta1) .. controls +(0.9, 0) and +(0.9, 0) .. (ba1);
\draw (ta2) .. controls +(0.45,0) and +(0.45,0) .. (ba2);

\node at (1, -1.3) {(a) $\|\mathbf{W}M\mathbf{W}^{\!\top}\|_F^2$};

\begin{scope}[xshift=4.6cm]
  \node[draw] (Ma1) at (-0.11, 0.5) {$M$};
  \node[draw] (Ma2) at (-0.11,-0.5) {$M$};
  \draw (0.2, 0.7) arc (90:-90:0.2);
  \draw (0.2,-0.3) arc (90:-90:0.2);
  \draw (0.8, 0.7) arc (90:-90:0.4 and 0.7) arc (90:-90:-0.2) arc (90:-90:0.2 and -0.3) arc (90:-90:-0.2); 
\end{scope}

\node at (6.2, 0) {$+$};

\begin{scope}[xshift=7cm]
  \node[draw] (Ma1) at (-0.11, 0.5) {$M$};
  \node[draw] (Ma2) at (-0.11,-0.5) {$M$};
  \draw (0.2, 0.7) arc (90:-90:0.2 and 0.5);
  \draw (0.2,0.3) arc (90:-90:0.2 and 0.5);
  \draw (0.8,0.7) arc (90:-90:0.4 and 0.7) arc (90:-90:-0.2 and -0.5) arc (90:-90:0.2 and 0.3) arc (90:-90:-0.2 and -0.5) ;
\end{scope}

\node at (8.5, 0) {$+$};

\begin{scope}[xshift=9.4cm]
  \node[draw] (Ma1) at (-0.11, 0.5) {$M$};
  \node[draw] (Ma2) at (-0.11,-0.5) {$M$};
  \draw (0.2, 0.7) arc (90:-90:0.4 and 0.7);
  \draw (0.2,0.3) arc (90:-90:0.2 and 0.3);
  \draw (1.2,0.7) arc (90:-270:0.4 and 0.7);
  \draw (1.2,0.3) arc (90:-270:0.2 and 0.3);
\end{scope}

\node at (7.6, -1.3) {(b) $n(\operatorname{Tr}(M))^2 + n\|M\|_F^2 + n^2\|M\|_F^2$};

\end{tikzpicture}
\caption{A diagrammatic proof of the identity 
$
\mathbb{E}_{\mathbf{W}}\left[ \| \mathbf{W}M\mathbf{W}^{\top}\|_F^2 \right] = \left(\frac{2}{n}\right)^2\left(n\mathrm{Tr}(M)^2+n\|M\|_F^2 + n^2 \|M\|_F^2\right) 
$ for a symmetric $n\times n$ matrix $M$. Each loop in the right hand side diagrams contributes a factor of $n$. The diagrams on the right are obtained by replacing the four copies of $\mathbf{W}$ with each of the three possible pairings. Explicitly, we replace $\prod_{r=1}^4 \mathbf{W}_{i_r,j_r}$ with $\delta_{i_1,i_2}\delta_{i_3,i_4}\delta_{j_1,j_2}\delta_{j_3,j_4}+\delta_{i_1,i_3}\delta_{i_2,i_4}\delta_{j_1,j_3}\delta_{j_2,j_4}+\delta_{i_1,i_4}\delta_{i_2,i_3}\delta_{j_1,j_4}\delta_{j_2,j_3}$.}
\label{fig:wdiags}
\end{figure}
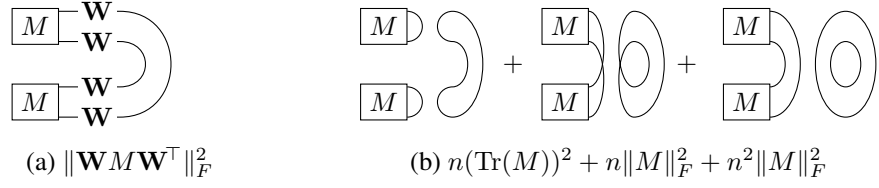

We also have a form of Cauchy-Schwartz for tensor diagrams given below. Both results are proven in Section \ref{wdiagsproof}. 

\begin{restatable}[Cauchy-Schwartz for tensor diagrams]{lemma}{cs}\label{lem:cs} 
Let $ D = (V_1,v_2,E,I, \left\{ T_v \right\} ) $ be a tensor diagram, and let $ V_1 = U_1 \sqcup U_2 $ be a partition of the left vertices of $ D $ into two parts. Consider the two tensor diagrams formed as follows: for each $ i \in \left\{ 1,2 \right\}  $ construct $ D_i $ by doubling the set $ U_i $ along with all its edges and tensors, formally
\begin{itemize}
  \item $ V_1(D_i) = [2] \times U_i $
  \item $ V_2(D_i) = V_2(D) $
  \item $ E(D_i) = \left\{ ((j,v),w) : (v,w) \in E(D), v \in U_i \right\}  $
  \item $ I(D_i) = I(D) $
  \item $ T_{ (j,v) } = T_v $ for all $ j = 1,2 , v \in U_i $.
\end{itemize}
Then all entries of $ D_1,D_2 $ are non-negative, and
\[
  D _{ i_1,...,i_r }^2 \leq (D_1) _{ i_1,...,i_r } (D _{ 2 } ) _{ i_1,...,i_r } .
\] 
In addition, if $ D_i $ have $ n ^{ \rho _{ i }  }  $-growth, then $ D $ has $ n ^{ \frac{1}{2} (\rho _1 + \rho _2) }  $-growth.
\end{restatable}


\subsection{Cumulant analysis for polynomial activation functions}
\label{sec:cumulantsizes}
In this section, we will discuss the cumulant tensors of the layers (both pre- and post-activation). As previously discussed, we are considering the cumulants for fixed weights -- the distribution comes from the input $ X ^{ (0) } \sim \mathcal{ N } (0,I_n) $. The main result is the following Proposition. The discovery and proof of this Proposition was pivotal to our understanding of the proposed algorithms (both in the polynomial and Hermite-based settings), and so we provide the proof here rather than delaying it to Section \ref{sec:prelimproofs}. 

\begin{restatable}[Cumulants growth bounds]{proposition}{cumulantgrowth}\label{prop:cumulantgrowth} 
  Suppose that $ D $ is a fully-contracted connected tensor diagram such that for some $ \ell \in [L+1] $, either
\begin{itemize}
\item every tensor sequence for $ v \in V_1 $ is $ \kappa _{ \mathrm{ deg } (v) } [ X ^{ (\ell-1) } ] $; or
\item every tensor sequence for $ v \in V_1 $ is $ \kappa _{ \mathrm{ deg } (v) } [ Z ^{ (\ell) } ] $.
\end{itemize}
Then, $ D $ has $ n $-growth.
\end{restatable}

Before proving this result, let us first recap some simple facts about the cumulants. For the $ 0 $-th layer, the cumulants are trivial:
\[
  \kappa _{ r } [X ^{ (0) } ] = \begin{cases}
  I_n, &\text{ if } r=2 \\
  0, & \text{ o/w}
  \end{cases}
\] 
To compute the cumulants after applying the linear layer, we apply symmetric contraction of the weights matrix $\mathbf W^{(\ell)}$,
\[
\kappa _r \left[Z^{(\ell)}\right] = \kappa_r \left[ X^{(\ell-1)} \right] \odot \mathbf{W}^{(\ell)}
\]
for $\ell \in [L+1]$. For the activation step, we have the following simple result.

\begin{lemma}
  \label{lem:nonlinconnected} 
For any rank $ r $, the tensor sequence
\[
  \kappa _{ r } \left[ \phi^{(\ell)} (Z^{(\ell)}) \right] 
\] 
can be written as a finite sum of connected rank-$r$ tensor diagrams involving only tensor sequences from the set $\left\{ \kappa _{ r' } \left[ Z^{(\ell)} \right] : r' \ge 1 \right\} $.
\end{lemma}
\begin{proof}
  By multi-linearity of $ \kappa _{ r }  $ we get a finite sum of cumulants of monomials. Lemma \ref{lem:simplepowers} then applies producing a sum of tensor diagrams of $\kappa_{r'}[Z^{(\ell)}]$'s. The identification of indices corresponds to the partition $ [r_1]\cup ... \cup [r_d] $, in particular there are $ d $ indices, and $ I= \phi  $, so no indices have been contracted. The cumulant tensors correspond to the partition $ \pi  $. By assumption, these two partitions satisfy that their join is the trivial partition and therefore this tensor diagram is connected.
\end{proof}

We now prove Proposition \ref{prop:cumulantgrowth}, by induction.

\begin{proof}
We prove this by induction, starting with the case that every tensor sequence in $ D $ is $ \kappa _{ \mathrm{ deg } (v) } \left[ X ^{ (0) }  \right] $. Since only the second cumulant of $ X ^{ (0) }  $ is non-zero, and it is equal to $ I_2 $, we see that any two indices connected to the same cumulant must be equal, and therefore connectedness implies all indices must be equal, therefore
\[
  D _{ n } = n.
\] 
We now prove that if the result holds for $ Z ^{ (\ell) }  $-cumulants, then it holds for $ X ^{ (\ell) }  $-cumulants. This follows from Lemma \ref{lem:nonlinconnected}, since if we expand the $ X ^{ (\ell) }  $-cumulants in terms of $ Z ^{ (\ell) }  $-cumulants, each is a sum (not disjoint union!) of connected tensor diagrams of $Z^{(\ell)}$-cumulants. Placing these on each node $ v \in V_1(D) $ and then expanding the sum, we get a sum of fully-contracted connected tensor diagrams of the $ Z ^{ (l) }  $-cumulants. Consequently this has $ n $-growth.

Finally, we prove that if the result holds for $ X ^{ (\ell-1) }  $-cumulants, then it holds for $ Z ^{ (\ell) }  $-cumulants. Consider the diagram $ D $: taking the $p$-th power of this quantity corresponds to taking the disjoint sum of $p$ copies of this diagram. 
By Proposition \ref{wdiagsprop}, taking the expectation over $\mathbf{W}^{(\ell)}$ produces a sum
\[
  \left( \frac{2}{n} \right) ^{ p|E|/2 } \sum_{ \pi \text{ pairing of } pE(B) }^{  } D _{ \pi ,1 } D _{ \pi ,2 } 
\] 
where the number of connected components of $ D _{ \pi,1  } D _{ \pi ,2 } $ is at most $ p(1 + \frac{1}{2}|E|)$. Each of the connected components of $ D _{ \pi ,1 }  $ is a connected fully-contracted tensor diagram of cumulants of $ X ^{ (\ell-1) }  $ so contributes a factor of $ O(n) $, and each of the connected component of $ D _{ \pi ,2 }  $ is a connected fully-contracted tensor diagram of $ I_2 $'s, which also contribute a factor of $ n $. Therefore we get that the whole diagram is $ O(n^p) $ as required.
\end{proof}

\begin{restatable}[]{corollary}{elmtgrowth}
    \label{cor:elementgrowth}
    Let $r\ge 1, \ell\in[L+1], n \ge1 $ and fix $i_1,...,i_r\in [n]$. Then
    \[
    \mathbb{E} \left[ \left| \kappa_r[Z^{(\ell)}] _{i_1,...,i_r} \right|^2 \right] = 
    \begin{cases}
        O( n^{2-r} ) & \parbox[t]{8cm}{\centering if $r$ is even, and each value in $i_1,\ldots,i_r$\\
appears with even multiplicity, } \\
        O( n^{1-r} ) & \parbox[t]{8cm}{\centering otherwise.}
    \end{cases}
    \]
\end{restatable}
\begin{proof}
    Consider $\mathbb{E} \left[ \kappa_r[Z^{(\ell)}]_{i_1,...,i_r} ^p \right]$ for $p \ge 2$ even. Applying Proposition \ref{wdiagsprop}, we get a prefactor of $n^{-\frac{1}{2}pr}$, multiplied by fully-contracted diagram produced by pairing $p$-copies of $\kappa_r[X^{(\ell)}]$. If $r$ is odd, the maximal number of connected components is $\frac{1}{2}p$, giving $n^{\frac{1}{2}p(1-r)}$. If $r$ is even, the maximal number of connected components is $p$, giving $n^{p(1-\frac{1}{2}r)}$.

    If $r$ is even and some index appears an odd number of times, the maximal number of connected components is again $\frac{1}{2}p$, giving $n^{\frac{1}{2}p(1-r)}$-growth.
\end{proof}

We demonstrate two instances of this Corollary in Figure \ref{fig:cumulantgrowth}.

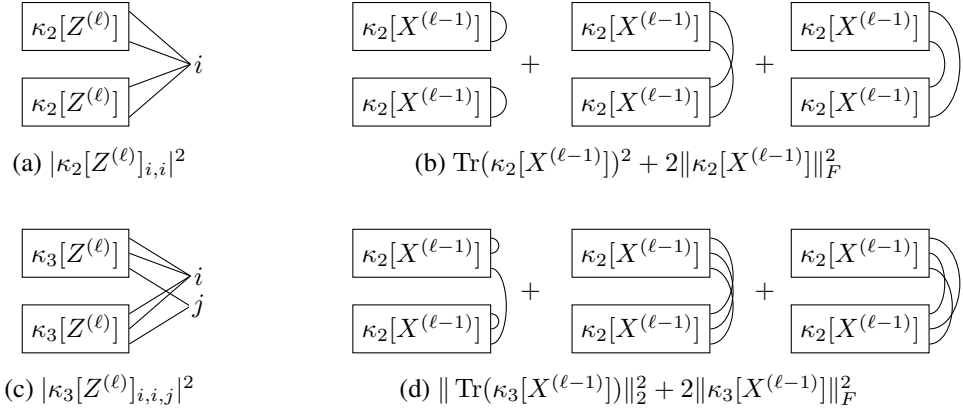
\begin{figure}[ht]
\centering
\begin{tikzpicture}

\node[draw] (Ma1) at (-0.12, 0.5) {$\kappa_2[Z^{(\ell)}]$};
\node[draw] (Ma2) at (-0.12,-0.5) {$\kappa_2[Z^{(\ell)}]$};

\draw (0.59, 0.7) -- (1.4, 0) coordinate (ta1);
\draw (0.59, 0.3) -- (1.4, 0) coordinate (ta2);

\draw (0.59,-0.3) -- (1.4,0) coordinate (ba2);
\draw (0.59,-0.7) -- (1.4,0) coordinate (ba1);
\node[fill=white, inner sep=1.2pt] at (1.51,0) {$i$};

\node at (0.2, -1.3) {(a) $|\kappa_2[Z^{(\ell)}]_{i,i}|^2$};

\begin{scope}[xshift=4.6cm]
  \node[draw] (Ma1) at (-0.14, 0.5) {$\kappa_2[X^{(\ell-1)}]$};
  \node[draw] (Ma2) at (-0.14,-0.5) {$\kappa_2[X^{(\ell-1)}]$};
  \draw (Ma1.east) ++(0,0.2) arc (90:-90:0.2);
  \draw (Ma2.east) ++(0,0.2) arc (90:-90:0.2);
\end{scope}

\node at (5.9, 0) {$+$};

\begin{scope}[xshift=7.5cm]
  \node[draw] (Ma1) at (-0.14, 0.5) {$\kappa_2[X^{(\ell-1)}]$};
  \node[draw] (Ma2) at (-0.14,-0.5) {$\kappa_2[X^{(\ell-1)}]$};
  \draw (Ma1.east) ++(0,0.2) arc (90:-90:0.3 and 0.5);
  \draw (Ma1.east) ++(0,-0.2) arc (90:-90:0.3 and 0.5);
\end{scope}

\node at (9, 0) {$+$};

\begin{scope}[xshift=10.4cm]
  \node[draw] (Ma1) at (-0.14, 0.5) {$\kappa_2[X^{(\ell-1)}]$};
  \node[draw] (Ma2) at (-0.14,-0.5) {$\kappa_2[X^{(\ell-1)}]$};
  \draw (Ma1.east) ++(0,0.2) arc (90:-90:0.4 and 0.7);
  \draw (Ma1.east) ++(0,-0.2) arc (90:-90:0.2 and 0.3);
\end{scope}

\node at (7.2, -1.3) {(b) $\operatorname{Tr}(\kappa_2[X^{(\ell-1)}])^2 + 2\|\kappa_2[X^{(\ell-1)}]\|_F^2$};

\begin{scope}[yshift = -3cm]
\node[draw] (Ma1) at (-0.12, 0.5) {$\kappa_3[Z^{(\ell)}]$};
\node[draw] (Ma2) at (-0.12,-0.5) {$\kappa_3[Z^{(\ell)}]$};

\draw (Ma1.east) ++(0,0.2) -- (1.4, 0.2) coordinate (ta1);
\draw (Ma1.east) -- (1.4, 0.2) coordinate (ta2);
\draw (Ma1.east) ++(0,-0.2) -- (1.4, -0.2) coordinate (ta3);

\draw (Ma2.east) ++(0,0.2) -- (1.4, 0.2) coordinate (ta1);
\draw (Ma2.east) -- (1.4, 0.2) coordinate (ta2);
\draw (Ma2.east) ++(0,-0.2) -- (1.4, -0.2) coordinate (ta3);

\node[fill=white, inner sep=1.2pt] at (1.51,0.2) {$i$};
\node[fill=white, inner sep=1.2pt] at (1.51,-0.2) {$j$};
\end{scope}
\node at (0.2, -4.3) {(c) $|\kappa_3[Z^{(\ell)}]_{i,i,j}|^2$};

\begin{scope}[xshift=4.6cm, yshift=-3cm]
  \node[draw] (Ma1) at (-0.14, 0.5) {$\kappa_2[X^{(\ell-1)}]$};
  \node[draw] (Ma2) at (-0.14,-0.5) {$\kappa_2[X^{(\ell-1)}]$};
  \draw (Ma1.east) ++(0,0.2) arc (90:-90:0.1);
  \draw (Ma1.east) ++(0,-0.2) arc (90:-90:0.2 and 0.5);
  \draw (Ma2.east) ++(0,0.2) arc (90:-90:0.1);
\end{scope}

\node at (5.9, -3) {$+$};

\begin{scope}[xshift=7.5cm, yshift=-3cm]
  \node[draw] (Ma1) at (-0.14, 0.5) {$\kappa_2[X^{(\ell-1)}]$};
  \node[draw] (Ma2) at (-0.14,-0.5) {$\kappa_2[X^{(\ell-1)}]$};
  \draw (Ma1.east) ++(0,0.2) arc (90:-90:0.3 and 0.5);
  \draw (Ma1.east) ++(0,-0.2) arc (90:-90:0.3 and 0.5);
  \draw (Ma1.east) arc (90:-90:0.3 and 0.5);
\end{scope}

\node at (9, -3) {$+$};

\begin{scope}[xshift=10.4cm,yshift=-3cm]
  \node[draw] (Ma1) at (-0.14, 0.5) {$\kappa_2[X^{(\ell-1)}]$};
  \node[draw] (Ma2) at (-0.14,-0.5) {$\kappa_2[X^{(\ell-1)}]$};
  \draw (Ma1.east) ++(0,0.2) arc (90:-90:0.4 and 0.6);
  \draw (Ma1.east) ++(0,-0.2) arc (90:-90:0.3 and 0.5);
  \draw (Ma1.east) arc (90:-90:0.2 and 0.4);
\end{scope}

\node at (7.2, -4.3) {(d) $\|\operatorname{Tr}(\kappa_3[X^{(\ell-1)}])\|_2^2 + 2\|\kappa_3[X^{(\ell-1)}]\|_F^2$};

\end{tikzpicture}
\caption{Two square magnitudes of elements of $\kappa[Z^{(\ell)}]$ tensors (a) and (c). Since the first cumulant has an even number of each index, there is a partition splitting the two copies (the first diagram in (b)), and therefore the growth is $O(n^{-2}n^2)=O(1)$ (recall that we suppress by a factor of $\left(\frac{2}{n}\right)^{\# \text{indices}}$). By contrast, in the second diagram we cannot split the two copies (every diagram in (d) is connected), and so has $O(n^{-3}n)=O(n^{-2})$-growth.}
\label{fig:cumulantgrowth}
\end{figure}


\subsection{Description of polynomial algorithm} \label{sec:polyalgdesc}

For this section, we assume that the activation function is a polynomial $\phi(x)=\sum_{j=0}^m a_j x^j$.

The polynomial algorithm propagates tensors $\tilde \eta_r[X^{(\layernumber)}]\in\syms{r-2s_\maxcumulantorder(r)}{n}$ through MLP layers. These are estimates for $h_{\leq r-2s_\maxcumulantorder(r)}(\kappa_r[X^{(\ell)}])$,
where 
\begin{equation*}
    s_\maxcumulantorder(r):= \begin{cases}
        0 & r \leq \maxcumulantorder\\
        \dfrac{r}{2} & \text{$r=\maxcumulantorder+1$ and $\maxcumulantorder\equiv 1\mod 2$}\\
        \infty & \text{otherwise.}
    \end{cases}
\end{equation*}
By convention, if $2s_\maxcumulantorder(r)>r$, then $\tilde \eta_r[X^{(\layernumber)}]=0$.
Thus, only cumulants with degree $r$ no more than
\begin{equation*}
    R := \begin{cases}
        \maxcumulantorder & \maxcumulantorder\equiv 0\mod 2\\
        \maxcumulantorder+1 & \maxcumulantorder\equiv 1\mod 2
    \end{cases}
\end{equation*}
are tracked at all.

The algorithm proceeds as follows:
\begin{enumerate}[leftmargin=0.5cm]
    \item For $r\in[R]$, initialize tensors
    \begin{align*}
        \tilde \eta_r[X^{(0)}]:=h_{\leq r-2s_\maxcumulantorder(r)}(\kappa_r[X^{(0)}]) = \begin{cases}
            1 & \text{$r = 2$ and $\maxcumulantorder=1$} \\
            \mathbf I_n & \text{$r = 2$ and $\maxcumulantorder>1$} \\
            0 & r \neq 2.
        \end{cases}
    \end{align*}
    \item For $1\leq \layernumber \leq \numhiddenlayers$, repeat the following:
    \begin{enumerate}[leftmargin=0.5cm]
        \item\label{polylinearstep} 
        Compute the tensors
        \begin{align*}
            \bfT _r^{(\layernumber)}&:= \tilde \eta_r[X^{(\layernumber-1)}]\odot \mathbf W^{(\layernumber)} & \forall r\in[R]\\
            \bfM ^{(\layernumber)}_{i,j}&:= \begin{cases}
                (\mathbf W^{(\layernumber)}\mathbf W^{(\layernumber)\top})_{i,j} & \text{$\maxcumulantorder>1$ or $i=j$}\\
                0 & \text{otherwise.}
            \end{cases}
        \end{align*}
        Note that each $\bfT _r^{(\layernumber)}$ is a symmetric contraction of a $(r-2s_\maxcumulantorder(r))$-tensor, and thus takes time
        $O(n^{r-2s_\maxcumulantorder(r)+1})\leq O(n^{\maxcumulantorder+1})$.
        The matrix $\bfM ^{(\ell)}$ takes time $O(n^2)$ if $\maxcumulantorder=1$ and otherwise $O(n^3)$.
        In any case, this is no more than the $O(n^{\maxcumulantorder+1})$ budget.
        \item\label{polypowercumulantsstep} 
        Define $\tilde\kappa_r [Z^{(\ell)}]:= g_{\bf M^{(\ell)}}^{s_\maxcumulantorder(r)}\bfT_{r}^{(\ell)}$. Compute estimated cumulants of $X$ using Lemma~\ref{lem:simplepowers} and multilinearity:
        \begin{equation*}
            \tilde\kappa_r[X^{(\ell)}]_{i_1,\ldots, i_r} = \sum_{\underline{k}\in[m]^r}\left(\prod_{j=1}^r a_{k_j}\right)\sum_{\pi\vdash \bigsqcup_j[k_j]}\prod_{S\in\pi}\tilde\kappa_{|S|} [Z^{(\ell)}]_{i_v: (v,l)\in S}, \qquad \forall r\in [R]
        \end{equation*}
        For $r > R$, define $\tilde\kappa_r[X^{(\ell)}]=0$.
        \item \label{polyharmonicproj}
        Project $\tilde\kappa_r[X^{(\ell)}]$ onto $(\leq r-2s_\maxcumulantorder)$-harmonics using Theorem~\ref{thm:harmonicdecomposition}:
        \begin{equation*}
            \tilde\eta_r[X^{(\ell)}] := \sum_{s'=s}^{\lfloor r/2\rfloor}g^{s'-s}(h_{r-2s'}(\tilde\kappa_r[X^{(\ell)}])),
        \end{equation*}
        where $h_{r-2s'}$ is defined in \eqref{harmonicparteq}.
        By the definition of $s_\maxcumulantorder$, this simplifies to
        \begin{equation*}
            \tilde\eta_r[X^{(\ell)}] = \begin{cases}
                \tilde\kappa_r[X^{(\ell)}] & r \leq \maxcumulantorder \\
                h_0(\tilde\kappa_r[X^{(\ell)}]) & r=\maxcumulantorder +1 \equiv 0 \mod 2\\
                0 & \text{otherwise.}
            \end{cases}
        \end{equation*}

    \end{enumerate}
    \item Return $\mathbf W^{(\numhiddenlayers+1)} \tilde\eta_1[X^{(\numhiddenlayers)}]$.
\end{enumerate}


\subsection{Error analysis}%
\label{sec:polyerror}
In this section, we prove that the polynomial algorithm described above satisfies the conditions of Theorem \ref{matchingsamplingtheorem}. Define the error tensors
\begin{align*}
  \Delta _{ r } [X ^{ (\ell) } ] &:= \kappa _{ r } [X ^{ (\ell) } ] -  \tilde\kappa_r [X^{(\ell)}]\\
  \Delta _{ r } [Z ^{ (\ell) } ] &:= \kappa _{ r } [Z ^{ (\ell) } ] - \tilde\kappa_r [Z^{(\ell)}]
\end{align*}
These are tensor sequences. Our main result is the following theorem:

\begin{theorem}
  \label{thm:main} 
For all $ r \ge 1 $ and $\ell \in [L+1]$, the tensor sequence $ \Delta _{ r } [Z ^{ (\ell) } ] $ has $ n ^{ -(K/2) }  $-growth.
\end{theorem}

First, we will prove a useful lemma.

\begin{lemma}
\label{lem:xdel} 
For any rank $ r $, the tensor sequence
\[
  \Delta _{ r } [X ^{ (\ell) } ]
\] 
can be written as a finite sum of connected rank-$r$ tensor diagrams involving tensor sequences from $ \left\{ \kappa _{ r' } [Z ^{ (\ell) } ] : r' \ge 1 \right\} \cup \left\{ \Delta  _{ r' } \left[ Z ^{ (\ell) }  \right] : r' \ge 1  \right\}  $. Furthermore, each term must contain at least one tensor sequence of the form $ \Delta _{ r' } \left[ Z ^{ (\ell) }  \right]   $.
\end{lemma}
\begin{proof}
By definition of $ \Delta _r \left[ X ^{ (\ell) }  \right]  $, we get directly that
\[
  \Delta _{ r } [X ^{ (\ell) } ] _{ i_1,...,i_r }  = \sum_{ \underline{k}\in[m]^r }^{  } \left( \prod_{ j=1 }^{ r } a _{ k_j }  \right) \sum_{ \pi  }^{  } \left( \prod_{ S \in \pi  }^{  } \kappa \left[ Z _{ i_v } ^{ (\ell) } : (v,l) \in S \right] - \prod_{ S \in \pi  }^{  } \widetilde{ \kappa  } \left[ Z _{ i_v } ^{ (\ell) } : (v,l) \in S \right]  \right) 
\] 
Consider the difference inside the $ \pi  $ summation: since $ \widetilde{ \kappa  }  = \kappa - \Delta  $, we can use the binomial theorem to rewrite this as
\[
  \sum_{ \pi' \subsetneq \pi  }^{  } (-1) ^{ | \pi |- |\pi'| } \left( \prod_{ S \in \pi' }^{  } \kappa \left[ Z _{ i_v } ^{ (\ell) } : (v,l) \in S \right]  \right) \left( \prod_{ S \not\in \pi' }^{  } \Delta \left[ Z _{ i_v } ^{ (\ell) } : (v,l) \in S \right]  \right) .
\] 
For each $ \pi' $, the quantity inside is (the $ i_1,...,i_r $-entry of) a tensor diagram formed as follows:
\begin{itemize}
  \item $ V_1 = \pi , V_2 = [r] $, the edges consist of $ (S,v) $ such that $ (v,l) \in S $ for some $ l $, with multiplicity. 
  \item $ I = \emptyset $
  \item $ T_v = \kappa _{ \mathrm{ deg } (v) } [Z ^{ (\ell) } ] $ if $ v \in \pi' $, and $ T_v = \Delta _{ \deg (v) } [X ^{ (\ell) } ] $ otherwise.
\end{itemize}
The number of diagrams, and their coefficients, are independent of $ n $. This completes the proof of the Lemma.
\end{proof}

\begin{proof}[Proof of Theorem \ref{thm:main}]
We will do this by induction on $ \ell $. By defining $\tilde \kappa _r [Z^{(0)}] := \tilde\kappa _r[X^{(0)}]$, and $\phi^{(0)} := \mathrm{id}$, we may start the induction by proving the result for $\Delta_r[Z^{(0)}]$, which is clear since at $\ell = 0 $ all the error tensors are identically zero.

For the inductive step, assume that $ \Delta _{ r } [Z ^{ (\ell) } ] $ has $ n ^{ -K/2 }  $-growth for some $ \ell \in 0,...,L-1 $. We would like to prove the same result for $ \Delta _{ r } [Z ^{ (\ell+1) } ] $. Let $p \ge 2$ be an even integer (the exponent in the definition of tensor growth).

\underline{\textbf{General case: $K$ is even or $r\neq K+1$}}

It is clear from the definitions that
\[
  \Delta _{ r } [Z ^{ (\ell+1) } ] = \Delta _{ r } [X ^{ (\ell) } ] \odot \mathbf W ^{ (\ell+1) } 
\] 
We will now apply Proposition \ref{wdiagsprop} to the following tensor diagram: $ V_1 = [p], V_2 = [r], E = [p] \times [r] , I = \emptyset, T_v = \Delta _{ r } [X ^{ (\ell) } ] $. Since $p$ is even,
\[
  \mathbb{ E } [ \Delta _{ r } [Z ^{ (\ell+1) } ] _{ i_1,...,i_r } ^{ p }  ] = \left( \frac{2}{n} \right) ^{ rp/2 } \sum_{ \pi  }^{  } \mathbb{E}[D _{ \pi ,1 }] (D _{ \pi ,2 } ) _{ i_1,...,i_r } 
\] 
where the sum is over pairings of $ [p] \times [r] $, the $ D _{ \pi ,1 }  $ are arbitrary tensor diagrams on $ p $ copies of $ \Delta _{ r } [X ^{ (\ell) } ] $ where every vertex on the right has degree 2, and $ I=V_2 $, and the value of $ (D _{ \pi ,2 }) _{ i_1,...,i_r }   $ is always $ 1 $ or $ 0 $.

Therefore, to complete the result, we are reduced to proving that tensor diagrams of the form $ D _{ \pi ,1 }  $ above have $n ^{ \frac{p}{2} (r-K) } $-growth. From Lemma \ref{lem:xdel}, we know that each terms $ \Delta _{ r } [X ^{ (\ell) } ] $ can be written as a sum of connected diagrams of the $ \kappa [Z ^{ (\ell) } ] $ and $ \Delta [Z ^{ (\ell) } ] $ tensors with at least one $ \Delta [Z ^{ (\ell) } ] $ included, and the number and coefficients of these diagrams are independent of $ n $. After replacing each copy of $ \Delta_d [X ^{ (\ell) } ] $ by one of these diagrams, we are reduced to proving that each of the resulting fully-contracted tensor diagrams has $n ^{ \frac{p}{2}(r-K) }$-growth. Let us recap which diagrams these are:

\begin{itemize}
  \item $ V_1 = \bigsqcup _{ i=1 } ^{ p } V _{ 1,i } $
  \item $ |V_2| = \frac{1}{2}pr $.
  \item The edges are arbitrary except for the condition that the restriction to $ V _{ 1,i } $ and its neighbors must be connected for all $ i \in [p] $, and for each $ v \in V_2 $ there are exactly two values of $ i $ for which $ V _{ 1,i }  $ is connected to $ v $ by an edge.
  \item $ I= V_2 $, the tensor diagram is fully contracted.
  \item All of the tensors are either $ \kappa [Z ^{ (\ell) } ] $ or $ \Delta [Z ^{ (\ell) } ] $, and for each $ i \in [p] $ at least one of the tensors attached to $ V _{ 1,i }  $ is a $ \Delta [Z ^{ (\ell) } ] $.
\end{itemize}

Now, we will apply Lemma \ref{lem:cs} by using the decomposition $ V_1 = U_1 \sqcup U_2 $ with $ U_1 $ consisting of all vertices with a $ \kappa [Z ^{ (l) }]  $ tensor attached, and $ U_2 $ all vertices with a $ \Delta [Z ^{ (l) } ] $ tensor. This tells us that
\[
  D _{ \pi ,1 } ^{ 2 } \leq D _{ \pi ,1 } ^{ \kappa  } D _{ \pi ,1 } ^{ \Delta  } .
\] 
What can we say about the tensor diagram $ D _{ \pi ,1 } ^{ \kappa  }  $? It is a fully-contracted tensor diagram where every tensor is of the form $ \kappa [Z ^{ (l) } ] $. The cardinality of $ V_2(D _{ \pi ,1 } ^{ \kappa  } ) $ is the same as $ | V_2(D _{ \pi ,1 } ) |= \frac{1}{2}pr $, and therefore we know by Proposition \ref{prop:cumulantgrowth} that this diagram has $ n ^{ \frac{p}{2}r } $-growth.

It remains to prove that $ D _{ \pi ,1 } ^{ \Delta  }  $ has $ n ^{ \frac{1}{2}pr - pK }  $-growth. This diagram evaluates to a sum of $ n ^{ \frac{1}{2}pr }  $ terms each of which is a product of at least $ 2p $ tensors $ \Delta [Z ^{ (\ell) } ] $ (the factor $ 2 $ comes from the fact that the $ D _{ \pi ,1 } ^{ \Delta  } $ is has double the number of $ \Delta  $'s as $ D _{ \pi ,1 }  $). If we take an individual term, we can bound its expectation using H{\"o}lder's inequality in terms of the bounds for moments of the $ \Delta [Z ^{ (\ell) } ] $'s. Since each term has at least $ 2p $ tensors, we get a uniform bound of $ O(n ^{ -pK } ) $ where the constant is independent of the term we chose.

Since the bound is uniform, we see that $ D _{ \pi ,1 } ^{ \Delta  }  $ has $ n ^{ \frac{1}{2}pr-pK }  $-growth.

\underline{\textbf{Special case: $4\leq r=K+1 \equiv 0 \text{ mod }2$}}

For this case, note that 
\begin{align*}
\Delta_{K+1}[Z^{(\ell+1)}]&= \left(h_{\ge 2} \kappa_{K+1}[X^{(\ell)}] + g^{\frac{K+1}{2}}h_0 \Delta _{K+1}[X^{(\ell)}]\right) \odot \mathbf W^{(\ell+1)}
\end{align*}

Consider the $ p $-th power of $ h_{\ge 2} \kappa _{ K+1 }[ X ^{(\ell)}] \odot \mathbf W ^{ (\ell+1) } $. Applying Proposition \ref{wdiagsprop}, we see that we require all fully-contracted diagrams with $p$ copies of $h_{\ge 2} \kappa _{ K+1 }[ X ^{(\ell)}]$ to have $n^{\frac{1}{2}p}$-growth.

A fully-contracted single copy of $h _{\ge 2}\kappa _{K+1} [X^{(\ell)}]$ is $0$, by definition of the harmonic decomposition. Therefore, it is sufficient to prove that a connected fully-contracted diagram of $h_{\ge2}\kappa_{K+1} [X^{(\ell)}]$'s must have $n$-growth (since no tensor can be a single connected component, the number of connected components must be at most $\frac{1}{2}p$). Suppose that some such connected fully-contracted diagram has $m \ge 1$ tensors. Using the expression
\[
h_{\ge2}\kappa _{K+1}[X^{(\ell)}] = \kappa_{K+1}[X^{(\ell)}] - c_{\frac{K+1}{2},0}^{\mathrm h}g^{\frac{K+1}{2}}\mathrm{tr}^{\frac{K+1}{2}}\kappa_{K+1}[X^{(\ell)}]
\]
we can expand using the binomial theorem, and assume that $0\leq a \leq m$ of the tensors are of the second type. Since $c_{\frac{K+1}{2},0}^{\mathrm h} = O(n^{-\frac{K+1}{2}})$, we gain a factor of $n^{-a\frac{K+1}{2}}$, but the number of connected components can increase by $a\frac{K+1}{2}$. Therefore it suffices to prove that any connected fully-contracted tensor diagram of $\kappa_{K+1} [X^{(\ell)}] $'s has $n$-growth. This follows directly from Proposition \ref{prop:cumulantgrowth}.

Now consider the $p$-th power of $g^{\frac{K+1}{2}}h_0 \Delta_{K+1}[X^{(\ell)}]\odot \mathbf W^{(\ell+1)}$. This is the $p$-th power of the scalar value $h_0\Delta _{K+1} [X^{(\ell)}]$ multiplied by $p\frac{K+1}{2}$ copies of $(\mathbf W ^{(\ell+1)}\mathbf W^{(\ell+1)\top})_{i,j}$, which is easily seen to be $O(1)$. So we simply need to show that $h_0 \Delta_{K+1}[X^{(\ell)}]$ has $n^{-K/2}$-growth. Given that
\[
    h_0 \Delta_{K+1}[X^{(\ell)}] = c_{\frac{K+1}{2},0}^{\mathrm h} \mathrm{Tr}^{\frac{K+1}{2}} \Delta_{K+1}[X^{(\ell)}]
\]
and $c_{\frac{K+1}{2},0}^{\mathrm h} = O(n^{-\frac{K+1}{2}})$, it suffices to prove that $\Delta_{K+1}[X^{(\ell)}]$ has $n^{-\frac{K}{2}}$-growth. This is true since Lemma \ref{lem:xdel} implies that the elements of $\Delta_{K+1}[X^{(\ell)}]$ are products of elements of $\kappa_{r'}[Z^{(\ell)}]$ and at least one element of $\Delta_{K+1}[Z^{(\ell)}]$, so the inductive hypothesis and Corollary \ref{cor:elementgrowth} complete the argument.

\underline{\textbf{Special case: $K=1, r=2$}}

In this case, 
\begin{align*}
\Delta_{2}[Z^{(\ell+1)}]&= \left(h_{ 2} \kappa_{2}[X^{(\ell)}] + gh_0 \Delta _{2}[X^{(\ell)}]\right) \odot \mathbf W^{(\ell+1)} + (g_{\mathbf M ^{(\ell+1)}}-g_{\mathbf W^{(\ell+1)}\mathbf W^{(\ell+1)\top}})h_0 \tilde \kappa_2 [X^{(\ell)}]
\end{align*}

The first term is dealt with exactly as in the previous special case. We simply need to show that the final term has $n^{-\frac{1}{2}}$-growth. The off-diagonal entries of $\mathbf W ^{(\ell+1)} \mathbf W^{(\ell+1)\top}$ have $n^{-\frac{1}{2}}$-growth. Using $\tilde \kappa _2[X^{(\ell)}] = \kappa_2[X^{(\ell)}] - \Delta_2[X^{(\ell)}] $, is suffices to show each of these terms has $O(1)$-growth. This follows exactly as in the final part of the previous special case, along with Corollary \ref{cor:elementgrowth}.
\end{proof}

\section{Mathematical preliminaries for Hermite-based algorithms}
\label{smathprelims}

In this section we introduce the notation and basic results needed to follow the precise description of our cumulant propagation algorithms given in Section \ref{kpropfull}.

\subsection{Partitions and diagonals}\label{diagonalsprereq}
Here, we introduce basic terminology for tensors, partitions, and diagonals.
\begin{definition}[Diagonals of tensors]
    
    Given $\pi\vdash[r]$ and indices $j_1,j_2\in[r]$, we say $j_1\sim_\pi j_2$ when $j_1$ and $j_2$ lie in the same block of $\pi$, and we say the indices $(i_1,\ldots,i_r)$ are of \textit{type} $\pi$ to mean that, for all $j_1,j_2\in[r]$, $i_{j_1}=i_{j_2}$ if and only if $j_1\sim_\pi j_2$. Then, for $\bfT\in\tens{r}{n}$,
    the \textit{$\pi$-diagonal} of $\bfT$ is defined by
    \begin{equation*}
        (\Diag_\pi \bfT)_{i_1,\ldots, i_r} := \begin{cases}
            \bfT_{i_1,\ldots, i_r} & \type(i_1,\ldots, i_r)=\pi\\
            0 & \text{otherwise.}
        \end{cases}
    \end{equation*}
\end{definition}
In this way, each integer partition $\pi\vdash [r]$ defines an orthogonal projection $\Diag_\pi\colon \tens{r}{n}\to\tens{r}{n}$. It is also clear that, for any $\bfT$,
\begin{equation*}
    \bfT = \sum_{\pi\vdash[d]} \Diag_\pi \bfT,
\end{equation*}
whence the images of the $\Diag_\pi$'s in fact form an orthogonal decomposition of $\tens{r}{n}$.

\begin{definition}[Diagonal slices of symmetric tensors]
    We say that $\pi\vdash[r]$ has \textit{type} equal to the \textit{integer partition} $\lambda=(\lambda_1,\ldots,\lambda_b)\vdash r$, 
    written $\type(\pi)=\lambda$, if 
    $\lambda_i$ is the size of the $i\nth$ largest block of $\pi$ and $b$ is the number of blocks in $\pi$.\footnote{
    Here we overload the $\type$ operator. The type of a tuple of indices is a set partition, and the type of a set partition is an integer partition.
    }
    By convention, $\lambda_1\geq \ldots\geq \lambda_b$.
    Suppose now $\bfT $ is \textit{symmetric}, which we write $\bfT \in\syms{r}{n}$.
    Then, all $\Diag_\pi \bfT $'s for $\pi$'s associated with the same $\lambda$ are equal up to permutation, whence
    \begin{equation}
    \label{eq:sum-diag-pi}
        \Diag_\lambda \bfT := \sum_{\substack{\pi\vdash[r]\\\type(\pi)=\lambda}}\Diag_\pi \bfT  \in\syms{r}{n}.
    \end{equation}
    When manipulating diagonals algorithmically, it suffices to track only \textit{diagonal slices}
    \begin{equation}
    \label{eq:def-tilda-diag}
        (\tilde\Diag_\lambda \bfT )_{i_1,\ldots, i_b} := (\Diag_\lambda \bfT )_{\underbrace{\scriptstyle{i_1,\ldots, i_1}}_{\text{$\lambda_1$ times}},\ldots, \underbrace{\scriptstyle{i_b,\ldots, i_b}}_{\text{$\lambda_b$ times}}}.
    \end{equation}
\end{definition}
Note that $\tilde\Diag_\lambda \bfT $ is \textit{not} symmetric.
In fact, writing $\lambda$ in the alternative notation $\lambda=(1^{b_1}2^{b_2}\ldots r^{b_r})$ (denoting that there are $b_i$ blocks of size $i$),
we find that $\tilde\Diag_\lambda \bfT $ is invariant with respect to the action of the Young subgroup $\prod_i S_{b_i}$ induced by the integer partition $(b_1,\ldots, b_d)\vdash b$.
That is, $\tilde\Diag_\lambda \bfT $ is only invariant to permutations of indices corresponding to equally sized blocks of $\lambda$.
We call such tensors $\lambda$-symmetric and denote their subspace $\syms{\lambda}{n}\subseteq\tens{b}{n}$.

We write the Moore-Penrose pseudoinverse of $\tilde\Diag_\lambda$ as $\tilde\Diag_\lambda^{-1}\colon \syms{\lambda}{n}\to\syms{r}{n}$, which has the following explicit form:
\begin{align*}
    (\tilde\Diag_\lambda^{-1}\bfT)_{\underbrace{\scriptstyle{i_1,\ldots, i_1}}_{\text{$\lambda_1$ times}},\ldots, \underbrace{\scriptstyle{i_b,\ldots, i_b}}_{\text{$\lambda_b$ times}}} &= \bfT_{i_1,\ldots, i_b}\\
    (\tilde\Diag_\lambda^{-1}\bfT)_{i_1,\ldots, i_r} &= 0  \qquad \text{if $\type(\type(i_1,\ldots, i_r))\neq \lambda$}.
\end{align*}
Then
\begin{equation*}
    \tilde\Diag_\lambda\tilde\Diag_\lambda^{-1} = \operatorname{id}_{\syms{\lambda}{n}},\qquad
    \tilde\Diag_\lambda^{-1}\tilde\Diag_\lambda = \Diag_\lambda.
\end{equation*}

As a slight abuse of notation, we may consider $\tilde\Diag_{\underline u} \bfT $ for arbitrary (possibly unsorted, not necessarily positive) 
$\underline{u}\in\Z_{\geq 0}^b$ satisfying $\sum_iu_i = r$.
In this case, we treat the diagonal slice as constant along indices corresponding to zero entries of $\underline{u}$,
excepting that we still enforce that any repeated-index entry is zero.
That is,
\begin{equation*}
    (\tilde\Diag_{\underline u} \bfT )_{i_1,\ldots, i_b} := {\mathbbm 1}[\text{all $j_i$ distinct}]\bfT_{\underbrace{\scriptstyle{i_1,\ldots, i_1}}_{\text{$u_1$ times}},\ldots, \underbrace{\scriptstyle{i_b,\ldots, i_b}}_{\text{$u_b$ times}}}.
\end{equation*}
(Note that \eqref{eq:def-tilda-diag} also gives zero entries when any $j_i$ is repeated---this follows from the definition of $\Diag_\lambda$. 
Thus the two definitions of $\tilde\Diag$ we give coincide for $\lambda\vdash d$.)
We let $\tilde\Diag_{\underline{u}}^{-1}$ denote the Moore-Penrose pseudoinverse of $\tilde\Diag_{\underline{u}}$.


\subsection{Combinatorial restatement of Hermite-based diagram summation formula}\label{diagramsprereq}

Theorem \ref{diagramstheorem}, the Hermite-based diagram summation formula for cumulants, 
serves as our basic formula for propagating cumulants through non-linear activation functions. 
We restate the requisite definitions and provide a combinatorial restatement here, and defer the proof to~\ref{diagramsproof}. This combinatorial restatement is possible simply due to the fact that in a diagram (see Definition \ref{def:diagram}) there is an ordering of the points in each $[k_j]$ which doesn't affect the value of the corresponding term. Grouping terms that are equal up to permuting these points gives the combinatorial statement below.

\begin{definition}
Let $r\in\mathbb{Z}_{>0}$ and $\underline k\in\Z^r_{\geq 0}$.
A \textit{vector partition} of $\underline k$ is a tuple of tuples in lexicographically descending order that sum to $\underline k$:
\begin{align*}
    \nu&=(\underline u_1,\ldots, \underline u_b)\in (\Z^r_{\geq 0})^b\\
    \text{s.t.}\quad k_j&=\sum_{i=1}^b u_{i,j}\qquad \forall j\in[r].\\
    \underline{u}_i &\geq_{\text{lex}} \underline{u}_{i+1}\qquad \forall i\in[b-1].
\end{align*}
We call the $\underline u_i$ \textit{blocks} of the vector partition.
Each vector partition is an equivalence class of diagrams modulo the forgetting of the distinction between vertices in the same group $\{(v,w):w\in\{1,\ldots, k_v\}\}$.
In this way, the relationship between diagrams and vector partitions is analogous to the relationship between set partitions and integer partitions;
we thus write $\nu=\type(\pi)$ if the diagram $\pi$ is in the equivalence class induced by the vector partition $\nu$.
We also say that the type of each block $S\in \pi$ is its corresponding vector in $\nu$:
\begin{equation*}
    \type(S)_i = \#\{w\in[k_i]: (v,w)\in S\}.
\end{equation*}
Taking this perspective, we can define the properties \textit{connected} and \textit{$\left[\leq M\right]$-mixed}
analogously to the definitions for diagrams.
We denote the set of all vector partitions of $\underline k$ by $\Vec(\underline{k})$,
and the subset of connected $\left[\leq M\right]$-mixed vector partitions by $\cVec{M}(\underline k)$.
\end{definition}

\begin{restatable}[]{proposition}{veccoef}
\label{prop:vec-coef}
    Fix $\nu\in\Vec(\underline{k})$, and for each $\underline 0\leq \underline{u}\leq \underline{k}$ (entrywise) let $\nu(\underline{u})$
    denote the number of times $\underline{u}$ appears in $\nu$. 
    Then, the number of diagrams with type $\nu$ is
    \begin{equation*}
     c^{\mathrm{vec}}_\nu := \dfrac{\prod_{i=1}^r k_i!}{\prod_{\underline{u}=\underline{0}}^{\underline{k}} \left(\nu(\underline{u})!\prod_{i=1}^r (u_i!)^{\nu(\underline{u})}\right)}.
    \end{equation*}
\end{restatable}

\begin{restatable}[Combinatorial restatement of \ref{diagramstheorem}]{theorem}{diagrestat}
\label{thm:diagrams-restatement}
Let $Z\in\R^n$ be a random variable with finite moments and let $\phi\colon \R\to\R$ be a polynomial,
which we think of as acting coordinatewise on $Z$. Then
\begin{equation}
\label{eq:diagram-restatement}
\begin{aligned}
    \kappa_r[\phi(Z)]_{j_1,\ldots j_r} 
    = \sum_{\underline{k}\in\Z_{\geq 0}^r} 
            \left(\prod_{v=1}^r \dfrac{1}{k_v!} \widehat{\phi}_{k_v}^{(\E[Z_{j_v}], \Var(Z_{j_v})}\right)
            \sum_{\nu\in\cVec{2}(\underline{k})}c_\nu^{\mathrm{vec}}\prod_{\underline{u}\in\nu}(\tilde\Diag_{\underline{u}}\kappa_{|\underline{u}|}[Z])_{j_1,\ldots, j_r}.\\
\end{aligned}
\end{equation}
\end{restatable}

Both of these results are proven in Section \ref{combrestatproof}. When $\phi$ is not a polynomial, the sum in \eqref{eq:diagram-restatement} is infinite and possibly diverges (for polynomials $\phi$, all Hermite coefficients of degree larger than $\mathrm{deg}(\phi)$ are zero, so this summation is finite).
It turns out that, for large enough $n$, a careful truncation of the series 
via a condition \eqref{eq:vec-part-weight} on vector partitions 
results in an approximation of the cumulant
that is sufficient to obtain the desired MSE bound; see Section~\ref{sec:nonpoly-activation} for the analysis.
Our cumulant propagation algorithm uses this truncation.

\subsection{Power cumulants}\label{powercumulantsprereq}
Section \ref{diagramsprereq} provides a summation formula for cumulants after applying a coordinatewise activation function. 
As mentioned in Section \ref{kpropsubsection}, we will apply this formula to powers of the activation function in order to compute power cumulants. 
\begin{definition}
    Let $X$ be an $\mathbb{R}^n$-valued random variable, and $r\ge 1$. For a multi-exponent $\underline{\alpha} \in \mathbb{Z}_{\ge 1}^r $, define the power cumulant
    $$
        \kappa _{ \underline{\alpha} } ^{\mathsf P} [X] _{i_1,...,i_r} = \begin{cases}
            \kappa[X_{i_1}^{\alpha_1},...,X_{i_r}^{\alpha_r}] \qquad & \text{if } i_1,...,i_r \text{ are pairwise distinct} \\
            0 &\text{o/w.}
        \end{cases}
    $$
    This is a tensor in $\syms{\alpha}{n}$, the set of tensors invariant under permutations in $\mathrm{Stab}_{S_r}(\alpha)$.
\end{definition}

The standard cumulant tensors can be written as a sum of products of these tensors with delta functions.

\begin{restatable}[Power cumulant partition formula]{proposition}{powercumulantsformula}\label{powercumulantsprop}
    Cumulants can be computed from power cumulants slicewise: for each $\lambda\vdash r$ with $b$ blocks,
    \begin{equation*}
    \label{eq:powercumulants}
        \tilde\Diag_\lambda \kappa_r[X] = \sum_{\nu\in\Vec(\lambda)}  c_\nu^{\mathrm{vec}}c_\nu\pow\prod_{\underline{u}\in\nu}\tilde \Diag_{\underline{u}}\kappa_{|\underline{u}|}\pow[X],
    \end{equation*}
    where the product is taken elementwise and 
    \begin{equation*}
        c_\nu\pow := \sum_{\omega\vdash\nu}\mathbbm{1}[\text{$\omega$ totally disconnected}](-1)^{|\omega|-1}(|\omega|-1)!.
    \end{equation*}
    Here the sum is over set partitions $\omega$ of $\nu$ (which itself is a set of tuples in $[b]$),
    $|\omega|$ is the number of blocks in $\omega$, and $\omega$ is said to be totally disconnected
    if for each block of $\omega$, all tuples in that block have disjoint support in $[b]$.
    
\end{restatable}

We defer the proof of Proposition~\ref{powercumulantsprop} to Section~\ref{powercumulantsproof}.

\subsection{Hermite coefficients of ReLU}
\label{sec:hermite-relu}
In this section, we provide explicit expressions for the Hermite coefficients of ReLU.
\begin{proposition}
\label{prop:hermite-relu}
    Let $\phi=\ReLU$.
    Fix $\mu\in\R$ and $\sigma^2\in\R_{>0}$ and let $Z\sim\mathcal{N}(\mu,\sigma^2)$.
    Write $\alpha=\mu/\sigma$, let $\He_i$ denote the $i\nth$ standard probabilist's Hermite polynomial,
    and let $\varphi$ and $\Phi$ denote the standard Gaussian probability density function and cumulative distribution function, respectively.
    Then, for any $k\in\Z_{\geq 0}$ and $p\in\Z_{>0}$,
    \begin{equation*}
        \widehat{\phi^p}^{(\mu,\sigma^2)}_k = \begin{cases}
            p! \Phi(\alpha) & k=p\\
            p!(-1)^{k-p-1}\sigma^{-(k-p)}\He_{k-p-1}(\alpha)\varphi(\alpha) & k>p\\
            \sigma^{p-k}(P_{1,p,k}(\alpha)\varphi(\alpha) + P_{2,p,k}(\alpha)\Phi(\alpha)) & k < p
        \end{cases}
    \end{equation*}
    where
    \begin{align*}
        P_{1,p,k}(\alpha)&:=(p-k+1)_{k} \sum_{j=0}^{p-k} \binom{p-k}{j}\alpha^{p-k-j} \left( \sum_{m=0}^{\lfloor (j-1)/2 \rfloor} \binom{j}{2m} (2m-1)!! \He_{j-2m-1}(-\alpha)\right)\\
        P_{2,p,k}(\alpha)&:=(p-k+1)_{k} \sum_{j=0}^{p-k} \mathbbm 1[\text{$j$ even}]\binom{p-k}{j}\alpha^{p-k-j}(j-1)!!.
    \end{align*}
    and $(a)_t := a(a+1)...(a+t-1)$ is the Pochhammer symbol.
\end{proposition}
\begin{proof}
    This is a scaling and shifting of the $(\mu,\sigma^2)=(0,1)$ case, which is calculated in~\citet{Kagawa2015}. 
\end{proof}

\subsection{Going between diagonals and harmonics}
\label{diagonal-harmonic}
\begin{definition}
    A \textit{multigraph} $\gamma$ on a vertex set $[b]$ is a multiset of unordered pairs of (possibly nondistinct) elements from $[b]$, 
    or equivalently a mapping $\gamma\colon \{\{i,j\}\subseteq[b]\}\to\mathbb{N}$.
    The total number of edges of $\gamma$ is $|\gamma|:=\sum_{\{i,j\}\subseteq [b]} \gamma(\{i,j\})$,
    and the degree of a vertex $j\in[b]$ is
    \begin{equation*}
        d_\gamma(j):=\sum_{j'\neq j}\gamma(\{j,j'\})+2\gamma(\{j,j\}).
    \end{equation*}
    We say $\gamma$ is \textit{totally disconnected (t.d.)} if it contains only self-loops, i.e.\ $\gamma(\{i,j\})=0$ for any $i\neq j$.
    Let $\Gamma([b], m)$ denote the set of multigraphs on $[b]$ with $m$ edges, and $\Gamma_{\text{t.d.}}([b],m)$ the subset of those
    multigraphs that are totally disconnected.
\end{definition}


\begin{restatable}[]{proposition}{tracedelta}
\label{prop:trace-delta}
    There exist coefficients $c^{\mathrm{g}}_{\lambda,\gamma}$ such that,
    for any $\bfT \in\syms{r}{n}$, any integer $0\leq s\leq \lfloor \frac{r}{2}\rfloor$, and any $\lambda=(\lambda_1,\ldots,\lambda_b)=(1^{b_1}\ldots r^{b_r})\vdash r$ with $b$ blocks,
    \begin{equation}
    \label{eq:trace-delta}
        \tr^s \Diag_\lambda \bfT  = \sum_{\gamma\in\Gamma_{\text{t.d.}}([b],s)}\left(c^{\mathrm{g}}_{\lambda,\gamma}{\mathbbm 1}_{\lambda\geq d_\gamma} \tilde\Diag^{-1}_{(\lambda-d_\gamma)}\operatorname*{Sum}_{\{j: \lambda_j=d_\gamma(j)\}} \tilde\Diag_\lambda \bfT \right),
    \end{equation}
    where $\operatorname{Sum}_{\{j: \lambda_j=d_\gamma(j)\}}$ is the operator that sums along indices $\{j:\lambda_j=d_\gamma(j)\}$ and the difference between $\lambda$ and $d_\gamma$ is taken entrywise.
    Explicitly,
    \begin{equation*}
        c^{\mathrm g}_{\lambda,\gamma}:=
        \dfrac{2^ss!r!\prod_{i\in[b]} (\lambda_i-d_\gamma(i)+1)_{d_\gamma(i)}}
        {(r-2s+1)_{2s} \prod_{j\in[r]} b_j!\prod_{i\in[b]}\lambda_i!2^{\gamma(\{i\})} \prod_{\{i,j\}\subseteq[b]} \gamma(\{i,j\})!}.
    \end{equation*}
\end{restatable}
Provided $\tilde\Diag_\lambda \bfT $, \eqref{eq:trace-delta} can be computed in $O(n^{r-2s})$ time, faster than the $O(n^r)$ time of na\"ively applying $\tilde\Diag^{-1}_\lambda$ then $\tr^s$ in sequence.

\begin{restatable}[]{proposition}{deltacup}
\label{prop:delta-cup}
    There exist coefficients $\tilde c^{\mathrm{g}}_{\lambda,\gamma}$ such that,
    given an integer $0\leq s\leq \lfloor \frac{r}{2}\rfloor$ and $\lambda=(\lambda_1,\ldots,\lambda_b)=(1^{b_1}\ldots r^{b_r})\vdash r$ with $b$ blocks, 
    for any $\bfT \in\syms{r}{n}$ and  $\mathbf M\in\syms{2}{n}$,
    \begin{equation}
    \label{eq:delta-cup}
        \tilde\Diag_\lambda  g^s_{\mathbf M}(\bfT ) = \sum_{\gamma\in\Gamma([b],s)}\tilde c^{\mathrm{g}}_{\lambda,\gamma}\bfT \odot_\gamma \mathbf M,
    \end{equation}
    where $(\bfT \odot_\gamma \mathbf M)\in\syms{\lambda}{n}$ is defined by
    \begin{equation*}
        (\bfT \odot_\gamma \mathbf M)_{i_1,\ldots,i_b} = \begin{cases}
            (\tilde\Diag_{\lambda-d_\gamma} \bfT )_{i_1,\ldots, i_b} \prod_{\{u,v\}\in\gamma} \mathbf M_{i_u,i_v}  & \text{$\underbar i$ all distinct}\\
            0 & \text{otherwise.}
        \end{cases}
    \end{equation*}
    Explicitly,
    \begin{equation*}
        \tilde c^{\mathrm g}_{\lambda,\gamma}:=
        \dfrac{\prod_{i\in[b]} (\lambda_i-d_\gamma(i)+1)_{d_\gamma(i)}}
        {\prod_{i\in[b]}2^{\gamma(\{i\})} \prod_{\{i,j\}\subseteq[b]} \gamma(\{i,j\})!}.
    \end{equation*}
\end{restatable}
Equation~\ref{eq:delta-cup} can be computed in $O(n^b)$ time, faster than the $O(n^{r+2s})$ time of na\"ively applying $g_{\mathbf W}$ and $\Diag_\lambda$ in sequence.

We defer the proofs to Section~\ref{sec:proof-harmonic-diag}.
\section{Descriptions of Hermite-based algorithms}\label{kpropfull}
In this section, we describe in detail the basic cumulant propagation algorithm (Section~\ref{sec:basic-algo}),
as well as the 
augmented (Section~\ref{sec:augmented-algo})
and factorized (Section~\ref{sec:factorized-algo}) variants.

Our basic cumulant propagation estimator takes $O(n^{\maxcumulantorder+1})$ time and incurs $O(n^{-\maxcumulantorder})$ MSE for all \textit{integral} $\maxcumulantorder\geq 0$,
thus matching sampling for non-integral $\maxcumulantorder$ as well (by taking the floor).
The factorized version of this algorithm takes only $O(n^{K})$ time to compute the same estimate (and thus obtain the same MSE).
The augmented version improves the constant factor of the leading-order term in MSE without worsening the leading-order term in time (assuming unfactorized).

Propagating the harmonic projections through a linear layer $\mathbf W$ is conceptually easy (but is asymptotically the most expensive step): 
we simply apply the symmetric contraction $\odot \mathbf W$ and replace cups with $\mathbf W\mathbf W^\top$-cups.
For nonlinearities, we use the formulae in Section~\ref{diagonal-harmonic} to move from the harmonic to the diagonal decomposition,
apply the power cumulant combinatorics in the diagonal setting, then project back to harmonics.

\subsection{Basic algorithm}
\label{sec:basic-algo}
The input to the algorithm is the following data:
\begin{itemize}
    \item The parameters of an MLP (as defined in Section \ref{problemstatementsection}). That is, $L \ge 0 $, $n \ge 1$, activation functions $\phi_1,...,\phi_L: \mathbb{R} \rightarrow \mathbb{R}$, and weights $\boldsymbol\theta = (\mathbf W^{(1)},...,\mathbf W^{(L+1)}) \in \mathbb{R} ^{(L+1)\times n\times n}$.
    \item The ``maximum degree'' $K \ge 1$. The choice of $K$ as a function of the error tolerance $\varepsilon$ is discussed in Section \ref{sec:basicstructure}.
\end{itemize}

Given these parameters, we define the following notation:

\begin{itemize}
    \item Let $X = X^{(0)}\sim\mathcal{N}(0,\mathbf I_n)$ denote the input to the MLP. Then recursively define
    \begin{align*}
        \forall 1 \leq \layernumber \leq L+1, \qquad &Z^{(\layernumber)} := \mathbf W^{(\layernumber)}X^{(\layernumber-1)} \\
        \forall 1 \leq \layernumber \leq L, \qquad & X^{(\layernumber)} = \phi_\layernumber(Z^{(\layernumber)}).
    \end{align*}
    \item Recall $(\widehat{\phi_\layernumber^p})_k^{(\mu,\sigma^2)}$ denotes the $k$-th probabilist's Hermite coefficient of $\phi_\ell^p$ (the pointwise $p$th power of the $\ell$-th layer activation $\phi_\ell$) with respect to $\mathcal{N}(\mu, \sigma^2)$, as defined in \eqref{eq:hermite-def}.
\end{itemize}

Our algorithm then propagates tensors $\tilde \eta_r[X^{(\layernumber)}]\in\syms{r-2s_\maxcumulantorder(r)}{n}$ through MLP layers.
These are estimates for $h_{\geq r-2s_\maxcumulantorder(r)}(\kappa_r[X^{(\ell)}])$,
where 
\begin{equation*}
    s_\maxcumulantorder(r):= \begin{cases}
        0 & r \leq \maxcumulantorder\\
        \dfrac{r}{2} & \text{$r=\maxcumulantorder+1\equiv 0\mod 2$}\\
        \infty & \text{otherwise.}
    \end{cases}
\end{equation*}
By convention, if $2s_\maxcumulantorder(r)>r$, then $\tilde \eta_r[X^{(\layernumber)}]=0$.
Thus, only cumulants with degree $r$ no more than
\begin{equation*}
    R := \begin{cases}
        \maxcumulantorder & \maxcumulantorder\equiv 0\mod 2\\
        \maxcumulantorder+1 & \maxcumulantorder\equiv 1\mod 2
    \end{cases}
\end{equation*}
are tracked at all.

The algorithm proceeds as follows:
\begin{enumerate}[leftmargin=0.5cm]
    \item Initialize tensors 
    \begin{align*}
        \tilde \eta_r[X^{(0)}]=h_{\geq r-2s_\maxcumulantorder(r)}(\kappa_r[X^{(0)}]) = \begin{cases}
            1 & \text{$r = 2$ and $\maxcumulantorder=1$} \\
            \mathbf I_n & \text{$r = 2$ and $\maxcumulantorder>1$} \\
            0 & r \neq 2.
        \end{cases}
    \end{align*}
    \item For $1\leq \layernumber \leq \numhiddenlayers$, repeat the following:
    \begin{enumerate}[leftmargin=0.5cm]
        \item\label{basiclinearstep} 
        Compute the tensors
        \begin{align*}
            \bfT _r^{(\layernumber)}&:= \tilde \eta_r[X^{(\layernumber-1)}]\odot \mathbf W^{(\layernumber)} & \forall r\in[R]\\
            \bfM ^{(\layernumber)}&:= \begin{cases}
                \Diag_{(2)}(\mathbf W^{(\layernumber)}\mathbf W^{(\layernumber)\top}) & \maxcumulantorder=1\\
                \mathbf W^{(\layernumber)}\mathbf W^{(\layernumber)\top} & \maxcumulantorder>1.
            \end{cases}
        \end{align*}
        Note that each $\bfT _r^{(\layernumber)}$ is a symmetric contraction of a $(r-2s_\maxcumulantorder(r))$-tensor, and thus takes time
        $O(n^{r-2s_\maxcumulantorder(r)+1})\leq O(n^{\maxcumulantorder+1})$.
        The matrix $\bfM ^{(\ell)}$ takes time $O(n^2)$ if $\maxcumulantorder=1$, otherwise $O(n^3)$.
        In any case, this is no more than the $O(n^{\maxcumulantorder+1})$ budget.
        \item\label{basicpowercumulantsstep} 
        Write the estimated pre-activation mean and variance vectors
        \begin{equation*}
            \mu:= \bfT _1^{(\layernumber)}\in\R^n,\qquad \sigma^2:= \tilde\Diag_{(2)}g_{\bfM ^{(\ell)}}^{s_\maxcumulantorder(2)} \bfT _2^{(\layernumber)}\in\R^n.
        \end{equation*}
        For each $r\in[\maxcumulantorder]$ and $\lambda\vdash r$, compute the $\lambda$-slice of the activation power cumulant
        using a truncation of \eqref{eq:diagram-restatement} applied to powers of the activation function:
        \footnote{
            In practice, instead of summing diagrams directly, we sum orbits of diagrams under permutations of indices then $\lambda$-symmetrize the result.
        }
        \begin{align}
        \label{eq:algo-nonlin-sum}
            &\tilde \Diag_\lambda(\tilde\kappa^{\mathsf P}[X^{(\layernumber)}])_{i_1,\ldots, i_b} \\
            &:= \sum_{\substack{\underline{k}\in\mathbb{Z}^b_{\ge0} \\ |\underline{k}| \leq 2K-1}}
            \frac{1}{\underline{k}!}\left(\prod_{a=1}^b\left(\widehat{(\phi_\layernumber^{\lambda_a})}^{(\mu_{k_a},\sigma^2_{k_a})}_{k_a}\right)_{i_a}\right)
            \sum_{\substack{\nu\in\cVec{2}(\underline{k})\\ k(\nu)\leq\maxcumulantorder}}
            \prod_{\underline{u}\in\nu}(\tilde \Diag_{\underline{u}}g_{\bfM ^{(\ell)}}^{s_\maxcumulantorder(|\underline{u}|)} \bfT _{|\underline{u}|}^{(\layernumber)})_{i_1,\ldots, i_b}, \nonumber
        \end{align}
        where each $\tilde \Diag_{\underline{u}}g_{\bfM ^{(\layernumber)}}^{s_\maxcumulantorder(|\underline{u}|)} \bfT _{|S|}^{(\layernumber)}$ is computed using \eqref{eq:delta-cup},
        each Hermite coefficient is computed using Proposition~\ref{prop:hermite-relu} if $\phi_\ell=\ReLU$ and Gauss-Hermite otherwise (see Section~\ref{otheractivationsappendix}),
        and
        \begin{equation}
        \label{eq:vec-part-weight}
            1+\sum_{\underline{u}\in\nu}\left(\sum_{j=1}^bu_j-1-\mathbbm{1}[\text{$u_j$ all even}]\right)=:k(\nu)\leq\maxcumulantorder
        \end{equation}
        is the condition for including vector partition $\nu$.
        If $\maxcumulantorder$ is odd, do the same for all $\lambda\vdash \maxcumulantorder+1$ having no singleton blocks. 
        \item \label{basicmobius}
        For the same set of $\lambda$ as in the previous step, convert power cumulants to cumulants slicewise:
        \begin{equation*}
            \tilde\Diag_\lambda \tilde\kappa_r[X^{(\layernumber)}] := \sum_{\nu\in\Vec(\lambda)}  c_\nu^{\mathrm{vec}}c_\nu\pow\prod_{\underline{u}\in\nu}\tilde \Diag_{\underline{u}}\kappa_r\pow[X^{(\layernumber)}],
        \end{equation*}
        (This is \eqref{eq:powercumulants} applied to the estimated power cumulants.)
        \item \label{basicharmonicproj}
        For each $r\in[R]$, project the estimated cumulants onto the $(\leq r-2s_\maxcumulantorder)$-harmonics:
        \begin{align*}
            \eta_r[X^{(\layernumber)}]:=\sum_{s'=s}^{\lfloor r/2 \rfloor}g^{s'-s}\left(\sum_{t=0}^{\lfloor (r-2s)/2 \rfloor} c^{\mathrm h}_{s',t} g^t\tr^{s'+t}\left(\sum_{\lambda\vdash[r]}\tilde\Diag_\lambda \tilde\kappa_r[X^{(\layernumber)}]\right)\right)
        \end{align*}
        where the trace is distributed over the sum of diagonals then computed with \eqref{eq:trace-delta}.

    \end{enumerate}
    \item Return $\mathbf W^{(\numhiddenlayers+1)} \tilde\eta_1[X^{(\numhiddenlayers)}]$.
\end{enumerate}


\subsection{Augmented algorithm}
\label{sec:augmented-algo}

The basic algorithm in Section~\ref{sec:basic-algo} works by tracking every harmonic part that would incur at least $\Omega(n^{-\maxcumulantorder+1})$ MSE if dropped;
the largest of these parts takes $\Theta(n^{\maxcumulantorder+1})$ time to propagate.
We may consider an algorithm that additionally tracks every harmonic part that can be propagated in $O(n^{\maxcumulantorder})$ time
and contributes at least $\Omega(n^{-\maxcumulantorder})$ MSE.
This augmented algorithm would improve the constant factor of the leading-order term in MSE, while keeping the leading-order term in FLOPs the same.\footnote{
As long as $\maxcumulantorder\geq 3$. For lower $\maxcumulantorder$, the linear step is not the only computation that is leading order in FLOPs, so
augmenting cumulant propagation in fact increases the leading-order FLOPs term.
}
It follows from our bounds on cumulant sizes (Proposition~\ref{prop:cumulantgrowth}) that the augmented algorithm corresponds to replacing $s_\maxcumulantorder$ in the basic algorithm with
\begin{equation*}
    s^{\mathrm{aug}}_{\maxcumulantorder}(r) := \begin{cases}
        0 & r\leq \maxcumulantorder\\
        1 & r = \maxcumulantorder+1\\
        \dfrac{r}{2} & \text{$r=\maxcumulantorder+2\equiv 0\mod 2$}\\
        \infty & \text{otherwise.}
    \end{cases}
\end{equation*}
All other parts of cumulant propagation remain the same.

\subsection{Factorized algorithm}
\label{sec:factorized-algo}

The time complexity of our cumulant propagation algorithm is dominated by the $O(n^{\maxcumulantorder+1})$ symmetric contraction of the full degree-$\maxcumulantorder$ cumulant tensor during the linear step.
By representing this tensor in a factorized form, we are able to reduce the linear step's runtime, and thus that of the entire algorithm, to $O(n^{\maxcumulantorder})$.

Note that this improvement concerns only the width dependence---recall that our big-$O$ notation treats the depth $L$ as constant.
If we track the dependence on $L$ explicitly, the time complexity of basic cumulant propagation through the entire MLP
scales linearly with $L$, while that of the factored version scales quadratically.
The reason is that the factored version propagates a running list of factors is propagated through the network, 
adding a constant number of new factors at each layer.

For ease of exposition, we demonstrate the factorized form of the basic algorithm when $\maxcumulantorder=3$,
and defer the proof that general $\maxcumulantorder$ can be factorized to Section~\ref{sfactorized}.
With some additional effort, augmented cumulant propagation can also be factorized; 
we refer interested readers to the code (\redactifanon{\url{https://github.com/alignment-research-center/mlp_kprop}}).

\paragraph{Factoring tensors}
When $\maxcumulantorder=3$, we factor $T\in\syms{3}{n}$ as follows:
\begin{equation}
\label{eq:factor3}
    \bfT_{i_1,i_2,i_3}=\dfrac{1}{6}\sum_{\sigma\in S_3}\sum_{j=1}^J \bfA _{i_{\sigma(1)},j} \bfB _{i_{\sigma(2)},j} \bfC _{i_{\sigma(3)},j}
\end{equation}
where $J$ is the number of factors and  $\bfA ,\bfB ,\bfC \in\R^{n\times J}$.
Throughout, $J=O(n)$.
We notate \eqref{eq:factor3} sans indices as
\begin{equation*}
    \bfT = F(\bfA , \bfB , \bfC ).
\end{equation*}

It is easily verified that sums of factored tensors are factored by concatenation:
\begin{equation*}
    F(\bfA ,\bfB ,\bfC )+F(\bfA ',\bfB ',\bfC ') = F\left( \begin{bmatrix}\bfA &\bfA '\end{bmatrix}, \begin{bmatrix}\bfB &\bfB '\end{bmatrix}, \begin{bmatrix}\bfC &\bfC '\end{bmatrix}\right).
\end{equation*}
Hence summing factored tensors sums the number of factors.

The symmetric contraction of a factored tensor is 
\begin{equation*}
    F(\bfA ,\bfB ,\bfC )\odot\mathbf W = F(\mathbf W \bfA , \mathbf W \bfB , \mathbf W \bfC ).
\end{equation*}
When $J=O(n)$, the factors of the symmetric contraction can be computed in $O(n^3)$ time---this is what gives us the asymptotic improvement over the na\"ive $O(n^4)$ linear step.

\paragraph{Diagonals of factored tensors}
Let $*$ denote the entrywise (Hadamard) matrix product. Then,
\begin{equation}
\label{eq:delta-factor}
\begin{aligned}
    (\tilde \Diag_{(3)}F(\bfA ,\bfB ,\bfC ))_i &= \sum_{j=1}^J (\bfA *\bfB *\bfC )_{i,j},\\
    (\tilde \Diag_{(2,1)}F(\bfA ,\bfB ,\bfC ))_{i_1,i_2} &= {\mathbbm 1}[i_1\neq i_2] \dfrac{1}{3}\sum_{j=1}^J \left((\bfB *\bfC )_{i_1,j} \bfA _{i_2,j} + (\bfA *\bfC )_{i_1,j} \bfB _{i_2,j} + (\bfA *\bfB )_{i_1,j} \bfC _{i_2,j} \right).
\end{aligned}
\end{equation}
When $J=O(n)$, these can be computed in $O(n^3)$ time.

\paragraph{Factoring diagonals}
For $\bfT\in\syms{3}{n}$,
\begin{equation}
\label{eq:factor-delta}
\begin{aligned}
    \Diag_{(3)} \bfT &= F((\tilde\Diag_{(3)} \bfT)\mathbf{1}_n\tran, \mathbf I_n, \mathbf I_n),\\
    \Diag_{(2,1)} \bfT  &= 3F(\tilde\Diag_{(2,1)} \bfT , \mathbf I_n, \mathbf I_n).
\end{aligned}
\end{equation}
where $\mathbf{1}_n\in\R^n$ is the all-ones vector and $\mathbf{I}_n\in\R^{n\times n}$ is the identity matrix.

\paragraph{Factoring the diagram summation formula}
We now consider factoring \eqref{eq:algo-nonlin-sum} when $\lambda=(1,1,1)$.
Suppose we are at layer $\layernumber$ of cumulant propagation.
Let $\mu$ and $\sigma^2$ be the propagated estimates for mean and variance respectively, and write
\begin{equation*}
\bff _k^{(\layernumber)} := \widehat{\phi_\layernumber}_k^{(\mu_k,\sigma_k^2)},\qquad
\bfQ ^{(\layernumber)}_{\underline u} := \tilde\Diag_{\underline{u}} g_{\bfM ^{(\layernumber)}}^{s_3(\underline{u})} \bfT ^{(\layernumber)}_{|\underline{u}|},
\qquad\bfM ^{(\layernumber)} := \mathbf{W}\mathbf{W}\tran.
\end{equation*}
where $\bfT $ comes from Step~\ref{basiclinearstep}.
Recall each $\bfQ _{\underline u}^{(\layernumber)}$ is computed using \eqref{eq:delta-cup}.
Then, grouping diagrams by orbits under permutations of indices,
\begin{align*}
    &(\tilde\Diag_{(1,1,1)}\tilde\kappa^{\mathsf P}[X^{(\layernumber)}])_{i_1,i_2,i_3} =
    {\mathbbm 1}[\text{$i_j$ all distinct}] \bigg(\\
        &\quad\sum_{\sigma\in A_3}\sum_{k_1,\ldots,k_4\in[2]} \dfrac{1}{\underline{k}!}
        (\bff ^{(\layernumber)}_{k_1})_{i_{\sigma(1)}} (\bff ^{(\layernumber)}_{k_2+k_3})_{i_{\sigma(2)}} (\bff ^{(\layernumber)}_{k_4})_{i_{\sigma(3)}} (\bfQ ^{(\layernumber)}_{(k_1,k_2)})_{i_{\sigma(1)}, i_{\sigma(2)}} (\bfQ ^{(\layernumber)}_{(k_3,k_4)})_{i_{\sigma(2)}, i_{\sigma(3)}}\\
        &\qquad+ (\bff ^{(\layernumber)}_1)_{i_{\sigma(1)}} (\bff ^{(\layernumber)}_1)_{i_{\sigma(2)}} (\bff ^{(\layernumber)}_1)_{i_{\sigma(1)}} (\bfQ ^{(\layernumber)}_{(1,1,1)})_{i_1,i_2,i_3}\\
        &\qquad+ \sum_{\sigma\in A_3}\dfrac{1}{2}(\bff ^{(\layernumber)}_2)_{i_{\sigma(1)}} (\bff ^{(\layernumber)}_1)_{i_{\sigma(2)}} (\bff ^{(\layernumber)}_1)_{i_{\sigma(3)}} (\bfQ ^{(\layernumber)}_{(2,1,1)})_{i_{\sigma(1)},i_{\sigma(2)},i_{\sigma(3)}}\\
    &\bigg),
\end{align*}
where $A_3\subseteq S_3$ is the group of 3-cycles.\footnote{
In general, the orbit-stabilizer theorem tells us to sum over coset representatives of the quotient of $S_3$ by the stabilizer of the diagram (or sum of diagrams) under consideration.
For example, the stabilizer of $\{(2,1,1)\}$ is $\langle (23) \rangle$ (the subgroup generated by the transposition of 2 and 3), while the stabilizer of 
$\{(1,1,1)\}$ is all of $S_3$.
}
We factor each summand, ignoring for now the ${\mathbbm 1}[\text{all $i_j$ distinct]}$ factor.
\begin{itemize}[leftmargin=0.5cm]
\item The sum of two-legged diagrams factors as
    \begin{align}
    \label{eq:factor-twoleg}
        &\sum_{\sigma\in A_3}\sum_{k_1,\ldots,k_4\in[2]} \dfrac{1}{\underline{k}!}
        (\bff ^{(\layernumber)}_{k_1})_{i_{\sigma(1)}} (\bff ^{(\layernumber)}_{k_2+k_3})_{i_{\sigma(2)}} (\bff ^{(\layernumber)}_{k_4})_{i_{\sigma(3)}} (\bfQ ^{(\layernumber)}_{(k_1,k_2)})_{i_{\sigma(1)}, i_{\sigma(2)}} (\bfQ ^{(\layernumber)}_{(k_3,k_4)})_{i_{\sigma(2)}, i_{\sigma(3)}}\nonumber\\
        &\,= \sum_{k_2\in[2]}F\left(\dfrac{3}{k_2!}\mathbf I_n, \sum_{k_1\in[2]}\dfrac{1}{k_1!}(\diag \bff ^{(\layernumber)}_{k_1}) \bfQ ^{(\layernumber)}_{(k_1,k_2)}, \sum_{k_3,k_4\in[2]}\dfrac{1}{k_3!k_4!}(\diag \bff ^{(\layernumber)}_{k_2+k_3})(\diag \bff ^{(\layernumber)}_{k_4}) \bfQ ^{(\layernumber)}_{(k_3,k_4)}  \right)_{i_1,i_2,i_3}\nonumber\\
        &\,=:\sum_{k_2\in[2]}F\left(\bfA ^{(\layernumber)}_{\text{two},k_2}, \bfB ^{(\layernumber)}_{\text{two},k_2}, \bfC ^{(\layernumber)}_{\text{two},k_2}\right)_{i_1,i_2,i_3}.
    \end{align}
    (Since $F$ by definition symmetrizes, the sum over $A_3$ in the left hand side turns into a factor of 3, which is pulled into the first factor of the right hand side.)
    This is $2n$ factors total.
\item
By induction, the previous layer's activation cumulants corresponding to $(1,1,1)$ are factored as $\tilde\kappa_3[X^{(\layernumber-1)}]=F(\bfA ^{(\layernumber-1)},\bfB ^{(\layernumber-1)},\bfC ^{(\layernumber-1)})$.
Then, since $\bfQ ^{(\layernumber)}_{(1,1,1)}=\tilde\kappa_3[X^{(\layernumber-1)}]\odot\mathbf W$,
\begin{align}
\label{eq:factor-111}
    &(\bff ^{(\layernumber)}_1)_{i_{\sigma(1)}} (\bff ^{(\layernumber)}_1)_{i_{\sigma(2)}} (\bff ^{(\layernumber)}_1)_{i_{\sigma(3)}} (\bfQ ^{(\layernumber)}_{(1,1,1)})_{i_1,i_2,i_3}\nonumber\\
    &\qquad=F((\diag \bff ^{(\layernumber)}_1)\mathbf W \bfA ^{(\layernumber-1)}, (\diag \bff ^{(\layernumber)}_1)\mathbf W \bfB ^{(\layernumber-1)}, (\diag \bff ^{(\layernumber)}_1)\mathbf W \bfC ^{(\layernumber-1)})_{i_1,i_2,i_3}\nonumber\\
    &\qquad=:F\left(\bfA ^{(\layernumber)}_{(1,1,1)}, \bfB ^{(\layernumber)}_{(1,1,1)}, \bfC ^{(\layernumber)}_{(1,1,1)}\right)_{i_1,i_2,i_3}.
\end{align}
This is the same number of factors that $\tilde\kappa_3[X^{(\layernumber-1)}]$ had in the previous layer.
\item To factor the $(2,1,1)$ diagram, write $(j_1,j_2,j_3,j_4)=(1,1,2,3)$.
Then, since $\bfT ^{(\layernumber)}_4$ is a scalar and $\bfM ^{(\layernumber)}$ is symmetric,
\begin{align*}
    (\bfQ ^{(\layernumber)}_{(2,1,1)})_{i_1,i_2,i_3} &= \left(\tilde\Diag_{(2,1,1)}g_{\bfM ^{(\layernumber)}}^{2} \bfT ^{(\layernumber)}_4\right)_{i_1,i_2,i_3}\\
    &=\bfT ^{(\layernumber)}_4 \sum_{\substack{u,v\in[4]\\u< v}} \bfM ^{(\layernumber)}_{i_{j_u},i_{j_v}}\bfM ^{(\layernumber)}_{\{i_{j_s}: s\not\in\{u,v\}\}}\\
    &=\bfT ^{(\layernumber)}_4 \left(2 \bfM ^{(\layernumber)}_{i_1,i_1} \bfM ^{(\layernumber)}_{i_2,i_3} + 4 \bfM ^{(\layernumber)}_{i_1, i_2} \bfM ^{(\layernumber)}_{i_1, i_3} \right).\nonumber
\end{align*}
Hence,
\begin{align}
\label{eq:factor-211}
    &\sum_{\sigma\in A_3}(\bff ^{(\layernumber)}_2)_{i_{\sigma(1)}}\dfrac{1}{2} (\bff ^{(\layernumber)}_1)_{i_{\sigma(2)}} (\bff ^{(\layernumber)}_1)_{i_{\sigma(3)}} (\bfQ^{(\layernumber)}_{(2,1,1)})_{i_{\sigma(1)},i_{\sigma(2)},i_{\sigma(3)}}\nonumber\\
    &\quad= F(3\bfT ^{(\layernumber)}_4 (\diag \bff ^{(\layernumber)}_2) (\Diag_{(2)} \bfM ^{(\layernumber)})\mathbf{1}_n\tran, (\diag \bff ^{(\layernumber)}_1), (\diag \bff ^{(\layernumber)}_1) \bfM ^{(\layernumber)})_{i_1,i_2,i_3}\nonumber\\
    &\qquad + F(6\bfT ^{(\layernumber)}_4(\diag \bff ^{(\layernumber)}_2), (\diag \bff ^{(\layernumber)}_1) \bfM ^{(\layernumber)}, (\diag \bff ^{(\layernumber)}_1) \bfM ^{(\layernumber)})_{i_1,i_2,i_3}\nonumber\\
    &\quad =: \sum_{k=1}^2 F\left(\bfA ^{(\layernumber)}_{(2,1,1),k}, \bfB ^{(\layernumber)}_{(2,1,1),k}, \bfC ^{(\layernumber)}_{(2,1,1),k}\right)_{i_1,i_2,i_3}.
\end{align}
This is $2n$ factors total.
\end{itemize}

\paragraph{Factorized cumulant propagation algorithm}
We get the factorized version of the $\maxcumulantorder=3$ basic cumulant propagation algorithm by replacing Steps~\ref{basiclinearstep}--\ref{basicharmonicproj} with the following:
\begin{enumerate}[label=(\alph*), leftmargin=1cm]
    \item From the previous layer, we have
    \begin{align*}
        &\tilde\eta_1[X^{(\layernumber-1)}]\in\syms{1}{n},\\
        &\tilde\eta_2[X^{(\layernumber-1)}]\in\syms{2}{n}\\
        &\bfA ^{(\layernumber-1)},\bfB ^{(\layernumber-1)},\bfC ^{(\layernumber-1)}\in\R^{n\times J^{(\layernumber-1)}}\\
        &\tilde\eta_4[X^{(\layernumber-1)}]\in\syms{1}{n}.
    \end{align*}
    where implicitly $\tilde\eta_3[X^{(\layernumber-1)}]=F(\bfA ^{(\layernumber-1)},\bfB ^{(\layernumber-1)},\bfC ^{(\layernumber-1)})$,
    but we never explicitly compute this tensor.
    Compute and store
    \begin{align*}
        \bfT_r^{(\layernumber)} &:= \tilde\eta_r[X^{(\layernumber-1)}]\odot \mathbf W^{(\layernumber)}\qquad \forall r\in\{1,2,4\}\\
        \bfM ^{(\layernumber)}&:=\mathbf W^{(\layernumber)}\mathbf W^{(\layernumber)\top},
    \end{align*}
    which take time $O(n^3)$,
    as well as $\mathbf W^{(\layernumber)} \bfA ^{(\layernumber-1)},\mathbf W^{(\layernumber)} \bfB ^{(\layernumber-1)},\mathbf W^{(\layernumber)} \bfC ^{(\layernumber-1)}$.
    which take time $O(n^2 J^{(\layernumber-1)})$.
    \item Compute power cumulants via diagram summation:
    \begin{enumerate}[label=\roman*., leftmargin=0.5cm, ref=\theenumi{} \roman*]
        \item
        \label{step:factor-linear}
        As in Step~\ref{basicpowercumulantsstep}, compute the estimated pre-activation mean and variance vectors.
        Then, for each $r\in[4]$ and $\lambda\vdash r$, if $\lambda$ has no more than 2 blocks and $\lambda\neq (3,1)$, compute $\tilde \Diag_\lambda(\tilde\kappa^{\mathsf P}[X^{(\layernumber)}])$ 
        using \eqref{eq:algo-nonlin-sum}.
        Note that this sum involves terms of the form
        \begin{equation*}
            \tilde\Diag_{\lambda} \bfT_3^{(\layernumber-1)} = \tilde\Diag_{\lambda} F(\mathbf W^{(\layernumber)} \bfA ^{(\layernumber-1)},\mathbf W^{(\layernumber)} \bfB ^{(\layernumber-1)},\mathbf W^{(\layernumber)} \bfC ^{(\layernumber-1)}).
        \end{equation*}
        for $\lambda\in\{(3),(2,1)\}$.
        Compute these terms using \eqref{eq:delta-factor} applied to the cached values of the $\mathbf W^{(\layernumber)}$-transformed previous-layer factors from Step~\ref{step:factor-linear}.
        \item \label{step:factor-111}
        Compute the $(1,1,1)$ slice in factored form.
        Explicitly,
        \begin{align*}
            \hat \bfA ^{(\layernumber)} &:= \begin{bmatrix}
                \bfA ^{(\layernumber)}_{\text{two},1} & \bfA ^{(\layernumber)}_{\text{two},2} & \bfA ^{(\layernumber)}_{(1,1,1)} & \bfA ^{(\layernumber)}_{(2,1,1),1} & \bfA ^{(\layernumber)}_{(2,1,1),2}
            \end{bmatrix}\\
            \hat \bfB ^{(\layernumber)} &:= \begin{bmatrix}
                \bfB ^{(\layernumber)}_{\text{two},1} & \bfB ^{(\layernumber)}_{\text{two},2} & \bfB ^{(\layernumber)}_{(1,1,1)} & \bfB ^{(\layernumber)}_{(2,1,1),1} & \bfB ^{(\layernumber)}_{(2,1,1),2}
            \end{bmatrix}\\
            \hat \bfC ^{(\layernumber)} &:= \begin{bmatrix}
                \bfC ^{(\layernumber)}_{\text{two},1} & \bfC ^{(\layernumber)}_{\text{two},2} & \bfC ^{(\layernumber)}_{(1,1,1)} & \bfC ^{(\layernumber)}_{(2,1,1),1} & \bfC ^{(\layernumber)}_{(2,1,1),2}
            \end{bmatrix}
        \end{align*}
        with the factors defined as in \eqref{eq:factor-twoleg},~\ref{eq:factor-111},~\ref{eq:factor-211}.
        (The $(1,1,1)$ factors make use of the $\mathbf W^{(\layernumber)}$-transformed previous-layer factors from Step~\ref{step:factor-linear}.)
        Then 
        \begin{equation*}
            \tilde \Diag_{(1,1,1)}(\tilde\kappa^{\mathsf P}[X^{(\layernumber)}])_{i_1,i_2,i_3} = {\mathbbm 1}[\text{all $i_j$ distinct}]F(\hat \bfA ^{(\layernumber)}, \hat \bfB ^{(\layernumber)}, \hat \bfC ^{(\layernumber)})_{i_1,i_2,i_3},
        \end{equation*}
        though we never explicitly evaluate this expression for the full tensor $\tilde \Diag_{(1,1,1)}(\tilde\kappa^{\mathsf P}[X^{(\layernumber)}])$.
        Note that this step adds $4n$ factors to those propagated from the previous layer.
    \end{enumerate}
    \item
    \label{step:factor-mobius}
    For $r\in[4]$ and $\lambda\vdash r$ with no more than 2 blocks and excluding $(3,1)$, compute $\tilde\Diag_\lambda \tilde\kappa_r[X^{(\layernumber)}]$ using \eqref{eq:powercumulants} as in Step~\ref{basicmobius} of the basic algorithm.
    \item Form the $\tilde\eta_r[X^{(\layernumber)}]$ tensors:
    \begin{enumerate}[label=\roman*., leftmargin=0.5cm, ref=\theenumi{} \roman*]
    \item For $r\in\{1,2,4\}$, compute $\eta_r[X^{(\layernumber)}]$ as in Step~\ref{basicharmonicproj} of the basic algorithm,
    using $\tilde\Diag_\lambda \tilde\kappa_r[X^{(\layernumber)}]$ computed in Step~\ref{step:factor-mobius}.
    \item On the $(1,1,1)$-diagonal, cumulants and power cumulants coincide. Thus the factorization $F(\hat \bfA ^{(\layernumber)}, \hat \bfB ^{(\layernumber)}, \hat \bfC ^{(\layernumber)})$
    computed in Step~\ref{step:factor-111} agrees with $\eta_3[X^{(\layernumber)}]$ on the $(1,1,1)$-diagonal, but we need to subtract out the other $\lambda$-diagonals and replace them with the 
    values from Step~\ref{step:factor-mobius}.
    That is, for $\lambda\in\{(3,),(2,1)\}$, define the diagonal adjustment
    \begin{equation*}
        \bfD_\lambda := \tilde\Diag_\lambda \tilde\kappa_r[X^{(\layernumber)}]-\tilde\Diag_\lambda F(\hat \bfA ^{(\layernumber)}, \hat \bfB ^{(\layernumber)}, \hat \bfC ^{(\layernumber)})
    \end{equation*}
    where the first term on the right hand side was computed in Step~\ref{step:factor-mobius} and the second is computed using \eqref{eq:delta-factor}.
    Factor $\bfD_\lambda$ using \eqref{eq:factor-delta}:
    \begin{equation*}
        \bfD_\lambda = F(\bfA ^{(\layernumber)}_\lambda, \bfB ^{(\layernumber)}_\lambda, \bfC ^{(\layernumber)}_\lambda)\qquad \forall\lambda\in\{(3),(2,1)\}
    \end{equation*}
    then append the factored diagonals to the existing factors:
    \begin{align*}
        \bfA ^{(\layernumber)} &:= \begin{bmatrix}
            \hat \bfA ^{(\layernumber)}& \bfA ^{(\layernumber)}_{(3)}& \bfA ^{(\layernumber)}_{(2,1)}
        \end{bmatrix},\\
        \bfB ^{(\layernumber)} &:= \begin{bmatrix}
            \hat \bfB ^{(\layernumber)}& \bfB ^{(\layernumber)}_{(3)}& \bfB ^{(\layernumber)}_{(2,1)}
        \end{bmatrix},\\
        \bfC ^{(\layernumber)} &:= \begin{bmatrix}
            \hat \bfC ^{(\layernumber)}& \bfC ^{(\layernumber)}_{(3)}& \bfC ^{(\layernumber)}_{(2,1)}.
        \end{bmatrix}
    \end{align*}
    We thus obtain factors $\bfA ^{(\layernumber)},\bfB ^{(\layernumber)},\bfC ^{(\layernumber)}$ satisfying 
    \begin{equation*}
            \tilde\eta_3[X^{(\layernumber)}] = F(\bfA ^{(\layernumber)}, \bfB ^{(\layernumber)}, \bfC ^{(\layernumber)}),
    \end{equation*}
    The diagonal adjustment adds another $2n$ factors. In total, $J^{(\layernumber)}=J^{(\layernumber-1)}+6n$.
    \end{enumerate}
\end{enumerate}
The number of factors at layer $\layernumber$ is then $J^{(\layernumber)}=6n\ell$.
The leading-order in time computation at each layer is the propagation of cumulant factors through the linearity (Step~\ref{step:factor-linear}),
which takes time $O(n^2J^{(\layernumber-1)})=O(n^3\ell)$.
Thus, the time complexity for an $\numhiddenlayers$-layer MLP is $O(n^3 L^2)$,
which is an improvement in $n$-dependence but worsening in $L$-dependence compared to the basic algorithm's $O(n^4 L)$.

\section{Proofs for Hermite-based algorithms with polynomial activation functions}\label{smainproof}

In this Section we will prove Theorem \ref{matchingsamplingtheorem}, using the basic algorithm described in Section \ref{sec:basic-algo}. Note that we still assume that the activation functions $\phi_1,...,\phi_L$ are all polynomials.

\subsection{Sizes of true cumulants}
\label{sec:truecumulants}

The results of Section \ref{sec:cumulantsizes} were independent of the algorithm, depending only on the underlying MLP with polynomial activations, and so hold again. We restate the result for convenience:

\cumulantgrowth*

\subsection{Error analysis}
\label{sec:errors}
As before, define
\begin{align*}
    \Delta _r [X^{(\ell)}] &:= \kappa_r[X^{(\ell)}] - \tilde\kappa_r[X^{(\ell)}] \\
        \Delta _r [Z^{(\ell)}] &:= \kappa_r[Z^{(\ell)}] - \tilde\kappa_r[Z^{(\ell)}]
\end{align*}

To analyze the basic algorithm, we must understand the error incurred by truncating the Hermite-based diagram summation formula for cumulants (see step (\ref{basicpowercumulantsstep}) of the description of the basic algorithm in Section \ref{sec:basic-algo}), as well as the asymptotic behavior of the Hermite coefficients. This can be easily done for polynomial activation functions (see Proposition \ref{prop:polytame}), and we believe it to be possible for quite general activation functions (see the discussion in Section \ref{sec:nonpoly-activation}). For now, we formalize this analytic requirement in the following definition.

\begin{restatable}[Tameness]{definition}{tame}
    \label{def:tameness}
    Let $(Z_n)_{n\ge 1} $ be a sequence of $\mathbb{R}^n$-valued random variables. We say that a function $\phi: \mathbb R \to \mathbb R$ is \textit{tame with respect to} $(Z_n)_{n\ge1}$ if: 
    \begin{enumerate}
        \item\label{tameerror} for all $k_0 \ge 0$, $r\ge 1$ and $\underline{\alpha}\in \mathbb{Z}_{\ge1}^r$, the tensor sequence
    \begin{align*}
    \label{residualtensor}
    R_{\underline{\alpha},k_0}^{\mathsf P}&\left[\phi(Z)\right]_{j_1,...,j_r}:= 1(j_1,...,j_r \text{ distinct}) \nonumber\\
    &\cdot\left(\kappa_{\underline{\alpha}}^{\mathsf P}[\phi(Z) ]-\sum_{\substack{\underline{k}\in\mathbb{Z}^r_{\ge0} \\ |\underline{k}|\leq k_0}}\frac{1}{\underline{k}!}\prod_{a=1}^r\left(\widehat{\phi^{\alpha_a}}^{(\mathbb E\left[Z\right],\operatorname{Var}\left[Z\right])}_{k_a}\right)_{j_a}\sum_{\pi \in \mathrm{cDia}_{[\leq2]}(\underline{k})}\prod_{S\in\pi}\kappa\left[Z_{j_v}:\left(v,w\right)\in S\right]\right)
    \end{align*}
    has $n^{-\frac{k_0+1}{4}}$-growth. 
    \item\label{tamehermite} For any $\alpha,k \ge 0$, the rank-$1$ tensor sequences $\widehat{\phi^{\alpha}}_k^{(\mathbb{E}[Z],\mathrm{Var}[Z])}$ have $O(1)$-growth. 
    \item\label{tamehermiteerror} If $\mu,\sigma^2\in \mathbb{R}^n$ are rank-$1$ tensor sequences such that $\mathbb{E}[Z]-\mu$ and $\mathrm{Var}[Z]-\sigma^2$ have $n^{-k_0}$-growth, then 
    \[
    \Delta\widehat{\phi^{\alpha}}_k := \widehat{\phi^{\alpha}}_k^{(\mathbb{E}[Z],\mathrm{Var}[Z])}- \widehat{\phi^{\alpha}}_k^{(\mu,\sigma^2)}
    \]
    has $n^{-k_0}$-growth.
    \end{enumerate}
\end{restatable}

We prove that polynomials are tame with respect to $Z^{(\ell)}$ in Section \ref{sec:proofoftameness}. Our main result is the following theorem, which does not assume that the activation functions are polynomials:

\begin{theorem}
    \label{thm:hermiteerror}
     Suppose that we use the basic algorithm to estimate cumulants of an MLP. Assume that 
     \begin{enumerate}
        \item\label{condtame} for all $\ell\in[L]$, $\phi_{\ell}$ is tame with respect to $Z^{(\ell)}$; and
        \item\label{condsizes} for all $\ell\in [L+1]$, the conclusion of Proposition \ref{prop:cumulantgrowth} holds: connected fully-contracted diagrams of (pre-)activation cumulants have $n$-growth.
    \end{enumerate}
    Then, for all $r\ge 1$ and $ \ell\in [L+1]$, the tensor sequence $\Delta_r[Z^{(\ell)}]$ has $n^{-(K/2)}$ growth.
\end{theorem}

As stated, we are focusing just on the basic algorithm (the error of the factorized algorithm is precisely the same, the only difference is in the run-time, which is discussed in Section \ref{sec:factorized-algo}, and the extra harmonic components computed in the augmented algorithm are dealt with using identical ideas to those presented here).

\begin{proof}[Proof of Theorem \ref{thm:hermiteerror}]
    We prove this by induction on $\ell$. We can define $Z^{(0)} := X^{(0)}$ and $\phi_0 = \mathrm{id}$ to make the base case $\ell=0$ which is trivial, since the error tensors are all identically $0$. For simplicity, we will just deal with the case $r\neq K+1\neq 2$. The other cases are dealt with in exactly the same way, with some additional notational difficulty due to the harmonic projections.
    
    Assume that for some $\ell\in 0,...,L-1$, the conclusion holds that for all $r\ge1$, the tensor $\Delta_r[Z^{(\ell)}]$ has $n^{-(K/2)}$-growth. We wish to prove the result for the errors $\Delta_d[Z^{(\ell+1)}]$. Let $p\ge 2$ be an even integer (the exponent in the definition of tensor growth).

    For $r \ge K+2$, note that the basic algorithm assigns $\tilde{\kappa}_r[Z^{(\ell+1)}] = 0$, so it suffices to prove that the tensor $\kappa_r[Z^{(\ell+1)}]$ has $n^{-K/2}$-growth. This follows precisely from Corollary \ref{cor:elementgrowth}. For $1\leq r \leq K$, we simply have 
    $$
    \Delta_r[Z^{(\ell+1)}] = \Delta_r[X^{(\ell)}] \odot \mathbf W^{(\ell+1)} .
    $$
    Therefore, by Proposition \ref{wdiagsprop},
    $$
        \mathbb{E}\left[ \Delta_r[Z^{(\ell+1)}]_{i_1,...,i_r} ^{p} \right] = \left( \frac{2}{n} \right) ^{pr/2} \sum_{\pi\in\mathcal{P}_2([p]\times[r]) } \mathbb{E}[D_{\pi,1}] (D_{\pi,2})_{i_1,...,i_r}
    $$
    where the diagrams $D_{\pi,1}$ are now fully-contracted  tensor diagrams with $p$ copies of $\Delta_r[X^{(\ell)}]$, and the value of $D_{\pi,2}$ is always $0$ or $1$. We must prove that each of the diagrams $D_{\pi,1}$ has $n^{p(r-K)/2}$-growth.

    \underline{\textbf{Error of power cumulants:}}
    
    The first step is to show that $\Delta_{\underline{\alpha}}^{\mathsf{P}}[X^{(\ell)}]$ has $n^{-K/2}$-growth, where this tensor is defined in the obvious way. By tameness of $\phi^{(\ell)}$, for any $j_1,...,j_r$ distinct,
    \[
    \kappa_{\underline{\alpha}}^{\mathsf P} [X^{(\ell)}] - \sum_{\substack{\underline{k}\in\mathbb{Z}^r_{\ge0} \\ |\underline{k}|\leq 2K-1}}\frac{1}{\underline{k}!}\left(\prod_{a=1}^r\widehat{\phi_{\ell}^{\alpha_a}}^{\left(\mathbb E\left[Z_{j_a}^{(\ell)}\right],\operatorname{Var}\left[Z_{j_a}^{(\ell)}\right]\right)}_{k_a}\right)\sum_{\pi \in \mathrm{cDia}_{[\leq2]}(\underline{k})}\prod_{S\in\pi}\kappa\left[Z_{j_v}^{(\ell)}:\left(v,w\right)\in S\right]
    \]
    has $n^{-K/2}$-growth, and therefore to prove that the power cumulants have $n^{-K/2}$-growth, it suffices to prove this for the difference
    \[
    \prod_{a=1}^r\widehat{\phi_{\ell}^{\alpha_a}} ^{\left(\mathbb{E} \left[Z_{j_a}^{(\ell)}\right],\operatorname{Var}\left[Z_{j_a}^{(\ell)}\right]\right)}_{k_a}\prod_{S\in\pi}\kappa\left[Z_{j_v}^{(\ell)}:\left(v,w\right)\in S\right] - 
    \prod_{a=1}^r\widehat{\phi_{\ell}^{\alpha_a}} ^{\left(\tilde{\kappa}_1 \left[Z_{j_a}^{(\ell)}\right],\tilde{\kappa}_2\left[Z^{(\ell)}\right]_{j_a,j_a}\right)}_{k_a}\prod_{S\in\pi}\tilde\kappa\left[Z_{j_v}^{(\ell)}:\left(v,w\right)\in S\right]
    \]
    for any $\underline{k}\in \mathbb{Z}_{\ge0}^r$ with $|\underline{k}|\leq 2K-1$, and $\pi\in \mathrm{cDia}_{[\leq2]}(\underline{k})$ satisfying $k(\pi)\leq K$ where
    \[
    k(\pi) := 1 + |\underline{k}| - |\pi|-\#\{S\in\pi : \forall v, |S\cap [k_v]| \equiv 0 \text{ mod }2\}
    \]
    and for 
    \[
    \prod_{a=1}^r\widehat{\phi_{\ell}^{\alpha_a}} ^{\left(\mathbb{E} \left[Z_{j_a}^{(\ell)}\right],\operatorname{Var}\left[Z_{j_a}^{(\ell)}\right]\right)}_{k_a}\prod_{S\in\pi}\kappa\left[Z_{j_v}^{(\ell)}:\left(v,w\right)\in S\right]
    \]
    for any $\underline{k}\in \mathbb{Z}_{\ge0}^r$ with $|\underline{k}|\leq 2K-1$, and $\pi\in \mathrm{cDia}_{[\leq2]}(\underline{k})$ satisfying $k(\pi) > K$. The second case is simplest so we deal with that first. Let $p'\ge2$ be even. Taking the $p'$-th power of the product of cumulants and then taking the expected value, Proposition \ref{wdiagsprop} tells us that we get a power $n^{-\frac{1}{2}p'|\underline{k}|}$ in front of a sum of fully-contracted diagrams of cumulants of $X^{(\ell-1)}$, and we need to compute the number of connected components. Any cumulant corresponding to $S\in \pi$ which has an odd intersection with some $[k_v]$ must be connected to some other cumulant, therefore the maximal number of connected components is
    \[
    \frac{p'}{2}\left(|\pi|+\#\{S\in\pi : \forall v, |S\cap [k_v]| \equiv 0 \text{ mod }2\}\right)
    \]
    Consequently (by assumption (\ref{condsizes})), we see that the growth is
    \[
    n^{\frac{p'}{2}(|\pi|+\#\{S\in\pi : \forall v, |S\cap [k_v]| \equiv 0 \text{ mod }2\}-|\underline{k}|)} = n^{\frac{p'}{2}(1-k(\pi))}= O(n^{-\frac{Kp'}{2}}),
    \]
    as required, provided that the product of Hermite coefficients has $n^0$-growth, which is indeed condition (\ref{tamehermite}) of tameness.

    For the first case, we use the binomial theorem to expand the product after writing $\tilde\phi = \phi - \Delta\phi$ and $\tilde\kappa = \kappa-\Delta$, and so we can look simply at a term of the form
    \[
    \prod_{a\not\in T}\widehat{\phi_{\ell}^{\alpha_a}} ^{\left(\mathbb{E} \left[Z_{j_a}^{(\ell)}\right],\operatorname{Var}\left[Z_{j_a}^{(\ell)}\right]\right)}_{k_a}\prod_{S\not\in\pi'}\kappa\left[Z_{j_v}^{(\ell)}:\left(v,w\right)\in S\right]
    \prod_{a\in T}\Delta\widehat{\phi_{\ell}^{\alpha_a}} _{k_a}\prod_{S\in\pi'}\Delta\left[Z_{j_v}^{(\ell)}:\left(v,w\right)\in S\right]
    \]
    where $T\subseteq [r], \pi'\subseteq\pi$ and at least one is non-empty. We know that the Hermite terms are $n^0$-growth, and the error terms have $n^{-K/2}$-growth (for the Hermite terms, this is condition (\ref{tamehermiteerror}) of tameness). The cumulant term has growth $\#\{S\not\in\pi'\} + \#\{S\not\in\pi' : \forall v, |S\cap[k_v]|\equiv0 \text{ mod }2\} $. Since $\pi$ is $[\leq2]$-mixed, any $S$ in the first group has size at least $2$, and any $S$ in the second group has size at least $4$, therefore, 
    \[
        \#\{S\not\in\pi'\} + \#\{S\not\in\pi' : \forall v, |S\cap[k_v]|\equiv0 \text{ mod }2\} \leq \frac{1}{2}\sum_{S\not\in \pi'}|S|.
    \]
    Applying Proposition \ref{wdiagsprop} (which holds by assumption (\ref{condsizes})) shows that the product over $S\not\in\pi'$ has $O(1)$-growth, and this term is complete.
    
    \underline{\textbf{Reduction to power cumulants:}}
    By applying Proposition \ref{powercumulantsprop} to both $\kappa_r[X^{(\ell)}]$ and $\tilde\kappa_r[X^{(\ell)}]$, we see that $\Delta_r[X^{(\ell)}]$ can be replaced by a sum of connected diagrams including the power cumulant errors $\Delta_{\alpha}^{\mathsf P}[X^{(\ell)}]$ and the cumulants $\kappa_{r'}[X^{(\ell)}]$. Replacing all $p$ copies of $\Delta_r[X^{(\ell)}]$ with this sum, we get a collection of diagrams which have the following form:
    \begin{itemize}
        \item The left vertices $V_1$ partition into $s$ parts $V_{1,1},...,V_{1,p}$, each of which contains at least one $\Delta^P$ power cumulant. 
        \item There are at most $pr/2$ indices in $V_2$.
        \item Each index connects to at most two parts $V_{1,i}$.
        \item Each part $V_{1,i}$ together with its neighbors forms a connected subgraph.
    \end{itemize}

    Again, we are required to prove that these diagrams have $n^{p(r-K)/2}$-growth. By Cauchy-Schwartz, Lemma \ref{lem:cs}, we can relate the square of diagrams to a diagram that contains $\Delta$'s and the diagram that contains $\kappa$'s. The $\kappa$ part will have at most $pr/2$ connected components, and so will have $n^{pr/2}$-growth (by assumption (\ref{condsizes})), which means that we are required to prove that the $\Delta$-diagram has $n^{p(\frac{1}{2}r-K)}$-growth. They are fully contracted tensor diagrams with at least $2p$ tensors, each with $n^{-K/2}$-growth, and at most $pr/2$ indices, proving the result.
\end{proof}

\subsection{Proof of tameness for polynomial activations}
\label{sec:proofoftameness}

\begin{proposition}
\label{prop:polytame}
    Polynomials are tame. More precisely, if $Z$ is the output distribution of an MLP with polynomial activation functions, then any polynomial $\phi$ is tame with respect to $Z$.
\end{proposition}
\begin{proof}
First, we prove condition (\ref{tamehermite}) of the tameness condition. Let $\alpha,k\ge 0$, then 
\[
(\mu,\sigma^2) \mapsto \widehat{\phi^{\alpha}}_{k}^{(\mu,\sigma^2)}
\]
is a polynomial map. We can consider each monomial individually, for example $\mu^a(\sigma^2)^b$ for $a,b\ge 1$. We get, for $p\ge2$ even, by Proposition \ref{wdiagsprop},
\[
\mathbb{E}\left[ \left(\kappa_1[Z^{(\ell)}]_i^a\kappa_2[Z^{(\ell)}]_{i,i}^b\right)^p\right] = \left(\frac{2}{n}\right)^{\frac{1}{2}p(a+2b)} \sum_{\pi} D_{\pi}
\]
where $D_{\pi}$ is the diagram formed from pairing $pa$ copies of $\kappa_1$ and $pb$ copies of $\kappa_2$. The most connected components we can produce is $\frac{1}{2}pa + pb$, and so this term is $O(1)$, completing the proof of (\ref{tamehermite}).

Next, we prove condition (\ref{tamehermiteerror}). By expanding the Hermite coefficient into monomials, and using the binomial theorem to express every term as a monomial in $\Delta\mu_i, (\Delta \kappa_2)_{i,i}, \kappa_1[Z^{(\ell)}]$ and $\kappa_2[Z^{(\ell)}]_{i,i}$, we have to prove that if $a,b,c,d\ge0$ with $a+b \ge0$, then
\[
(\Delta\mu_i)^a (\Delta \kappa_2)_{i,i}^b \kappa_1[Z^{(\ell)}]_i^c \kappa_2[Z^{(\ell)}]_{i,i}^d
\]
has $n^{-K/2}$-growth. Since, by assumption, the $\Delta$ vectors have $n^{-K/2}$-growth, we are done by applying the already proven part (\ref{tamehermite}) of tameness for polynomials.

Now we prove condition (\ref{tameerror}) of tameness for polynomials. Only a finite number of Hermite coefficients are non-zero, so we see that for any $\underline{\alpha},k_0,\ell$, the quantity $R_{\underline{\alpha},k_0}^{\mathsf P}[\phi(Z^{(\ell)})]$ is a sum of terms of the form
\begin{equation}
    \label{polyremainderterm}
\frac{1}{\underline{k}!}\left(\prod_{a=1}^r\widehat{\phi^{\alpha_a}}^{\left(\mathbb E\left[Z\right],\operatorname{Var}[Z]\right)}_{k_a}\right)\prod_{S\in\pi} \kappa [Z_{j_v} : (v,w) \in S]
\end{equation}

for a connected $[\leq2]$-mixed partition $\pi$ of $\{(v,w) : v\in\{1,...,m\},w\in\{1,...,k_v\} \}$ for some $\underline{k} \in \mathbb{Z}_{\ge0}^m$ with $|\underline{k}| > k_0$. It suffices to prove that the term above has $n^{-\frac{k_0+1}{4}}$-growth. The Hermite product has $n^0$-growth by condition (\ref{tamehermite}).

Let $p'\ge2$ be even. Taking the $p'$-th power of the product of cumulants and then taking the expected value, Proposition \ref{wdiagsprop} tells us that we get a power $n^{-\frac{1}{2}p'|\underline{k}|}$ in front of a sum of fully-contracted diagrams of cumulants of $X^{(\ell-1)}$, and we need to compute the number of connected components. Any cumulant corresponding to $S\in \pi$ which has an odd intersection with some $[k_v]$ must be connected to some other cumulant, therefore the maximal number of connected components is
\[
    \frac{p'}{2}\left(|\pi|+\#\{S\in\pi : \forall v, |S\cap [k_v]| \equiv 0 \text{ mod }2\}\right).
\]
Since $\pi$ is $[\leq2]$-mixed, any set $S\in\pi$ satisfying $\forall v, |S\cap[k_v]|\equiv0 \text{ mod }2$ has size at least $4$, and any other set has size at least $2$. Therefore
\[
    |\pi|+\#\{S\in\pi : \forall v, |S\cap [k_v]| \equiv 0 \text{ mod }2\} \leq \frac{1}{2}|\underline{k}|
\]
and so the growth of this diagram has exponent
\[
    -\frac{1}{2}|\underline{k}|+ \frac{1}{2}(|\pi|+\#\{S\in\pi : \forall v, |S\cap [k_v]| \equiv 0 \text{ mod }2\}) \leq -\frac{1}{4}|\underline{k}| \leq  -\frac{1}{4} (k_0+1)
\]
as required.
\end{proof}

\begin{corollary}
    Theorem \ref{matchingsamplingtheorem} is satisfied by the basic algorithm.
\end{corollary}
\begin{proof}
    By Propositions \ref{prop:polytame} and \ref{prop:cumulantgrowth}, the conditions of Theorem \ref{thm:hermiteerror} are satisfied when the $\phi_{\ell}$ are all polynomials.
\end{proof}

In Section \ref{sfactorized}, we will describe the general form of the factorized algorithm, and show that it computes the same estimates as the basic algorithm, while saving a factor of $n$ in the runtime for all $\maxcumulantorder\geq 3$. This will prove that Theorem \ref{beatingsamplingtheorem} is satisfied by the factorized algorithm.

\section{Discussion of Hermite-based algorithms with non-polynomial activation functions}
\label{sec:nonpoly-activation}

In the previous section, Theorem \ref{thm:hermiteerror} did not assume that the activation functions are polynomials. For polynomials, we proved assumption (\ref{condtame}) in Proposition \ref{prop:polytame} and assumption (\ref{condsizes}) in Proposition \ref{prop:cumulantgrowth}, thus completing the proof of Theorem \ref{matchingsamplingtheorem}, the main theoretical result of the paper. We now discuss how we might go about proving assumptions (\ref{condtame}) and (\ref{condsizes}) in the non-polynomial setting. 

\subsection{Cumulant growth}

We start with assumption (\ref{condsizes}), which we expect to follow from tameness. We will try to prove this by induction. The first step is independent of the activation function and follows exactly as in the proof of Proposition \ref{prop:cumulantgrowth}.


\begin{lemma}
    \label{inductionone}
    Suppose that for some $\ell\in [L+1]$, every connected fully-contracted tensor diagram of cumulants of $X^{(\ell-1)}$ has $n$-growth. Then the same holds for connected fully-contracted tensor diagrams of cumulants of $Z^{(\ell)}$.
\end{lemma}
\begin{proof}
    Consider the diagram $ D $: taking the $p$-th power of this quantity corresponds to taking the disjoint sum of $p$ copies of this diagram. 
    By Proposition \ref{wdiagsprop}, taking the expectation over $\mathbf{W}^{(\ell)}$ produces a sum
    \[
        \left( \frac{2}{n} \right) ^{ p|E|/2 } \sum_{ \pi \text{ pairing of } pE(B) }^{  } D _{ \pi ,1 } D _{ \pi ,2 } 
    \] 
    where the number of connected components of $ D _{ \pi,1  } D _{ \pi ,2 } $ is at most $ p(1 + \frac{1}{2}|E|)$. Each of the connected components of $ D _{ \pi ,1 }  $ is a connected fully-contracted tensor diagram of cumulants of $ X ^{ (\ell-1) }  $ so contributes a factor of $ O(n) $, and each of the connected component of $ D _{ \pi ,2 }  $ is a connected fully-contracted tensor diagram of $ I_2 $'s, which also contribute a factor of $ n $. Therefore we get that the whole diagram is $ O(n^p) $ as required.
\end{proof}

The second part of the induction depends on $\phi_{\ell}$. Let us first replace each $\kappa_r[X^{(\ell)}]$ with a sum of connected contractions of power cumulants $\kappa_{\underline{\alpha}}^{\mathsf P}[X^{(\ell)}]$. Suppose that $\phi_{\ell}$ is tame, then we can pick $k_0\ge0$ and replace $\kappa_{\underline{\alpha}}^{\mathsf P}[X^{(\ell)}]$ using part (\ref{tameerror}) of the definition of tameness, which will produce a sum of diagrams containing tensors $R_{\underline{\alpha},k_0}^{\mathsf P}[\phi_{\ell}(Z^{(\ell)})]$, Hermite coefficient vectors, and cumulants $\kappa[Z^{(\ell)}]$. Consider the $r$-rank tensor
\[
\prod_{a=1}^r\left( \widehat{\phi^{\alpha_a}} _{k_a}^{(\mathbb{E}[Z],\mathrm{Var}[Z])}\right) _{j_a} \prod_{S\in \pi} \kappa[Z_{j_v}: (v,w)\in S]
\]
The cumulant part has growth at most $n^{-\frac{1}{4}|\underline{k}|}$, and so by Cauchy-Schwartz and condition (\ref{tamehermite}) of tameness, this rank-$r$ tensor has growth at most $n^{\frac{1}{2}(r - |\underline{k}|)}$. Since this is bounded above independently of $k_0$, we can choose $k_0$ large enough that any diagram containing a residual $R_{\underline{\alpha},k_0}^{\mathsf P}$ has at most $n$-growth. It remains to deal with the terms involving just Hermite coefficient vectors and cumulants $\kappa[Z^{(\ell)}]$.

These terms look like connected fully-contracted tensor diagrams of $\kappa[Z^{(\ell)}]$'s, except each index in the summation is attached to a product of Hermite coefficients, each of which is $O(1)$. If all the Hermite coefficients are a rank-$1$ tensor diagram in the cumulants $\kappa_r[Z^{(\ell)}]$ as is the case for polynomial activation functions, we would have a complete proof. We expect that for tame activation functions, the well-behaved nature of the Hermite coefficients should be sufficient to prove the same result, however we have not been able to provide a comprehensive proof of this fact.

\subsection{Tameness}\label{subsec:nonpoly-tameness}

Tameness provides a much more significant analytic challenge. Here we simply give an informal discussion on how one might formally prove tameness.

Conditions (\ref{tamehermite}) and (\ref{tamehermiteerror}) roughly correspond to a bounded, Lipschitz property of the Hermite coefficients of $\phi$ around $(\mathbb{E}[Z],\mathrm{Var}[Z])$. For example, it is easy to prove that
\begin{equation}
\label{eq:variancecontrol}
\mathbb{E}\left[ \left( \kappa_2[Z^{(\ell)}]_{i,i} - \frac{1}{n} \mathrm{Tr} \kappa_2[Z^{(\ell)}] \right)^2 \right] = 8n^{-4}(n-1)^2 \| \kappa_2 [X^{(\ell)}] \|_F^2
\end{equation}
which has $n^{-1}$-growth. Provided the activation functions are normalized correctly to keep $\mathrm{Tr}\kappa_2[Z^{(\ell)}]$ constant, this suggests that as $n\to\infty$ the variances will concentrate, and then provided the Hermite coefficients are well-behaved in a rectangle $(-C,C) \times (1-\epsilon,1+\epsilon)$, the properties (\ref{tamehermite}) and (\ref{tamehermiteerror}) should follow. This is simply a sketch, since we need significantly better control on the variances than (\ref{eq:variancecontrol}). For example, the Hermite coefficients of $\mathrm{ReLU}$ have a singularity at $\sigma^2=0$, so much more care is required (see Section \ref{sec:hermite-relu}).

Condition (\ref{tameerror}) of tameness is a statement concerning the validity of Theorem \ref{diagramstheorem}, the diagram summation formula for cumulants, when extended to non-polynomial functions $\phi$. We do not, in fact, expect this infinite summation formula to converge at any fixed width $n$, but instead that a fixed truncation should be an asymptotically good approximation as $n\to\infty$. This is the property of an asymptotic expansion, and in fact we can motivate the diagram summation formula through the theory of Edgeworth expansions, a special form of asymptotic series \citep{Hall1992}. Note that in the proof of Proposition \ref{prop:polytame}, we proved exactly that every term of the diagram summation formula which remains in $R_{\underline{\alpha},k_0}^{\mathsf P}$ in condition (\ref{tameerror}) of tameness has $n^{-\frac{k_0+1}{4}}$-growth. For polynomials, there are only finitely many terms, which completes the proof, but for general activation functions the definition of tameness simply says that the asymptotic growth of the error of truncation is no larger than any particular term that has been removed. We expect the theory of Edgeworth expansions in \citet{Hall1992} to apply to our situation, since the application of the linear layer produces a collection of almost iid random variables, however we have not performed this analysis.

\section{Proof of diagram summation formulas}\label{diagramsproof}
In this section we prove of our diagram summation formulas. Specifically, we have the polynomial diagram summation formula, Theorem \ref{polydiagsthm}, the Hermite-based diagram summation formula, Theorem \ref{diagramstheorem}, and its combinatorial restatement, Theorem \ref{thm:diagrams-restatement}. We begin by restating Theorems \ref{polydiagsthm} and \ref{diagramstheorem}.

\polydiagramsummationformula*

\diagramsummationformula*

The cumulants and Hermite coefficients used in these formulas are defined in Section \ref{preliminariessection}, and the definitions of the set $\mathrm{cDia}(\underline{k})$ of all connected diagrams on $\underline{k}$ and the set $\cDia{M}(\underline k)$ of all connected $\left[\leq M\right]$-mixed diagrams on $\underline k$ are given in Appendix \ref{diagramsappendix}. Some further intuition for these formulas is given in Section \ref{sec:polyadj}.

Both Theorem \ref{polydiagsthm} and Theorem \ref{diagramstheorem} assume that $\phi:\mathbb R\to\mathbb R$ is a polynomial. We make this assumption to avoid convergence issues, since it ensures that the expansions are finite sums. For non-polynomials, the formulas still hold in a formal sense, but can fail to converge. Convergence of these formal expansions is fragile: we do not expect these expansions to converge for a fixed $n$, but instead truncate to a finite summation such that the error converges to zero as $n\to\infty$. We discuss how this could be done using the notion of tameness in Section \ref{subsec:nonpoly-tameness}.

\subsection{Polynomial diagram summation formula for cumulants}
Let $\phi(t) = \sum_{k=0}^{\infty} \frac{1}{k!}a_kt^k$ be a polynomial (so $a_k = 0, \forall k \gg 0$). Then by multi-linearity of cumulants, we can expand
\begin{equation}
\label{multilin}
\kappa[\phi(Z_{j_1}),...,\phi(Z_{j_r})] = \sum _{\underline{k}\in \mathbb{Z}^r_{\ge0}} \left( \prod_{v=1}^r \frac{1}{k_v!}a_{k_v} \right) \kappa_r[Z_{j_1}^{k_1},...,Z_{j_r}^{k_r}]
\end{equation}

We now explicitly compute the cumulants of powers using the following Lemma.

\begin{lemma}\label{lem:simplepowers}
Let $X_1,...,X_r$ be $\mathbb{R}$-valued random variables, and $k_1,...,k_r \in \mathbb{Z}_{\ge0}$. Then
\begin{equation}
  \label{eq:powtoreg}
  \kappa \left[ Z _{ j_1 } ^{ k_1 } ,...,Z _{ j_r } ^{ k_r }   \right] = \sum_{ \pi\in\mathrm{cDia}(\underline{k}) }^{  } \prod_{ S \in \pi  }^{  } \kappa \left[ Z _{ j_v } : (v,l) \in S \right] 
\end{equation}
where we say $\pi \vdash [k_1]\sqcup ... \sqcup [k_r]$ is connected if $\pi \vee \{ [k_1],...,[k_r]\} $ has a single block.
\end{lemma}
\begin{proof}
    Consider the following two ways of expanding
    \[
    \mathbb E \left[ \prod_{v=1}^r Z_{j_v}^{k_v} \right]
    \]
    with the defining formula for cumulants. Firstly, consider this as an expectation of a product of $r$ random variables $Z_{j_1}^{k_1},...,Z_{j_r}^{k_r}$, giving 
    \[
    \mathbb E \left[ \prod_{v=1}^r Z_{j_v}^{k_v} \right] = \sum_{\pi \vdash [r]} \prod_{S\in \pi} \kappa \left[ Z_{j_v} ^{k_v} : v\in S \right]
    \]
    Secondly, consider it as an expectation of a product of $k_1+...+k_r$ random variables, and we get
    \begin{align*}
    \mathbb E \left[ \prod_{v=1}^r Z_{j_v}^{k_v} \right] &= \sum_{\tau\in\mathrm{Dia}(\underline{k})} \prod_{T\in \tau} \kappa \left[ Z_{j_v} : (v,l)\in T \right] \\
    &= \sum_{\pi \vdash [r]} \prod_{S\in\pi} \sum_{\tau\in\mathrm{cDia}(\underline{k}_S)} \prod_{T\in \tau} \kappa \left[ Z_{j_v} : (v,l)\in T \right]
    \end{align*}
    Here $\underline{k}_S := (k_v : v\in S)\in \mathbb{Z}_{\ge0}^{|S|}$. In these decompositions, uniqueness implies that
    \[
    \kappa \left[ Z_{j_v} ^{k_v} : v\in S \right] = \sum_{\tau\in\mathrm{cDia}(\underline{k_S})} \prod_{T\in \tau_S} \kappa \left[ Z_{j_v} : (v,l)\in T \right]
    \]
    and in particular the required result follows (with $S= [r]$).
\end{proof}
Substituting this result into (\ref{multilin}) completes the proof of Theorem \ref{polydiagsthm}.

\subsection{Hermite coefficients of polynomials}
Recall from Section \ref{preliminariessection} the definition of the Hermite coefficients: for $(\mu,\sigma)\in \mathbb{R}\times \mathbb{R}_{>0}$ and $Y\sim \mathcal{N}(\mu,\sigma^2)$,
\[
\widehat{\phi}_s^{(\mu,\sigma^2)} := \frac{1}{\sigma^s}\mathbb{E} \left[ \phi(Y) \mathrm{He}_s \left( \frac{Y-\mu}{\sigma}\right)\right].
\]

For a polynomial, we get a simple expression.

\begin{lemma}
    \label{lem:polyhermite}
    Let $\phi(t) = \sum_{k=0}^{\infty} \frac{1}{k!} a_k t^k $ be a polynomial. Then
    \[
        \widehat{\phi}_s^{(\mu,\sigma^2)} = \sum_{k=0}^{\infty} \frac{1}{k!}a_{s+k} \mathbb{E}[Y^k]
    \]
\end{lemma}
\begin{proof}
    The defining equation for the probabilist's Hermite polynomial is 
    \[
        \mathrm{He}_s (x) = (-1)^s e^{\frac{1}{2}x^2}\frac{d^s}{dx^s}e^{-\frac{1}{2}x^2} .
    \]
    Therefore,
    \begin{align*}
        \frac{1}{\sigma^s}\mathbb{E} \left[ \phi(Y) \mathrm{He}_s \left(\frac{Y-\mu}{\sigma}\right) \right] &= \frac{1}{\sigma^s\sqrt{2\pi}}\int_{-\infty}^{\infty} \phi(\mu + \sigma x) \mathrm{He}_s(x) e^{-\frac{1}{2}x^2} dx \\
        &= (-1)^s\frac{1}{\sigma^s\sqrt{2\pi}} \int_{-\infty}^{\infty} \phi(\mu+\sigma x) \left(\frac{d^s}{dx^s} e^{-\frac{1}{2}x^2} \right)dx \\
        &= \frac{1}{\sqrt{2\pi}}\int_{-\infty}^{\infty} \phi^{(s)} (\mu+\sigma x) e^{-\frac{1}{2}x^2} dx \\
        &= \mathbb{E} \left[ \phi^{(s)} (Y)\right] \\
        &= \sum_{k=0}^{\infty} \frac{1}{k!}a_{s+k} \mathbb{E}[Y^k]\qedhere
    \end{align*}
    
\end{proof}

The terms $\mathbb{E}[Y^k]$ can be computed using the defining formula for cumulants, using the fact that the only non-zero cumulants for a Gaussian are $\kappa_1$ and $\kappa_2$, so
\begin{equation}
    \label{eq:isserlis}
    \mathbb{E}[Y^k] = \sum_{\substack{\pi \vdash [k] \\ \forall S\in \pi, |S| \leq 2}} \prod_{S\in\pi} \kappa_{|S|}[Y].
\end{equation}

\subsection{Hermite-based diagram summation formula for cumulants}
We start at the polynomial diagram summation formula, for a polynomial $\phi(t) = \sum_{k=0}^{\infty} \frac{1}{k!}a_k t^k$,
\[
\begin{aligned}
\kappa\left[\phi\left(Z_{j_1}\right),\dots,\phi\left(Z_{j_r}\right)\right]=\sum_{\underline{k}\in \mathbb{Z}_{\ge0}^r}\left(\prod_{v=1}^r\frac 1{k_v!} a_{k_v}\right)\sum_{\pi \in \mathrm{cDia}(\underline{k})}\prod_{S\in\pi}\kappa\left[Z_{j_v}:\left(v,w\right)\in S\right]
\end{aligned}.
\]
We now divide the diagram $\pi\in \mathrm{cDia}(\underline{k})$ into $r+1$ diagrams, 
\begin{itemize}
    \item $\pi'\in \mathrm{cDia}_{[\leq2]}(\underline{k}')$, consisting of all $S\in \pi$ such that $|S| \ge 3$ or $\# \{ v : S \cap [k_v] \neq\emptyset\} > 1$,
    \item for each $v\in [r]$, $\pi_v \vdash [d_v]$ consisting of all $S \in \pi$ such that $|S|\leq 2$, and $S\subseteq [k_v]$.
\end{itemize} 
The pre-image of $(\pi',\pi_1,...,\pi_r)$ under the map $\pi \mapsto (\pi',\pi_1,...,\pi_r)$ has size $\prod_{v=1}^r \binom{k'_v+d_v}{d_v}$ since the only choice is the image $\bigcup_{S\in \pi'}S \subset [k_1]\sqcup ... \sqcup [k_r]$. Therefore, we get
\begin{align*}
\kappa\left[\phi\left(Z_{j_1}\right),\dots,\phi\left(Z_{j_r}\right)\right]&=\sum_{\underline{k}'\in \mathbb{Z}_{\ge0}^r} \prod_{v=1}^r  \left( \sum_{d_v=0}^{\infty} \binom{k_v'+d_v}{d_v}\frac 1{(k_v'+d_v)!} a_{k_v'+d_v}\sum_{\substack{\pi_v\vdash[d_v] \\ \forall S\in \pi_v, |S|\leq 2}}\prod_{S\in\pi_v}\kappa_{|S|}[Z_{j_v}]\right) \\
& \hspace{2cm}  \cdot\sum_{\pi' \in \mathrm{cDia}_{[\leq2]}(\underline{k'})}\prod_{S\in\pi'}\kappa\left[Z_{j_v}:\left(v,w\right)\in S\right] \\
&= \sum_{\underline{k}'\in \mathbb{Z}_{\ge0}^r} \frac{1}{k_v!} \prod_{v=1}^r  \left( \sum_{d_v=0}^{\infty} \frac 1{d_v!} a_{k_v'+d_v} \mathbb{E}[Y_v^k]\right)\sum_{\pi' \in \mathrm{cDia}_{[\leq2]}(\underline{k'})}\prod_{S\in\pi'}\kappa\left[Z_{j_v}:\left(v,w\right)\in S\right] \\
\end{align*}
where $Y_v = \mathcal{N}(\mathbb{E}[Z_{j_v}], \mathrm{Var}[Z_{j_v}])$, and we have applied (\ref{eq:isserlis}). Comparing the summation in parentheses with Lemma \ref{lem:polyhermite} completes the proof of Theorem \ref{diagramstheorem}.

\subsection{Proof of combinatorial restatement}
\label{combrestatproof}
Recall the combinatorial restatement, Theorem \ref{thm:diagrams-restatement}, of Theorem \ref{diagramstheorem}.

\diagrestat*

This form of the diagram summation formula is obtained by rewriting Theorem~\ref{diagramstheorem} in tensor notation,
grouping diagrams by their vector partition type, and applying Proposition~\ref{prop:vec-coef}, which we prove now.

\veccoef*
\begin{proof}[Proof of Proposition~\ref{prop:vec-coef}]
    Fix $\underline{k}\in\Z_{>0}^r$ and $\nu\in\Vec(\underline{k})$.
    We wish to compute 
    \begin{equation*}
        c^{\text{vec}}_\nu := \# \{\pi \vdash \left\{\left(v,w\right):v\in[r],w\in[k_v]\right\} : \type(\pi)=\nu\},
    \end{equation*}
    which is precisely the cardinality of the orbit of some $\pi$ satisfying $\type(\pi)=\nu$ under the action of $H_{\underline{k}}:=\prod_{v=1}^r S_{k_v}$ by permuting elements
    within each block $\{(v,w): w\in[k_v]\}$ for $v\in[r]$.
    By orbit-stabilizer,
    \begin{equation}
    \label{eq:vec-orbit-stabilizer}
        c^{\text{vec}}_\nu = [H_{\underline{k}} : \operatorname{Stab}_{H_{\underline{k}}}(\pi)] = \dfrac{\prod_{v=1}^r k_v !}{|\operatorname{Stab}_{H_{\underline{k}}}(\pi)|}.
    \end{equation}
    To find the cardinality of the stabilizer of $\pi$, note that any element of $H_{\underline{k}}$ that fixes $\pi$ must map each block of $\pi$ to a block
    of the same type. It may also arbitrarily permute elements within intersections of $\pi$-blocks and $\underline{k}$-blocks.
    Thus, letting $\nu(\underline{u})$ denote the number of times $\underline{u}$ appears in $\nu$,
    \begin{equation*}
        |\operatorname{Stab}_{H_{\underline{k}}}(\pi)| = \prod_{\underline{u}} \underbrace{\nu(\underline{u})!}_{\text{\# ways to permute blocks of type $\underline{u}$}} \prod_{i=1}^r \underbrace{(u_i!)^{\nu(\underline{u})}}_{\text{\# ways to permute elements of form $(i,v)$ in each block of type $\underline{u}$}},
    \end{equation*}
    where the product is over all $\underline{u}$ appearing in $\nu$.
    Substituting into \eqref{eq:vec-orbit-stabilizer} completes the proof.
\end{proof}

\section{Proofs of preliminary results}
\label{sec:prelimproofs}

\subsection{Proofs for harmonic decompositions}\label{harmonicproof}

Here we prove Theorem \ref{thm:harmonicdecomposition}.

\harmonicdecompositionformula*

\begin{proof}
    First, for a rank $r$ tensor over $\mathbb{R}^n$, a simple induction proves that for all $k\ge 0$,
    $$
        \mathrm{tr} (g^k(\bfT ))-g^k(\mathrm{tr}(\bfT )) = \left(2kr + kn +2k(k-1)\right)g^{k-1}(\bfT ).
    $$
    Let us solve for the coefficients: the defining equations are
    \begin{align*}
        \forall s, \sum_{0\leq t \leq \lfloor \frac{r}{2}\rfloor-s} c^{\mathrm h}_{s,t}\mathrm{tr}(g^t(\mathrm{tr}^{s+t}(\bfT ))) &= 0 \\
        c^{\mathrm h}_{0,0}&=1 \\
        \forall k\ge 1, \qquad \sum_{s+t=k}c^{\mathrm h}_{s,t}&=0
    \end{align*}
    Applying the commutation relation to the first of these equations results in
    $$
        \sum_{1\leq t \leq \lfloor \frac{r}{2}\rfloor-s} \left( 2tc^{\mathrm h}_{s,t}\left(r-2s-t+\frac{n}{2}-1\right) + c^{\mathrm h}_{s,t-1} \right) g^{t-1}\mathrm{tr}^{s+t}(\bfT ) = 0.
    $$
    The fact that the coefficients here must be zero for this to hold for arbitrary $\bfT $ means that, by recursion,
    $$
    c^{\mathrm h}_{s,t} = (-1)^t \frac{(r-2s+\frac{n}{2}-2-t)!}{2^tt!(r-2s+\frac{n}{2}-2)!} c^{\mathrm h}_{s,0}
    $$
    Now we simply need to figure out the components, $c^{\mathrm h}_{s,0}$, which satisfy the third condition,
    $$
    \sum_{s+t=k} (-1)^j \frac{(r-2s+\frac{n}{2}-2-t)!}{2^tt!(r-2s+\frac{n}{2}-2)!} c^{\mathrm h}_{s,0} = 0.
    $$
    This allows us to recursively find the base coefficients, and it can be verified that
    $$
    c^{\mathrm h}_{s,0} = \frac{(r+\frac{n}{2}-2s-1)!}{2^ss!(r+\frac{n}{2}-s-1)!}.
    $$
    Therefore,
    \[
    c^{\mathrm h}_{s,t} = (-1)^t \frac{r+\frac{n}{2}-2s-1}{2^{s+t}s!t!(r+\frac{n}{2}-2s-t-1)_{s+t+1}}.\qedhere
    \]
\end{proof}

\subsection{Diagram expectations}
\label{wdiagsproof}
Here we prove Proposition \ref{wdiagsprop}.

\wdiags*

\begin{proof}
    The key result for the proof is Isserlis's Theorem (which can in fact be seen directly from the defining property of cumulants): if $(G_1,...,G_r)$ is a multivariate zero-mean Gaussian,
    $$
        \mathbb{E} [G_1...G_r] = \sum_{\pi\in\mathcal{P}_2([r])} \prod_{(u,v)=S\in \pi} \mathbb{E} [ G_uG_v].
    $$
    Applying this to entries of $\mathbf W^{(\ell+1)}$, and using the relation $\mathbb{E} [ \mathbf W^{(\ell+1)}_{i,j} \mathbf W^{(\ell+1)}_{i',j'}] = \frac{2}{n} \delta _{i,i'}\delta_{j,j'}$, we get
    $$
        \mathbb{E}\left[\prod_{i=1}^r \mathbf W^{(\ell+1)}_{i_r,j_r} \right] = \left( \frac{2}{n} \right)^{r/2} \sum _{\pi \in \mathcal{P}_2([r])} \prod_{(u,v)=S\in \pi} \delta_{i_u,i_v} \delta_{j_u,j_v}.
    $$
    Now, we apply \ref{def:symmcontr} and \eqref{tensordiagramdef}, and get
    \begin{align*}
    \mathbb{E}\left[ D \odot \mathbf W^{(\ell+1)} \right] &= \sum _{\substack{i_v=1 \\ \forall v\in I} }^n \mathbb{E} \left[\prod_{u \in V_1} \sum_{\substack{j_e=1\\ \forall e=(u,v) \in E}}^n (T_u)_{\{j_e : (u,v) \in E \} }  \prod_{v : (u,v) \in E} \mathbf W_{i_v,j_e}^{(\ell+1)} \right] \\
    &= \sum _{\substack{i_v=1 \\ \forall v\in I} }^n \sum _{ \substack{j_e=1 \\ \forall e\in E}}^n \left(\prod_{u\in V_1}(T_u)_{\{j_v : (u,v) \in E \} } \right)\mathbb{E}\left[ \prod_{e=(u,v) \in E} W_{i_v,j_e}^{(\ell+1)} \right] \\
    &= \left( \frac{2}{n} \right)^{|E|/2} \sum _{\pi \in \mathcal{P}_2(E)} \left( \sum _{\substack{j_S=1 \\ \forall S\in \pi}}^n \prod_{u\in V_1} (T_u)_{j_S : u\in e\in S\in \pi}  \right) \left( \sum_{\forall v \in I, i_v=1}^n \prod_{S\in \pi} \prod_{(u,v)= e \in S} \delta_{i_u,i_v} \right)
    \end{align*}
    Now we simply recognize the left bracket inside the sum as $D_{\pi,out}$ and the right bracket as $D_{\pi,in}$.

    The connected components statement follows from the fact that replacing a pair of edges of $E$ by a pair of identity matrices can increase the number of connected components by at most $1$.
\end{proof}

We also prove Lemma \ref{lem:cs}.

\cs*

\begin{proof}
    To actually talk about entries of $ D $, we need to pick an ordering of $ V_2 \setminus I $, say $ \sigma : [ \mathrm{rank}(D) ] \rightarrow V_2 \setminus I $. Then,
  \[
    D _{ i _{ \sigma (1) } ,...,i _{ \sigma (r) }  } = \sum_{ (i _{ w } ) \in [n] ^{ I }  }^{  } \prod_{ v \in V_1 }^{  } (T_{ v } ) _{ \left\{ i_w : (v,w) \in E \right\}  } .
  \] 
  Splitting this product over $ V_1 = U_1 \sqcup U_2 $ and applying Cauchy-Schwartz (on vectors indexed by $ [n]^I $) implies
\[
  D _{ i _{ \sigma (1) } ,..., i _{ \sigma (r) }  } ^{ 2 } \leq \left( \sum_{ (i _{ w } ) \in [n] ^{ I }  }^{  } \prod_{ v \in U_1 }^{  } (T _{ v } )^2 _{ \left\{ i_w : (v,w) \in E \right\}  }  \right) \left( \sum_{ (i_w) \in [n]^I }^{  } \prod_{ v \in U_2 }^{  } (T_v) _{ \left\{ i_w : (v,w) \in E \right\}  } ^{ 2 }  \right) 
\] 
These bracketed values are clearly equal to $ (D_1) _{ i _{ \sigma (1) } ,...,i _{ \sigma (r) }  }  $ and $ (D_2) _{ i _{ \sigma (1) } ,...,i _{ \sigma (r) }  }  $ respectively. This proves non-negativity and the inequality.

Raising this inequality to the power $ p $ and taking expectations, we can apply Cauchy-Schwartz again, 
\[
  \mathbb{ E } _{ \boldsymbol\theta } \left[ \left| D _{ i _{ \sigma (1) } ,...,i _{ \sigma (r) }  }  \right| ^{ p }  \right] \leq \mathbb{ E } _{ \boldsymbol\theta } \left[ (D_1) _{ i _{ \sigma (1) } ,..., i _{ \sigma (r) }  } ^{ p }  \right] ^{ \frac{1}{2} } \mathbb{ E } _{ \boldsymbol\theta } \left[ (D_2) _{ i _{ \sigma (1) } ,..., i _{ \sigma (r) }  } ^{ p }  \right] ^{ \frac{1}{2} } 
\] 
which completes the growth claim.
\end{proof}

\subsection{Proof of power cumulant formula}
\label{powercumulantsproof}

Here, we restate and prove Proposition~\ref{powercumulantsprop}.

\powercumulantsformula*

\begin{proof}

Fix $\pi\vdash[r]$ and $(i_1,\ldots, i_r)$ of type $\pi$.
Let $\mu$ denote the M\"obius function on the partition lattice, so that $\mu(\pi)=(-1)^{|\pi|-1}(|\pi|-1)!$.
Then,
\begin{align}
    \label{eq:pK1}
    \kappa[X_{i_1},\ldots, X_{i_r}] &= \sum_{\tau\vdash[r]} \mu(\tau) \prod_{B\in\tau} \E\left[\prod_{j\in B} X_{i_j}\right]\\
    \label{eq:pK2}
    &=\sum_{\tau\vdash[r]} \mu(\tau) \prod_{B\in\tau} \E\left[\prod_{C\in\pi|_B} X_{i_{\min(C)}}^{|C|}\right]\\
    \label{eq:pK3}
    &=\sum_{\tau\vdash[r]} \mu(\tau) \prod_{B\in\tau}\left(\sum_{\rho_B \geq \pi|_B}\prod_{D\in\rho_B}  \kappa^{\mathsf P}_{|D|}[X]_{i_j: j\in D}\right)\\
    \label{eq:pK4}
    &=\sum_{\tau\in[r]}\mu(\tau)\sum_{\pi\wedge\tau\leq\rho\leq\tau}\prod_{D\in\rho}\kappa^{\mathsf P}_{|D|}[X]_{i_j: j\in D}\\
    \label{eq:pK5}
    &=\sum_{\rho\vdash[r]}\left(\sum_{\substack{\tau\vdash[r]\\\tau\geq \rho\\\pi\wedge\tau\leq\rho}}\mu(\tau)\right)\prod_{D\in\rho}\kappa^{\mathsf P}_{|D|}[X]_{i_j: j\in D}.
\end{align}
where
\begin{itemize}
\item \eqref{eq:pK1} is the usual M\"obius inversion to go from moments to cumulants,
\item \eqref{eq:pK2} is grouping $X_{i_j}$ by $\type(i_j: j\in B)=\pi|_B$,
\item \eqref{eq:pK3} is the usual sum over partitions to go from cumulants to moments (separately for each $B$),
\item \eqref{eq:pK4} is the observation that the product of partition lattices of blocks of $\tau$ is in bijection with the sublattice beneath $\tau$ in the partition lattice of $[r]$.
\end{itemize}
Now, let $\lambda\vdash r$ be the type of $\pi$.
Then \eqref{eq:pK5} is written in diagonal slice notation as 
\begin{equation*}
    \tilde\Diag_\lambda \kappa_r[X] = \sum_{\rho\vdash[r]}\left(\sum_{\substack{\tau\vdash[r]\\\tau\geq \rho\\\pi\wedge\tau\leq\rho}}\mu(\tau)\right) \prod_{D\in\rho}\tilde\Diag_{\type(\pi|_D)}\kappa_{|D|}\pow[X],
\end{equation*}
and the Proposition follows from grouping $(\pi,\rho)$ of the same vector partition type.
(Note that $\{\tau\vdash[r]:\tau\geq\rho\}$ is in bijection with $\{\tilde\tau\vdash\rho\}$, and $\pi\wedge\tau\leq\rho$ iff each block of $\tilde\tau$
contains blocks of $\rho$ such that no pair of them intersects the same block of $\pi$, whence the sum over totally disconnected set partitions
of the vector partition corresponding to $(\pi,\rho)$.)
\end{proof}



This Proposition tells us that each entry of the cumulants tensors is calculated in $O(1)$ time from the power cumulants.

\subsection{Proofs of diagonal-harmonic formulae}
\label{sdiagharmonic}
This section is devoted to proofs of Propositions \ref{prop:trace-delta} and \ref{prop:delta-cup}. Before restating and proving the results, we prove some necessary preliminaries.

\begin{lemma}
    Let $\lambda=(\lambda_1,\ldots,\lambda_b)=(1^{b_1}\ldots r^{b_r})\vdash r$. Then
    \begin{equation*}
        c^{\mathrm{part}}_\lambda := \#\{\pi\vdash[r]: \type(\pi)=\lambda\} = \dfrac{r!}{\prod_{i=1}^b \lambda_i! \prod_{j=1}^r b_j!}.
    \end{equation*}
\end{lemma}
\begin{proof}
    By orbit-stabilizer applied to the action of $S_n$.
    Fix a particular $\pi\vdash[r]$ with $\type(\pi)=\lambda$.
    A permutation that stabilizes $\pi$ can only permute blocks of the same size ($\prod_j b_j!$ permutations) and then permute elements within each block ($\prod_i m_i!$ permutations).
    Thus the stabilizer has cardinality $\prod_i \lambda_i! \prod_j b_j!$.
\end{proof}

\begin{definition}
    A \textit{$m$-partial pairing} of $[r]$ is a set of $m$ disjoint subsets $[r]$ each of cardinality two.
    We notate the set of $m$-partial pairings of $[r]$ as $\Pair(r,m)$.
    Given $\lambda\vdash r$ with $b$ blocks, write $\pi_\lambda\vdash[r]$ for the set partition of type $\lambda$
    formed by partitioning $[r]$ into consecutive contiguous blocks of cardinality $\lambda_1,\ldots,\lambda_b$:
    \begin{equation*}
        \pi_\lambda = (\{1,\ldots, \lambda_1\}, \{\lambda_1+1,\ldots, \lambda_1+\lambda_2\}, \ldots, \{r-\lambda_b+1,\ldots,r\}).
    \end{equation*}
    We call this the \textit{sorted set partition} of type $\lambda$.
    The \textit{$\lambda$-type} of $\rho\in\Pair([r],m)$, which we denote $\type_\lambda(\rho)$, is the multigraph with $m$ edges on $[b]$
    induced by forgetting the ordering on elements within each block of $\pi_{\lambda}$.
\end{definition}

\begin{lemma}
    Let $\lambda\vdash[r]$ with $b$ blocks, and $\gamma\in\Gamma([b],m)$. Then
    \begin{equation}
    \label{eq:c-g}
        \tilde c^{\mathrm g}_{\lambda,\gamma}:=\#\{\rho\in\Pair([r],m) : \type_\lambda(\rho)=\gamma\} = 
        \dfrac{\prod_{i\in[b]} (\lambda_i-d_\gamma(i)+1)_{d_\gamma(i)}}
        {\prod_{i\in[b]}2^{\gamma(\{i\})} \prod_{\{i,j\}\subseteq[b]} \gamma(\{i,j\})!}.
    \end{equation}
\end{lemma}
\begin{proof}
    One approach is another orbit-stabilizer argument. 
    Alternatively, $\prod_{i\in[b]} (\lambda_i-d_\gamma(i)+1)_{d_\gamma(i)}$ is the number of ways to assign endpoints $s_i,s_j\in[r]$
    to each edge $\{i,j\}\subseteq[b]$ such that each $s_i$ is in block $i$ and $s_j$ is in block $j$.
    This double-counts the loops (since in that case $i=j$), which contributes the factor $\prod_{i}2^{\gamma(\{i\})}$.
    We also overcount by permutations of edges between the same two vertices, which contributes the factor $\prod_{\{i,j\}} \gamma(\{i,j\})!$
\end{proof}

\begin{definition}
    Let $\rho\in\Pair([r], m)$. Then $\tr_{\rho}\colon \tens{r}{n}\to \tens{r-2m}{n}$ is the linear map that traces out indices corresponding to each pair in $\rho$.
    To be more precise, fix an ordering of pairs $\rho=\{\{u_1,v_1\},\ldots,\{u_m,v_m\}\}$, and sort the vertices
    \begin{equation*}
    \{u_1,v_1,\ldots, u_m, v_m\} = \{c_1,\ldots, c_{2m}\} \qquad 0=c_0<c_1<\ldots <c_m.
    \end{equation*}
    Then, for $\bfT\in\syms{n}{r}$,
    \begin{equation*}
        (\tr_\rho\bfT)_{i_1,\ldots, i_{r-2m}} = \sum_{i_{r-2m+1},\ldots, i_{r-m}=1}^n \bfT_{i_{j_1},\ldots,i_{j_r}},
    \end{equation*}
    where the indices $j$ satisfy
    \begin{equation*}
        j_s = \begin{cases}
            r-2m+k & s\in\{u_k,v_k\}\\
            s-t & c_t < s < c_{t+1}.
        \end{cases}
    \end{equation*}
    As a slight abuse of notation, we may think of the standard $\tr$ operator as mapping $\tens{r}{n}\to \tens{r-2}{n}$
    by contracting the first two indices, i.e.
    \begin{equation}
    \label{eq:trace12}
    \tr:=\tr_{\{\{1,2\}\}}.
    \end{equation}
    This agrees with the original definition of $\tr$ when restricted to $\syms{r}{n}\subseteq\tens{r}{n}$.
\end{definition}

\begin{definition}
    For $\bfM\in\syms{2}{r}$, define $g_{\bfM,\rho}\colon \tens{r-2m}{n}\to\tens{r}{n}$ by
    \begin{equation*}
        (g_{\bfM,\rho} \bfT)_{i_1,\ldots, i_r} = \bfT_{i_{j'_1},\ldots,i_{j'_{r-2m}}} \prod_{\{u,v\}\in\rho}^m\bfM_{i_{u}, i_{v}}
    \end{equation*}
    where
    \begin{equation*}
        j'_s = s+t \quad \text{when} \quad c_t < s < c_{t+1}.
    \end{equation*}
    Note then that
    \begin{equation}
    \label{eq:sum-cup-pair}
        g^m_{\bfM} = \sum_{\rho\in\Pair([r],m)}g_{\bfM,\rho}.
    \end{equation}
\end{definition}

\begin{definition}
    For $\sigma\in S_r$, the operator $\Perm_\sigma\tens{r}{n}\to\tens{r}{n}$ acts by permuting indices, i.e.
    \begin{equation*}
        (\Perm_\sigma \bfT)_{i_1,\ldots, i_r} = \bfT_{i_{\sigma(1)},\ldots,i_{\sigma(r)}}.
    \end{equation*}
\end{definition}

Now we can prove the required results. Specifically,

\tracedelta*

\begin{proof}[Proof of Proposition~\ref{prop:trace-delta}]
\label{sec:proof-harmonic-diag}
    Let $\lambda\vdash r$ with $b$ blocks, and let $\pi_\lambda\vdash[r]$ be the sorted set partition of type $\lambda$.
    For $0\leq m\leq \lfloor r/2\rfloor$ and $\rho\in\Pair([r],m)$ and any $\bfS\in\syms{r}{n}$

    Let us consider the operator $\tr_\rho\Diag_{\pi_\lambda} \bfS$ for some $\rho\in\Pair([r],m)$ and $\bfS\in\syms{r}{n}$.
    If $\type_\lambda(\rho)$ is not totally disconnected,
    then there exists some $\{u,v\}\in\rho$ such that $u\not\sim_{\pi_\lambda} v$.
    By definition $\tr_\rho$ sums over entries corresponding to indices with $i_u=i_v$, 
    and the output of $\Diag_{\pi_\lambda}$ is nonzero only when $i_u\neq i_v$,
    whence $\tr_\rho\Diag_{\pi_\lambda}\bfS=0$.

    What happens when $\type_\lambda(\rho)$ is totally disconnected?
    The result depends only on $\gamma=\type_\lambda(\rho)\leq\lambda$, where the inequality is taken elementwise.
    For $j\in[b]$ such that $\lambda_j=d_\gamma(j)$, all indices in the $j\nth$ block are summed over,
    corresponding to summing over the $j\nth$ index of $\tilde\Diag_\lambda\bfS$.
    On the other hand, if $\lambda_j>d_\gamma(j)$, there are $\lambda_j-d_\gamma(j)$ free indices
    in block $j$ not summed over.
    By the definition of $\Diag_{\pi_\lambda}$, if there are multiple such indices, they must be set to the same value
    for the corresponding entry to be nonzero. 
    Across all indices, this corresponds to applying $\tilde\Diag^{-1}_{\lambda_j-d_\gamma(j)}$ to the result of summing
    over indices of $\tilde\Diag_\lambda\bfS$ with $\lambda_j>d_\gamma(j)$.
    We conclude
    \begin{equation}
    \label{eq:tr-rho-diag}
        \tr_\rho\Diag_{\pi_\lambda}\bfS = \tilde\Diag_{\lambda-d_{\gamma}}^{-1}\operatorname*{Sum}_{\{j:\lambda_j=d_{\gamma}(j)\}} \tilde\Diag_\lambda\mathbf S,
    \end{equation}
    
    
    Hence,
    \begin{align*}
        \tr^m\Diag_\lambda\bfT 
        &\stackrel{\text{(\ref{eq:sum-diag-pi})}}{=}\tr^m\left( \sum_{\substack{\pi\vdash[r]\\\type(\pi)=\lambda}} \Diag_\pi \bfT \right)\\
        &=\tr^m\left( \dfrac{c^{\mathrm{part}}_\lambda}{r!}\sum_{\sigma\in S_r} (\Perm_\sigma\Diag_{\pi_\lambda})\bfT \right)\nonumber\\
        &=\dfrac{c^{\mathrm{part}}_\lambda}{r!} \left(  \sum_{\sigma\in S_r} (\tr^m \Perm_\sigma) \Diag_{\pi_\lambda}\right)   \bfT\nonumber\\
        &\stackrel{\text{(\ref{eq:trace12})}}{=}\dfrac{c^{\mathrm{part}}_\lambda}{\binom{r}{2m}(2m-1)!!} \left(  \sum_{\rho\in \Pair([r],m)} \tr_\rho \Diag_{\pi_\lambda}\right) \bfT\nonumber\\
        &\stackrel{\text{(\ref{eq:tr-rho-diag})}}{=}\dfrac{c^{\mathrm{part}}_\lambda}{\binom{r}{2m}(2m-1)!!} \left(\sum_{\rho\in \Pair([r],m)} \mathbbm{1}_{\type_{\lambda}(\rho)\text{ t.d.}}\tilde\Diag_{\lambda-d_{\type_{\lambda}(\rho)}}^{-1}\operatorname*{Sum}_{\{j:\lambda_j=d_{\type_{\lambda}(\rho)}(j)\}} \right)  \tilde\Diag_\lambda \bfT\\
        &\stackrel{\text{(\ref{eq:c-g})}}{=}\dfrac{c^{\mathrm{part}}_\lambda}{\binom{r}{2m}(2m-1)!!} \left( \sum_{\gamma\in\Gamma_{\text{t.d.}}([b], m)} \tilde c^{\mathrm g}_{\lambda,\gamma}\mathbbm{1}_{\lambda\geq d_\gamma} \tilde\Diag_{\lambda-d_{\gamma}}^{-1}\operatorname*{Sum}_{\{j:\lambda_j=d_{\gamma}(j)\}}\right)  \tilde\Diag_\lambda \bfT\\
        &= \left( \sum_{\gamma\in\Gamma_{\text{t.d.}}([b], m)} c^{\mathrm g}_{\lambda,\gamma}\mathbbm{1}_{\lambda\geq d_\gamma}\tilde\Diag_{\lambda-d_{\gamma}}^{-1}\operatorname*{Sum}_{\{j:\lambda_j=d_{\gamma}(j)\}}\right)  \tilde\Diag_\lambda \bfT.\nonumber\qedhere
    \end{align*}
\end{proof}
%

\deltacup*

\begin{proof}[Proof of Proposition~\ref{prop:delta-cup}]
First, for $\lambda\vdash r$ with $b$ blocks and $\rho\in\Pair([r], m)$, calculate
\begin{align}
    (\tilde\Diag_\lambda g_{\bf M,\rho} \bfT)_{i_1,\ldots, i_b} &= (g_{\bfM,\rho} \bfT)_{\underbrace{\scriptstyle{i_1,\ldots, i_1}}_{\text{$\lambda_1$ times}},\ldots, \underbrace{\scriptstyle{i_b,\ldots, i_b}}_{\text{$\lambda_b$ times}}}\nonumber\\
    &=\bfT_{\underbrace{\scriptstyle{i_1,\ldots, i_1}}_{\text{$\lambda_1-d_{\type_\lambda(\rho)}(1)$ times}},\ldots, \underbrace{\scriptstyle{i_b,\ldots, i_b}}_{\text{$\lambda_b-d_{\type_\lambda(\rho)}(b)$ times}}} \prod_{\{u,v\}\in\type_\lambda(\rho)} \bfM_{i_u,i_v}\mathbbm{1}[\text{$\underline i$ distinct}]\nonumber \\
    &=(\tilde\Diag_{\lambda-d_{\type_\lambda(\rho)}} \bfT )_{i_1,\ldots, i_b} \prod_{\{u,v\}\in\gamma} \mathbf M_{i_u,i_v} \mathbf{1}[\text{$\underline i$ all distinct}]\nonumber\\
    \label{eq:odot-gamma}
    &=: (\bfT\odot_{\type_\lambda\rho} \bfM)_{i_1,\ldots, i_b}.
\end{align}
Hence,
\begin{align*}
    \tilde\Diag_\lambda g^m_\bfM(\bfT) &\stackrel{\text{(\ref{eq:sum-cup-pair})}}{=} \sum_{\rho\in\Pair([r],m)} \tilde\Diag_\lambda g_{\bfM,\rho}\bfT\\
    &\stackrel{\text{(\ref{eq:odot-gamma})}}{=}\sum_{\rho\in\Pair([r],m)}  \bfT\odot_{\type_\lambda\rho} \bfM\\
    &\stackrel{\text{(\ref{eq:c-g})}}{=}\sum_{\gamma\in\Gamma([b],m)} \tilde c^{\mathrm g}_{\lambda,\gamma} \bfT\odot_{\gamma} \bfM.\qedhere
\end{align*}
\end{proof}

\subsection{Factorized cumulant propagation for arbitrary \texorpdfstring{$\maxcumulantorder\geq 3$}{K >= 3}}
\label{sfactorized}
\begin{proof}[Proof of Theorem~\ref{beatingsamplingtheorem}]
    Given Theorem~\ref{matchingsamplingtheorem} (proved in Section~\ref{smainproof}),
    it suffices to show that the factorization strategy exhibited in Section~\ref{sec:factorized-algo} for $\maxcumulantorder=3$ can be extended to $\maxcumulantorder>3$.
    The largest cumulant tensor has arity $\maxcumulantorder>3$ and we factor it as
    \begin{equation}
    \label{eq:gen-fac}
        \bfT_{i_1,\ldots, s_\maxcumulantorder} = 
        \dfrac{1}{\maxcumulantorder!}\sum_{\sigma\in S_\maxcumulantorder}\sum_{s=2}^{\lfloor r/2\rfloor}\sum_{t=1}^T\bfA^{(s)}_{i_{\sigma(1)}\ldots i_{\sigma(s)},t} \bfB^{(s)}_{i_{\sigma(s+1)}\ldots i_{\sigma(\maxcumulantorder)},t}.
    \end{equation}
Lower-degree cumulants are left unfactored.

For factorized cumulant propagation through each layer to run in time $O(n^{\maxcumulantorder})$, we need to verify each of the following:

\begin{itemize}[leftmargin=0.5cm]
\item
\textbf{Symmetric contraction $-\odot\bfW$ can be done in $O(n^\maxcumulantorder)$ time when $T=O(n)$}\\
Symmetric contraction of the factorized representation is done by contracting $\bfW$
into each of the indices other than $t$ of all of the factors $\bfA^{(s)}$ and $\bfB^{(s)}$ of \eqref{eq:gen-fac}.
Since the largest such factor has arity $\maxcumulantorder-1$, this takes time $O(n^{\maxcumulantorder-1} T)=O(n^{\maxcumulantorder})$.

\item
\textbf{Hermite coefficients can be multiplied in in time $O(n^\maxcumulantorder)$ time when $T=O(n)$}\\
This is clear because multiplying in Hermite coefficients is an elementwise operation on the factors.

\item
\textbf{Non-discrete diagonals can be extracted in $O(n^\maxcumulantorder)$ time when $T=O(n)$}\\
Let $\lambda\vdash r$ be other than the discrete partition $(1,\ldots, 1)$.
Then $\tilde \Diag_\lambda\bfT$ has no more than $n^{\maxcumulantorder-1}$ entries,
each of which can be computed from \eqref{eq:gen-fac} in time $T$.
Thus the total time is $O(n^{\maxcumulantorder-1} T)=O(n^{\maxcumulantorder})$.

\item
\textbf{Non-discrete diagonals can be factored into $T=\Theta(n)$ factors}\\
Given $\lambda\vdash r$ other than the discrete partition $(1,\ldots, 1)$, we claim
$\Diag_\lambda\bfT = \tilde\Diag_\lambda^{-1}\tilde\Diag_\lambda\bfT$
has a factorization in terms of $\tilde\Diag_\lambda\bfT$ with $\Theta(n)$ factors.
And indeed, since $\lambda_1>1$,
\begin{align*}
    (\Diag_\lambda\bfT)_{i_1,\ldots, i_r} 
    &=\mathbbm{1}[i_1=\cdots=i_{\lambda_1}](\tilde\Diag^{-1}_{(1,\lambda_2,\ldots,\lambda_n)}\tilde\Diag_\lambda\bfT)_{i_1,i_{\lambda_1+1},\ldots, i_r}\\
    &=\sum_{t=1}^n\delta^{(\lambda_1)}_{i_1,\ldots,i_{\lambda_1},t}(\tilde\Diag^{-1}_{(1,\lambda_2,\ldots,\lambda_n)}\tilde\Diag_\lambda\bfT)_{t,i_{\lambda_1+1},\ldots, i_r}
\end{align*}
is a factorization with $n$ factors,
where $\delta^{(\lambda_{1})}\in\syms{\lambda_{1}}{n}$ is the Kronecker delta, i.e.\ the tensor that is 1 along the $(\lambda_1)$-diagonal and zero elsewhere.

\item
\textbf{Power cumulant diagrams can be factored into $T=\Theta(n)$ factors}\\
Consider the set of diagrams contributing to the $(1,\ldots,1)\vdash \maxcumulantorder$ diagonal slice of the activation power cumulant,
which is the only slice that needs to be factored.
The key observation is that, 
because we are dealing with diagrams with the maximum possible arity,
if a vector partition of $(1,\ldots,1)$ is to satisfy \eqref{eq:vec-part-weight},
then it cannot have any cycles when viewed as a multi-hypergraph on $[\maxcumulantorder]$.
Indeed, an acyclic connected diagram $\nu$ already makes \eqref{eq:vec-part-weight} tight:\footnote{
This is a generalization of the standard fact that a tree on $V$ vertices has exactly $V-1$ edges. It can be proved using a straightforward inductive argument.
}
\begin{equation*}
    k(\nu)=1+\sum_{\underline{u}\in\nu}\left(\sum_{j=1}^\maxcumulantorder \left\lfloor \frac{u_j}{2}\right\rfloor-1\right)\geq1+ \sum_{\underline{u}\in\nu}\left(\sum_{j=1}^\maxcumulantorder u_j-1\right) =  \maxcumulantorder,
\end{equation*}
and completing a cycle with a hyperedge would cause the condition to be broken.

This leaves only a few possibilities for $\nu$:
\begin{enumerate}[label=(\alph*), leftmargin=0.5cm]
\item Each block of the $\nu$ has degree no more than $\maxcumulantorder-1$, in which case there must
be at least two blocks, and thus there exists a block with degree $\leq \maxcumulantorder-2$. 
(Since two blocks each of degree $\geq \maxcumulantorder-1$ must form a cycle.)
Call this block $\bfC\in\syms{c}{n}$ and the rest of the diagram $\bfD\in\syms{d}{n}$,
with $c,d\geq 2$ and $c\leq \maxcumulantorder-2$.
By the acyclicity and connectedness conditions, $\bfC$ and $\bfD$ must be incident along exactly one vertex.
Grouping by orbits under the action of $S_\maxcumulantorder$ on indices, the contribution corresponding to this diagram is then
\begin{align*}
    \bfT_{i_1,\ldots, s_\maxcumulantorder} &= c_\nu\sum_{\sigma\in S_\maxcumulantorder}\Perm_\sigma\left(\bfC_{i_1,\ldots,i_c}\bfD_{i_c,i_{c+1},\ldots, s_\maxcumulantorder}\right)\\
    &=c_\nu\sum_{\sigma\in S_\maxcumulantorder} \Perm_\sigma \left(\sum_{i=1}^n (\bfC_{i_1,\ldots,i_{c-1},i}\delta^{(2)}_{i_c,i}\bfD_{i,i_{c+1},\ldots, s_\maxcumulantorder}\right),
\end{align*}
where $c_{\sigma,\nu}$ is the appropriate coefficient coming from orbit-stabilizer.
This is a factorization into $n$ factors.
Note that the first factor $\bfC_{i_1,\ldots,i_c}\delta^{(2)}_{i_c,i}$ has arity $c+1\leq\maxcumulantorder-1$,
and so is indeed a valid factorization.

\item $\nu=\{(1,\ldots, 1)\}$, and the corresponding tensor is the already factorized $(1,\ldots, 1)$-slice from the previous layer.
\item $\maxcumulantorder$ is odd and $\nu=\{(2,\ldots,1)\}$ or one of its permutations.
Let $\bfM=\bfW\bfW\tran.$
Since in this case $i_{\maxcumulantorder}(\maxcumulantorder+1)=(\maxcumulantorder+1)/2$,
the corresponding tensor is $g_\bfM^{(\maxcumulantorder+1)/2}\tilde\eta_{\maxcumulantorder+1}[X]$,
with $\eta_{\maxcumulantorder+1}[X]$ a scalar.
Hence it suffices to write $g^{(\maxcumulantorder+1)/2}_\bfM1$ as a single factor
\begin{equation*}
    (g_\bfM^{(\maxcumulantorder+1)/2}1)_{i_1,\ldots, i_{\maxcumulantorder+1}} 
    =\dfrac{(\maxcumulantorder-1)!!}{\maxcumulantorder!}\sum_{\sigma\in S_\maxcumulantorder} \sum_{t=1}^1 \left(\bfM_{i_{\sigma(1)},i_{\sigma(2)}} \mathbf{1}_t\right) \left(\mathbf{1}_t \prod_{j=2}^{(\maxcumulantorder+1)/2}\bfM_{i_{\sigma(2j-1)},i_{\sigma(2j)}}\right).
\end{equation*}

\end{enumerate}
\end{itemize}
\end{proof}

\ifpreprint{}{
\clearpage
\section*{NeurIPS Paper Checklist}

\begin{enumerate}

\item {\bf Claims}
    \item[] Question: Do the main claims made in the abstract and introduction accurately reflect the paper's contributions and scope?
    \item[] Answer: \answerYes{}
    \item[] Justification: The abstract and introduction clearly summarize the main theoretical and empirical results.

\item {\bf Limitations}
    \item[] Question: Does the paper discuss the limitations of the work performed by the authors?
    \item[] Answer: \answerYes{}
    \item[] Justification: The assumption that the weights are randomly initialized, as well as the reliance on large width, are clearly stated in the title, abstract, introduction and throughout the paper. The poor depth dependence is mentioned in the introduction and discussed further in Appendix \ref{depthappendix}. The potential for extending the results to partially-trained networks is discussed in Section \ref{lpediscussionsubsection}.

\item {\bf Theory assumptions and proofs}
    \item[] Question: For each theoretical result, does the paper provide the full set of assumptions and a complete (and correct) proof?
    \item[] Answer: \answerYes{}
    \item[] Justification: Full assumptions are given in theorem statements. Full proofs of all theoretical results are given in the technical supplement, Appendix \ref{supplementappendix}.

    \item {\bf Experimental result reproducibility}
    \item[] Question: Does the paper fully disclose all the information needed to reproduce the main experimental results of the paper to the extent that it affects the main claims and/or conclusions of the paper (regardless of whether the code and data are provided or not)?
    \item[] Answer: \answerYes{}
    \item[] Justification: Precise descriptions of all algorithms are given in the technical supplement, Appendix \ref{supplementappendix}, and code is provided.

\item {\bf Open access to data and code}
    \item[] Question: Does the paper provide open access to the data and code, with sufficient instructions to faithfully reproduce the main experimental results, as described in supplemental material?
    \item[] Answer: \answerYes{}
    \item[] Justification: No data is required, and code is provided.

\item {\bf Experimental setting/details}
    \item[] Question: Does the paper specify all the training and test details (e.g., data splits, hyperparameters, how they were chosen, type of optimizer) necessary to understand the results?
    \item[] Answer: \answerYes{}
    \item[] Justification: The experimental setup is clearly described in Section \ref{empiricalresultssection}, and no data is required.

\item {\bf Experiment statistical significance}
    \item[] Question: Does the paper report error bars suitably and correctly defined or other appropriate information about the statistical significance of the experiments?
    \item[] Answer: \answerYes{}
    \item[] Justification: Error bars are included in all plots, except in cases where they would be too small to be readable, where this is explicitly noted.

\item {\bf Experiments compute resources}
    \item[] Question: For each experiment, does the paper provide sufficient information on the computer resources (type of compute workers, memory, time of execution) needed to reproduce the experiments?
    \item[] Answer: \answerNo{}
    \item[] Justification: Experiments can easily be replicated on CPU with minimal compute resources.
    
\item {\bf Code of ethics}
    \item[] Question: Does the research conducted in the paper conform, in every respect, with the NeurIPS Code of Ethics \url{https://neurips.cc/public/EthicsGuidelines}?
    \item[] Answer: \answerYes{}
    \item[] Justification: There are no immediate ethical implications of the results, and the safety motivation for the paper is discussed in Section \ref{lpediscussionsubsection}.

\item {\bf Broader impacts}
    \item[] Question: Does the paper discuss both potential positive societal impacts and negative societal impacts of the work performed?
    \item[] Answer: \answerYes{}
    \item[] Justification: The implications of the results for improving training efficiency and reducing catastrophic tail risk are discussed in Section \ref{lpediscussionsubsection}.
    
\item {\bf Safeguards}
    \item[] Question: Does the paper describe safeguards that have been put in place for responsible release of data or models that have a high risk for misuse (e.g., pre-trained language models, image generators, or scraped datasets)?
    \item[] Answer: \answerNA{}
    \item[] Justification: There is no direct risk of misuse.

\item {\bf Licenses for existing assets}
    \item[] Question: Are the creators or original owners of assets (e.g., code, data, models), used in the paper, properly credited and are the license and terms of use explicitly mentioned and properly respected?
    \item[] Answer: \answerNA{}
    \item[] Justification: No existing assets are used.

\item {\bf New assets}
    \item[] Question: Are new assets introduced in the paper well documented and is the documentation provided alongside the assets?
    \item[] Answer: \answerYes{}
    \item[] Justification: The provided code includes documentation.

\item {\bf Crowdsourcing and research with human subjects}
    \item[] Question: For crowdsourcing experiments and research with human subjects, does the paper include the full text of instructions given to participants and screenshots, if applicable, as well as details about compensation (if any)? 
    \item[] Answer: \answerNA{}
    \item[] Justification: The paper does not involve crowdsourcing nor research with human subjects.

\item {\bf Institutional review board (IRB) approvals or equivalent for research with human subjects}
    \item[] Question: Does the paper describe potential risks incurred by study participants, whether such risks were disclosed to the subjects, and whether Institutional Review Board (IRB) approvals (or an equivalent approval/review based on the requirements of your country or institution) were obtained?
    \item[] Answer: \answerNA{}
    \item[] Justification: The paper does not involve crowdsourcing nor research with human subjects.

\item {\bf Declaration of LLM usage}
    \item[] Question: Does the paper describe the usage of LLMs if it is an important, original, or non-standard component of the core methods in this research? Note that if the LLM is used only for writing, editing, or formatting purposes and does \emph{not} impact the core methodology, scientific rigor, or originality of the research, declaration is not required.
    \item[] Answer: \answerNA{}
    \item[] Justification: LLMs were used to help solve mathematical problems in the course of the research, but the final results including proofs were composed by the authors, with non-essential assistance from LLMs. LLMs were also used to help implement algorithms, but all code was carefully reviewed by the authors.

\end{enumerate}
}

\end{document}